%% file: iclr2027_conference.tex
\documentclass{article} 
\usepackage{iclr2027_conference,times}

\input{math_commands.tex}

\usepackage{hyperref}
\usepackage{url}

\usepackage[utf8]{inputenc} 
\usepackage[T1]{fontenc}    
\usepackage{hyperref}       
\usepackage{url}            
\usepackage{booktabs}       
\usepackage{amsfonts}       
\usepackage{nicefrac}       
\usepackage{microtype}      
\usepackage{xcolor}         
\usepackage{graphicx}
\usepackage{amsmath}
\usepackage{amssymb}
\usepackage{booktabs}
\usepackage[table]{xcolor}
\usepackage{multirow}
\usepackage{ulem}
\usepackage{wrapfig}
\usepackage{amsthm}
\usepackage{subcaption}
\newtheorem{theorem}{Theorem}

\usepackage{float}
\usepackage{adjustbox}
\newfloat{algorithm}{t}{loa}
\floatname{algorithm}{Algorithm}

\newcommand{\name}{G$^2$PTQ}

\definecolor{g2ptqblue}{HTML}{009FE0}
\newcommand{\highlight}[1]{\textcolor{g2ptqblue}{#1}}

\title{\name: Improving LLM Post-Training Quantization with Generalized Gradient Compensation}

\author{%
  \textbf{Ruikang Liu}$^1$, \textbf{Haoli Bai}$^2$, \textbf{Yuxuan Sun}$^3$, \textbf{Qian Zhang}$^4$, \textbf{Wenzheng Cai}$^1$, \textbf{Yanqi Hao}$^1$, \\
  \textbf{Feiyu Wang}$^1$, \textbf{Weidong Zhong}$^1$, \textbf{Zhuang Wang}$^1$, \textbf{Tong Yang}$^4$, \textbf{Xiangsheng Zhou}$^{1,5}$\thanks{Corresponding author.} \\
  \\
  $^1$ZTE Corporation \quad $^2$The Chinese University of Hong Kong \quad $^3$Northwestern Polytechnical University \\
  $^4$Peking University \quad $^5$Nanjing University of Aeronautics and Astronautics \\
  \texttt{\{liu.ruikang,zhou.xiangsheng\}@zte.com.cn}
}

\iclrfinalcopy 
\begin{document}

\maketitle

\begin{abstract}
  Post-training quantization (PTQ) is a practical approach to reducing the memory and computational footprint of large language models (LLMs) without retraining. GPTQ-based methods have become the de facto standard, yet they suffer from two complementary limitations. Methods with local, layer-wise objectives lack global supervision; while methods with global objectives fix their Hessian estimates at the start and ignore first-order gradients, so their guidance grows stale as quantization proceeds. This paper presents \name, a unified PTQ framework with Generalized Gradient Compensation that integrates both first- and second-order information under a globally supervised, block-wise optimization objective. By refreshing gradient and Hessian estimates before quantizing each Transformer block, \name\ avoids the staleness of prior global methods. Furthermore, to stabilize the exact first-order compensation, we introduce a trust-region scaling mechanism that dynamically bounds the gradient step to prevent exploding weight updates. Finally, we derive efficient implementations for block-wise Hessian approximation and exact gradient compensation. Experimental results on various model families and bit-widths demonstrate that \name\ enables better alignment with the full-precision model, outperforming state-of-the-art baselines. Code is available at: \url{https://github.com/G2PTQ/G2PTQ}.
\end{abstract}


\input{texes/introduction}
\input{texes/related_work}

\input{texes/background}
\input{texes/method}

\input{texes/experiments}

\input{texes/conclusions}
\input{references}
\bibliographystyle{iclr2027_conference}

\input{texes/appendix}

\end{document}

%% file: math_commands.tex
\usepackage{amsmath,amsfonts,bm}

\def\eqref#1{equation~\ref{#1}}

\def\1{\bm{1}}

\DeclareMathAlphabet{\mathsfit}{\encodingdefault}{\sfdefault}{m}{sl}
\SetMathAlphabet{\mathsfit}{bold}{\encodingdefault}{\sfdefault}{bx}{n}



%% file: texes/introduction.tex
\section{Introduction}

Large language models (LLMs) have demonstrated remarkable capabilities across a wide range of tasks, yet their deployment remains challenging due to substantial memory and computational requirements. Post-training quantization (PTQ) has emerged as a practical solution, compressing model weights to low-bit representations without retraining. Among PTQ methods, GPTQ~\citep{frantaroptq} and its variants have become the de facto standard, framing quantization as a second-order optimization problem: when a weight is quantized, the remaining weights are updated to compensate for the induced error using the layer-wise Hessian matrix.

Despite their success, existing GPTQ-based methods suffer from two complementary limitations. First, methods such as GPTQ~\citep{frantaroptq} and QuIP~\citep{chee2023quip} minimize a local layer-wise mean-squared error (MSE) objective, which aligns individual linear layer outputs but may lead to suboptimal global model performance. Recent work~\citep{edalati2025oac,kim2025guidedquant,tseng2025model} addresses this by replacing the local objective with an end-to-end negative log-likelihood (NLL) or KL divergence loss, using the Fisher information matrix to approximate the resulting Hessian. However, these methods compute the Hessian once at the start of quantization and keep it fixed throughout, causing the estimates to become increasingly stale as weights are progressively quantized. Second, GPTQ explicitly omits the first-order gradient term under the assumption of model convergence. While FOEM~\citep{zheng2026first} reintroduces this term and demonstrates its importance, it relies on a first-order Taylor approximation to compute gradients — an approximation that is exact only for the layer-wise MSE loss and introduces unnecessary error under more expressive objectives such as KL divergence or block-wise MSE. Neither line of work simultaneously addresses both limitations: existing global-objective methods lack first-order information, while first-order methods remain confined to local objectives.

We propose \name\ (Post-Training Quantization with Generalized Gradient Compensation), a unified framework that integrates both first- and second-order information under a globally supervised, block-wise optimization objective. Rather than minimizing layer-wise MSE, \name\ adopts a block-wise strategy: for each Transformer block, it minimizes the discrepancy between the quantized and full-precision block outputs, using block-wise MSE for intermediate blocks and KL divergence for the final block. Crucially, \name\ refreshes both the gradient and Hessian estimates before quantizing each block via efficient backward passes through that block alone, avoiding the staleness that plagues methods with fixed Hessians while remaining far cheaper than full-model backward passes. For the first-order term, \name\ computes exact gradients rather than relying on Taylor approximations, and introduces a trust-region scaling mechanism that dynamically bounds the gradient compensation step to prevent exploding weight updates. We further derive efficient Hessian approximations for both block-wise objectives and a gradient compensation scheme with computational complexity of $\mathcal O\!\left(\max\left\{d_{\text{col}}^3,\; d_{\text{row}} \cdot d_{\text{col}}^2\right\}\right)$, matching the complexity of vanilla GPTQ. 

The main contributions of this work are summarized as follows:
\begin{itemize}
    \item We propose \name, a PTQ framework that integrates first- and second-order quantization error compensation under a block-wise optimization objective. Together with the refreshed information and a trust-region gradient compensation strategy, \name\ provides accurate and globally informed guidance for weight updates.
    \item We derive efficient implementations for the block-wise Hessian approximation and exact gradient compensation, making \name\ practical to deploy. For instance, quantizing an 8B model requires only 2.24 hours and 12.54~GB of memory on a single accelerator.
    \item Extensive experiments across diverse model families demonstrate that \name\ achieves closer alignment with full-precision models, as evidenced by lower KL divergence compared to state-of-the-art baselines, ultimately leading to enhanced downstream task accuracy.
\end{itemize}


%% file: texes/related_work.tex
\section{Related Work}
\label{sec:related_work}

Quantization can be broadly categorized into post-training quantization (PTQ)~\citep{frantaroptq,lin2024awq,ashkboos2024quarot} and quantization-aware training (QAT)~\citep{liu2024llm,du2024bitdistiller,lv2026makes}. Owing to its simplicity and ease of deployment, PTQ has attracted widespread attention, with GPTQ~\citep{frantaroptq} in particular emerging as a widely adopted representative method. This section focuses on variants of GPTQ. A more comprehensive discussion on related work is deferred to Appendix~\ref{apdx:related_work}.

\paragraph{Optimal Brain Quantization.}
A central idea in PTQ is that the error introduced by quantizing a weight can be compensated by updating the remaining unquantized weights. This perspective underlies a line of methods~\citep{frantar2022optimal,frantaroptq,chee2023quip,zheng2026first} built on the optimal brain quantization framework. OBQ~\citep{frantar2022optimal} first establishes a framework for quantization error compensation. After one weight is quantized, the remaining weights in the same row are adjusted to compensate for the resulting error based on second-order information. GPTQ~\citep{frantaroptq} introduces several optimizations to the OBQ framework and significantly improves quantization efficiency, enabling its successful application to large-scale models. QuIP~\citep{chee2023quip} derives a more efficient implementation of GPTQ, termed LDLQ. More recently, FOEM~\citep{zheng2026first} revisits the GPTQ formulation and derives an improved algorithm by explicitly incorporating the previously neglected first-order quantization-error term into the GPTQ loss, further improving the quantization accuracy.

\paragraph{From Local to Global Optimization Objectives.}   
The vanilla GPTQ method minimizes a per-linear-layer MSE objective, which may result in suboptimal quantization accuracy due to local optimum. Recently, several studies~\citep{li2025gptaq,kim2024boa,kim2026turboboa,edalati2025oac,kim2025guidedquant,tseng2025model} have proposed alternative optimization objectives that deliver stronger performance by incorporating more global supervision. GPTAQ~\citep{li2025gptaq} introduces a cumulative quantization error term in the MSE objective of GPTQ, effectively compensating for the quantization error in the previously quantized linear layers. BOA~\citep{kim2024boa,kim2026turboboa} derives attention-aware Hessian matrices by considering inter-layer dependencies in attention modules, thereby minimizing the attention output distortion. Another line of work~\citep{edalati2025oac,kim2025guidedquant,tseng2025model} seeks to replace the local objective altogether with an end-to-end global objective, typically based on negative log-likelihood (NLL) loss. OAC~\citep{edalati2025oac} first proposes to minimize the end-to-end NLL loss by using a Hessian matrix approximated by Fisher information matrix. To mitigate the substantial computational and memory overhead, OAC assumes row-wise independence and computes a single Hessian matrix shared across all output channels within a linear layer. GuidedQuant~\citep{kim2025guidedquant} further refines this sharing scheme by partitioning the channels into multiple groups and computing a separate Hessian matrix for each group. YAQA~\citep{tseng2025model} does not impose any independence assumption on the weights within a linear layer and instead approximates the Hessian matrix using a Kronecker product.

%% file: texes/background.tex
\section{Preliminaries}
\label{sec:background}

We begin with necessary notations and review the prior work that forms the foundation of our method. Throughout, we adopt the row-vector convention. Let $\mathbf W, \hat{\mathbf W} \in \mathbb{R}^{d_{\text{row}} \times d_{\text{col}}}$ denote the full-precision and quantized weight matrices, $\mathbf X \in \mathbb{R}^{d_{\text{col}} \times n}$ the input activations with $n$ tokens.


\paragraph{GPTQ: Second-order Weight Compensation.}
GPTQ~\citep{frantaroptq} builds on OBQ~\citep{frantar2022optimal} and frames quantization as minimizing the layer-wise MSE:
\begin{equation}
\label{eq:mse_loss}
    \operatorname{argmin}_{\hat{\mathbf{W}}}\|\mathbf{W}\mathbf{X}-\hat{\mathbf{W}}\mathbf{X}\|_2^2.
\end{equation}
Under this objective, each row can be quantized independently. When the $t$-th weight is quantized, the remaining unquantized weights in the same row are updated to compensate for the induced error. The weight update $\Delta \mathbf{w}$ and the resulting loss change $\Delta L$ are given by
\begin{equation}
\label{eq:hessian_weight_update}
    \Delta\mathbf{w}=-\frac{\mathbf{w}_t-\hat{\mathbf{w}}_t}{\mathbf{H}^{-1}_{t,t}}\mathbf{H}^{-1}_{t,:}, \quad
    \Delta L = \frac{1}{2} \frac{(\mathbf{w}_t - \hat{\mathbf{w}}_t)^2}{\mathbf{H}^{-1}_{t,t}},
\end{equation}
where $\mathbf{H}=2\mathbf{X}\mathbf{X}^\top$ is the layer-wise MSE Hessian. To improve efficiency, GPTQ quantizes all rows in the same column order, sharing a single Hessian across rows. With the Cholesky factorization $\mathbf{H}^{-1} = \mathbf{L}\mathbf{L}^\top$, the update extends to all rows as
\begin{equation}
    \Delta\mathbf{W}_{:,t:}=-\frac{\mathbf{W}_{:,t}-\hat{\mathbf{W}}_{:,t}}{\mathbf{L}^{\top}_{t,t}}\mathbf{L}^{\top}_{t,t:}
    =-\mathbf{E}_{:,t}\mathbf{L}^{\top}_{t,t:},
\end{equation}
where $\mathbf{E}\in\mathbb{R}^{d_{\text{row}}\times d_{\text{col}}}$ accumulates the per-column quantization errors. To reduce memory-bandwidth pressure, GPTQ defers updates to non-current columns via a lazy-batch strategy: columns are grouped into batches of size $B$, and the remaining columns $R$ are updated only after all $B$ columns in the current batch $Q$ have been quantized, yielding $\Delta\mathbf{W}_{:,R} = -\mathbf{E}_{:,Q}\mathbf{L}^{\top}_{Q,R}$.

\paragraph{FOEM: First-order Gradient Compensation.}
GPTQ derives the weight update by assuming model convergence, which justifies omitting the first-order gradient term. Formally, it solves
\begin{equation}
    \operatorname{argmin}_{\Delta\mathbf{w}}\left(\frac{1}{2}\Delta\mathbf{w}\mathbf{H}\Delta\mathbf{w}^\top\right),\quad \text{s.t.}\ \mathbf{e}_t\Delta \mathbf{w}^\top+\mathbf{w}_t-\hat{\mathbf{w}}_t=0.
\end{equation}
However, recent studies~\citep{chee2025discquant,hu2025identifying,zheng2026first} have demonstrated that the first-order term plays a pivotal role in quantization accuracy. FOEM~\citep{zheng2026first} incorporates this term and reformulates the problem as
\begin{equation}
    \label{eq:optim_prob}
    \operatorname{argmin}_{\Delta\mathbf{w}}\left(\mathbf{g}\Delta\mathbf{w}^\top+\frac{1}{2}\Delta\mathbf{w}\mathbf{H}\Delta\mathbf{w}^\top\right),\quad \text{s.t.}\ \mathbf{e}_t\Delta \mathbf{w}^\top+\mathbf{w}_t-\hat{\mathbf{w}}_t=0,
\end{equation}
where $\mathbf{g}$ is the gradient of the loss with respect to the weight row. The closed-form solution yields
\begin{equation}
\label{eq:grad_weight_update}
    \Delta\mathbf{w}=\underbrace{\frac{\hat{\mathbf{w}}_t-\mathbf{w}_t}{\mathbf{H}^{-1}_{t,t}} \mathbf{H}^{-1}_{t,:}}_{\Delta \mathbf{w}_{\text{GPTQ}}} + \underbrace{\left[ \frac{\mathbf{g} \mathbf{H}^{-1}\mathbf{e}_t^\top}{\mathbf{H}^{-1}_{t,t}} \mathbf{H}^{-1}_{t,:} - \mathbf{g} \mathbf{H}^{-1} \right]}_{\Delta \mathbf{w}_{\text{grad}}}.
\end{equation}
When $\mathbf{g}=\mathbf{0}$, FOEM reduces to GPTQ (Equation~\ref{eq:hessian_weight_update}). To evaluate $\mathbf{g}\mathbf{H}^{-1}$ efficiently, FOEM approximates $\mathbf{g}$ via a first-order Taylor expansion:
\begin{equation}
    \mathbf{g}\approx(\mathbf{w}-\mathbf{w}_{\text{orig}})\mathbf{H},
\end{equation}
where $\mathbf{w}_{\text{orig}}$ denotes the original full-precision weights. This approximation is exact under the layer-wise MSE objective (Equation~\ref{eq:mse_loss}), but introduces unnecessary error under more expressive objectives such as block-wise MSE or KL divergence. As a result, FOEM remains confined to the same local, layer-wise objective as GPTQ.

\paragraph{GuidedQuant: Global-objective Hessian.}
Layer-wise MSE aligns individual linear layer outputs but may yield suboptimal global performance. Recent work~\citep{edalati2025oac,kim2025guidedquant} addresses this by quantizing weights under end-to-end supervision. GuidedQuant~\citep{kim2025guidedquant} replaces the layer-wise MSE Hessian in Equation~\ref{eq:hessian_weight_update} with one derived from an end-to-end NLL or KL divergence loss, approximated via the Fisher information matrix. Assuming independence across output channels, the Hessian for the $j$-th output channel is
\begin{equation}
\label{eq:hessian_fisher}
\begin{aligned}
    \mathbf{H}^{(j)}&=\frac{1}{n}\sum_{i=1}^n\left(\frac{\partial\ell_i}{\partial\mathbf{W}_{j,:}}\right)^\top\left(\frac{\partial\ell_i}{\partial\mathbf{W}_{j,:}}\right)
    =\frac{1}{n}\mathbf{X}\ \operatorname{Diag}\!\left(\frac{\partial\boldsymbol\ell}{\partial\mathbf{Z}_{j,:}}\right)^{\!2}\mathbf{X}^\top,
\end{aligned}
\end{equation}
where $\ell_i$ is the end-to-end NLL loss for the $i$-th token, $\boldsymbol\ell=(\ell_1,\ldots,\ell_n)\in\mathbb{R}^n$, and $\mathbf{Z}\in\mathbb{R}^{d_{\text{row}}\times n}$ denotes the layer output activations. To reduce memory overhead, GuidedQuant partitions the output channels into $g$ groups and shares a single averaged Hessian within each group:
\begin{equation}
\label{eq:avg_hessian}
    \bar{\mathbf{H}}^{(i)}=\frac{1}{|\mathcal C_i|}\sum_{j\in \mathcal C_i}\mathbf{H}^{(j)},
\end{equation}
where $\mathcal C_i$ denotes the set of channel indices in the $i$-th group. However, GuidedQuant computes this Hessian once prior to quantization and keeps it fixed throughout, causing the estimates to become increasingly stale as weights are progressively updated. Furthermore, GuidedQuant does not incorporate the first-order gradient term, leaving $\Delta\mathbf{w}_{\text{grad}}$ in Equation~\ref{eq:grad_weight_update} unused.

These two lines of work offer complementary strengths and suffer from complementary weaknesses: FOEM captures first-order information but is confined to a local objective, while GuidedQuant employs a more expressive objective but discards first-order information and relies on stale Hessian estimates. We address both limitations simultaneously in Section~\ref{sec:method}.

%% file: texes/method.tex
\section{Method}
\label{sec:method}
\subsection{Generalized Gradient Compensation}
\label{sec:optimal_framework}

We propose \name, which addresses both limitations of prior work by computing the first- and second-order information under a block-wise supervision objective. For each Transformer block, we minimize the discrepancy between the outputs of the quantized and full-precision blocks. 
 Specifically, let $\hat{\mathbf h}^{(l)}$ and $\mathbf h^{(l)}\in\mathbb R^d$ denote the output hidden states of the $l$-th block under the quantized and full-precision weights, respectively, given input activation $\mathbf x$. For the $l$-th Transformer block ($l=1,\ldots,L$), we solve the following optimization problem:
\begin{equation}
    \operatorname{argmin}_{\hat{\mathbf W}}\;
    \ell_l\bigl(\hat{\mathbf h}^{(l)},\,\mathbf h^{(l)}\bigr), \quad
    \ell_l =
    \begin{cases}
        \ell_{\text{MSE}}(\hat{\mathbf h}^{(l)}, \mathbf h^{(l)}) = \frac{1}{2d}\|\hat{\mathbf h}^{(l)} - \mathbf h^{(l)}\|_2^2, & l < L, \\
        D_{\mathrm{KL}}\!\left(p_{\mathbf W}(\cdot|\mathbf x) \,\|\, p_{\hat{\mathbf W}}(\cdot|\mathbf x)\right), & l = L.
    \end{cases}
    \label{eq:block_obj}
\end{equation}
The two forms of objectives enjoy both efficient computation and stronger supervision than prior layer-wise MSE objectives~\citep{frantaroptq}.
Block-wise MSE captures nonlinear interactions within each Transformer block, where the Hessian can be readily approximated with Theorem~\ref{thm:hessian_mse}. The KL divergence for the final block directly targets the model's predictive distribution, with Hessian computed via Theorem~\ref{thm:hessian_kl}.

Another key challenge for any weight compensation method is that the gradient and Hessian should reflect the current state of the model. End-to-end methods~\citep{edalati2025oac,kim2025guidedquant,tseng2025model} sidestep this by computing the Hessian once via a full-model backward pass and keeping it fixed throughout, which is practical but means later blocks are guided by estimates that no longer match the partially quantized model. Instead, we obtain fresh gradient and Hessian estimates before quantizing each block, with one backward pass through that block alone for each estimate. This is far cheaper than two separate full-model backward passes, yet ensures the compensation signal remains accurate throughout quantization.

\paragraph{Hessian Approximation.}
To compute the weight update in Equation~\ref{eq:grad_weight_update} under the block-wise objectives in Equation~\ref{eq:block_obj}, we need efficient expressions for both the gradient $\mathbf g$ and the Hessian $\mathbf H$.
Unlike the layer-wise MSE loss whose Hessian admits a simple closed form $\mathbf H = 2\mathbf X\mathbf X^\top$, the block-wise objectives introduced above involve compositions of nonlinear transformations, making their exact Hessians intractable to compute. We therefore seek efficient approximations. The following two theorems show that the Hessians of both the block-wise MSE and KL divergence losses can be approximated by the outer product of gradient vectors, which can be computed efficiently via standard backpropagation. Full proofs are provided in Appendix~\ref{apdx:hessian_aprox}.

\begin{theorem}
\label{thm:hessian_mse}
    The Hessian of the block-wise MSE loss can be approximated as
    \begin{equation}
        \mathbf H_{\text{MSE}}\approx \mathbb{E}_{\mathbf x \sim q(\mathbf x)}\mathbb{E}_{ \boldsymbol\epsilon \sim \mathcal N(\mathbf 0,d\mathbf I)} \left[  \nabla_{\hat{\mathbf W}}  \{\ell_\text{MSE}(\mathbf h,\mathbf h+\boldsymbol\epsilon)\}^\top \nabla_{\hat{\mathbf W}}  \{\ell_\text{MSE}(\mathbf h,\mathbf h+\boldsymbol\epsilon)\} \right],
    \end{equation}
    where the model input $\mathbf x$ is sampled from the data distribution $q(\mathbf x)$, and the noise $\boldsymbol{\epsilon}\in\mathbb R^d$ is sampled from the Gaussian distribution $\mathcal N(\mathbf 0,d\mathbf I)$, with $d$ being the hidden dimension. $\mathbf h\in\mathbb R^d$ is the block-wise output hidden state of the quantized model. The block-wise MSE loss is defined as $\ell_\text{MSE}(\mathbf h,\mathbf y)=\frac{1}{2d}\|\mathbf h-\mathbf y\|_2^2$. The Hessian matrix can be approximated by the outer product of the gradient vectors computed from the block-wise MSE loss.
\end{theorem}

\begin{theorem}
\label{thm:hessian_kl}
    The Hessian of the KL divergence loss can be approximated as
    \begin{equation}
        \mathbf H_{\text{KL}}\approx \mathbb{E}_{\mathbf x \sim q(\mathbf x)}\mathbb{E}_{ y \sim p_{\mathbf W}( y|\mathbf x)} \left[  \nabla_{\hat{\mathbf W}}  \{- \log p_{\hat{\mathbf W}}( y |\mathbf x)\}^\top \nabla_{\hat{\mathbf W}}  \{-\log p_{\hat{\mathbf W}}(  y |\mathbf x)\} \right],
    \end{equation}
    where the model input $\mathbf x$ is sampled from the data distribution $q(\mathbf x)$, and the label $y$ is sampled from the output distribution of the full-precision model $p_{\mathbf W}( y|\mathbf x)$. The Hessian matrix can be approximated by the outer product of the gradient vectors computed from the NLL loss.
\end{theorem}

Computing and storing the full Hessian matrix for the entire weight matrix is still prohibitively expensive. Recent studies~\citep{zhang2024adam} find that Hessian matrices in LLMs exhibit block-diagonal structures, indicating that weights in different output channels are approximately independent. Therefore, we follow prior work~\citep{edalati2025oac} to employ separate Hessians for different output channels. To further improve efficiency, we follow GuidedQuant~\citep{kim2025guidedquant} to partition the output channels into $g$ groups, and reuse a single Hessian matrix for all channels within the same group as formalized in Equation~\ref{eq:avg_hessian}. With the block-diagonal approximation, the Hessian in Equation~\ref{eq:hessian_fisher} can be computed efficiently via one single backward pass. We leave the details in Appendix~\ref{apdx:hessian_aprox}.

\paragraph{Trust-region Gradient Compensation.}
While the Hessian can be efficiently approximated as described above, the gradient $\mathbf g$ in Equation~\ref{eq:grad_weight_update} requires careful treatment. FOEM~\citep{zheng2026first} uses a first-order Taylor expansion to approximate gradients for computational efficiency. This approximation is exact for the layer-wise MSE loss, but introduces unnecessary error for commonly used objectives such as block-wise MSE~\citep{sun2024flatquant} and KL divergence~\citep{hu2025ostquant}. A formal proof is provided in Appendix~\ref{apdx:grad_aprox}. In contrast, we avoid this approximation and instead compute the gradient exactly via backpropagation, with an efficient implementation detailed in Section~\ref{sec:effi_impl}. However, the exact first-order update term $\Delta \mathbf w_{\text{grad}}$ in Equation~\ref{eq:grad_weight_update} is highly sensitive to the numerical scales of the gradient and Hessian, which can easily result in exploding weight updates (see the analysis in Appendix~\ref{apdx:exploding_first_order}). Directly applying the exact first-order compensation often produces excessively large weight update steps that violate the local validity of the Taylor approximation in Equation~\ref{eq:optim_prob}, ultimately leading to severe performance degradation. To resolve this instability, we propose a trust-region method that dynamically scales the gradient compensation step $\Delta \mathbf w_{\text{grad}}$ to restrict the resulting loss change to a predefined budget. This guarantees that the update remains strictly within the locally valid neighborhood. The computation of this scaling factor is formalized in Theorem~\ref{thm:trust_region_beta}, with the full derivation and implementation details deferred to Appendix~\ref{apdx:trust_region_derivation}.
\begin{theorem}[Trust-Region Scaling Factor]
\label{thm:trust_region_beta}
    To bound the estimated change in the final loss within a trust region, the first-order update term $\Delta \mathbf{w}_{\text{grad}}$ is scaled by a factor $\beta \in [0, 1]$ given by
    \begin{equation}
    \beta = 1 - \sqrt{\max\left(1 - \frac{2\alpha \cdot |\Delta L_{\text{GPTQ}}|}{c}, 0\right)},
    \end{equation}
    where 
    $\alpha>0$ is a hyperparameter controlling the strength of gradient compensation, and $\Delta L_{\text{GPTQ}}$ is the estimated loss change induced by $\Delta \mathbf{w}_{\text{GPTQ}}$. The scalar $c$ and $\Delta L_{\text{GPTQ}}$ are defined as
    \begin{equation}
    c=\mathbf{g}\mathbf{H}^{-1}\mathbf{g}^\top - \frac{(\mathbf{g}\mathbf{H}^{-1}\mathbf{e}_t^\top)^2}{\mathbf{H}^{-1}_{t,t}} \geq 0,\quad
    \Delta L_{\text{GPTQ}}=\frac{\hat{\mathbf w}_t - \mathbf w_t}{\mathbf{H}^{-1}_{t,t}} (\mathbf{g} \mathbf{H}^{-1} \mathbf{e}_t^\top) + \frac{1}{2} \frac{(\hat{\mathbf w}_t - \mathbf w_t)^2}{\mathbf{H}^{-1}_{t,t}}.
    \end{equation}
\end{theorem}
Combining first- and second-order information with the trust-region scaling introduced in Theorem~\ref{thm:trust_region_beta}, the weight update is given by
\begin{equation}
    \Delta\mathbf w=\underbrace{\frac{\mathbf {\hat w}_t-\mathbf w_t}{\mathbf{H}^{-1}_{t,t}} \mathbf H^{-1}_{t,:}}_{\Delta \mathbf w_{\text{GPTQ}}} + \beta\cdot\underbrace{\left[ \frac{\mathbf g \mathbf H^{-1}\mathbf e_t^\top}{\mathbf{H}^{-1}_{t,t}} \mathbf H^{-1}_{t,:} - \mathbf g \mathbf H^{-1} \right]}_{\Delta \mathbf w_{\text{grad}}}.
\end{equation}

\subsection{An Efficient Implementation}
\label{sec:effi_impl}

\begin{algorithm}[t]
\begingroup
\setlength{\heavyrulewidth}{1.2pt}
\setlength{\lightrulewidth}{0.5pt}
\setlength{\aboverulesep}{1pt}
\setlength{\belowrulesep}{1.4pt}
\begin{tabular}{@{}p{0.98\textwidth}@{}}
\toprule
\textbf{Algorithm 1:} \name\ quantization for one linear layer \\
\midrule\hspace*{6pt}%
\begin{minipage}{\dimexpr\linewidth-12pt\relax}
\textbf{Input:} Weight matrix $\mathbf{W}$, gradient $\mathbf{G}$, per-group Hessians $\{\bar{\mathbf{H}}^{(i)}\}_{i=0}^{g-1}$ (Equation~\ref{eq:avg_hessian}), row-wise trust-region scaling $\boldsymbol\beta\in\mathbb R^{d_{\text{row}}}$ (Theorem~\ref{thm:trust_region_beta}), block size $B$, and \#channel groups $g$ \\
\textbf{Output:} quantized weight $\hat{\mathbf{W}}$ \\[3pt]
\phantom{0}1:\quad $\hat{\mathbf{W}},\mathbf E \leftarrow \mathbf 0_{d_{\text{row}}\times d_{\text{col}}},\mathbf 0_{d_{\text{row}}\times d_{\text{col}}}$ \\
\phantom{0}2:\quad \highlight{$\mathbf D \leftarrow \text{diag}(B-1,B-2,\ldots,0)$}\\
\phantom{0}3:\quad \textbf{for} each channel group $i = 0,\ldots,g-1$ \textbf{do} \\
\phantom{0}4:\quad \quad $\mathcal C_i \leftarrow i\cdot(d_{\text{row}}/g):(i+1)\cdot(d_{\text{row}}/g)$ \\
\phantom{0}5:\quad \quad $\mathbf W^{(i)},\hat{\mathbf W}^{(i)},\mathbf G^{(i)},\mathbf E^{(i)},\boldsymbol\beta^{(i)} \leftarrow \mathbf W_{\mathcal C_i,:}, \hat{\mathbf W}_{\mathcal C_i,:}, \mathbf G_{\mathcal C_i,:}, \mathbf E_{\mathcal C_i,:}, \boldsymbol\beta_{\mathcal C_i}$ \\
\phantom{0}6:\quad \quad $\mathbf{L} \leftarrow \text{Cholesky}\big((\bar{\mathbf{H}}^{(i)})^{-1}\big)$ \\
\phantom{0}7:\quad \quad \highlight{$\mathbf{M} \leftarrow \left({\boldsymbol\beta^{(i)}}\right)^\top \circ (\mathbf{G}^{(i)}\mathbf{L})$} \\
\phantom{0}8:\quad \quad \highlight{$\mathbf{F} \leftarrow \mathbf{M}\mathbf{L}^\top$} \hfill \textit{ $\triangleright$ (Equation~\ref{eq:recursive_f})} \\
\phantom{0}9:\quad \quad \textbf{for} $b=0,B,2B,\ldots$ \textbf{do} \\
10:\quad \quad \quad \textbf{for} $t=b,b+1,\ldots,b+B-1$ \textbf{do} \\
11:\quad \quad \quad \quad $\hat{\mathbf{W}}^{(i)}_{:,t} \leftarrow \text{Quantize}(\mathbf{W}^{(i)}_{:,t})$ \\
12:\quad \quad \quad \quad $\mathbf E_{:,t}^{(i)} \leftarrow ({\mathbf{W}^{(i)}_{:,t} - \mathbf{\hat{W}}^{(i)}_{:,t} \highlight{- \mathbf{F}_{:,t}}})/{\mathbf{L}^\top_{t,t}}$ \\
13:\quad \quad \quad \quad $\mathbf{W}_{:,t:(b+B)}^{(i)} \leftarrow \mathbf{W}^{(i)}_{:,t:(b+B)} - \mathbf E_{:,t}^{(i)} \mathbf{L}^\top_{t,t:(b+B)} \highlight{- \mathbf F_{:,t:(b+B)}}$ \hfill \textit{ $\triangleright$ (Equation~\ref{eq:g2ptq_weight_update})} \\
14:\quad \quad \quad \quad \highlight{$\mathbf{F}_{:,t:(b+B)} \leftarrow \mathbf{F}_{:,t:(b+B)} - \mathbf{M}_{:,t}\,\mathbf{L}^\top_{t,t:(b+B)}$} \hfill \textit{ $\triangleright$ (Equation~\ref{eq:recursive_f})} \\
15:\quad \quad \quad \textbf{end for} \\
16:\quad \quad \quad $\mathbf{W}_{:,(b+B):}^{(i)} \leftarrow \mathbf{W}_{:,(b+B):}^{(i)} - \mathbf{E}^{(i)}_{:,b:(b+B)}\,\mathbf{L}^\top_{b:(b+B),(b+B):} \highlight{- B\cdot\mathbf{F}_{:,(b+B):}}$ \\
\phantom{00:\quad}\quad \quad \quad \quad \quad \quad \quad \quad \quad $\highlight{+ \mathbf{M}_{:,b:(b+B)}\,\mathbf{D}\,\mathbf{L}^\top_{b:(b+B),(b+B):}}$ \\
17:\quad \quad \quad $\highlight{\mathbf{F}_{:,(b+B):} \leftarrow \mathbf{F}_{:,(b+B):} - \mathbf{M}_{:,b:(b+B)}\,\mathbf{L}^\top_{b:(b+B),(b+B):}}$ \hfill \textit{ $\triangleright$ (Equation~\ref{eq:lazy_batch})} \\
18:\quad \quad \textbf{end for} \\
19:\quad \textbf{end for}

\vspace{2pt}
\end{minipage}\\
\bottomrule
\end{tabular}
\endgroup
\captionsetup{labelformat=empty,labelsep=none,textformat=empty}
\caption{}
\vspace{-\baselineskip}
\label{alg:g2ptq}
\end{algorithm}

Algorithm~\ref{alg:g2ptq} presents the pseudocode of \name\ for quantizing a single linear layer, highlighting the differences from GPTQ in blue. The pseudocode for the entire model is provided in Appendix~\ref{apdx:pseudocode}.

\paragraph{Efficient Weight Update.}
At the $t$-th quantization step, where $t\in\{0,\ldots,d_{\text{col}}-1\}$, we quantize the $t$-th weight column and update the remaining unquantized columns. With the Cholesky reformulation $\mathbf{H}^{-1} = \mathbf{L}\mathbf{L}^\top$, the weight update can be written as
\begin{equation}
\label{eq:g2ptq_weight_update}
    \Delta\mathbf W_{:,t:}={\frac{\mathbf {\hat W}_{:,t}-\mathbf W_{:,t}}{\mathbf{L}^{\top}_{t,t}} \mathbf{L}^{\top}_{t,t:}} + \boldsymbol\beta^\top\circ{\left[ \frac{\mathbf F^{(t)}_{:,t}}{\mathbf{L}^{\top}_{t,t}} \mathbf{L}^{\top}_{t,t:} - \mathbf F^{(t)}_{:,t:} \right]},
\end{equation}
where $\boldsymbol\beta\in\mathbb R^{d_{\text{row}}}$ is the row-wise trust-region scaling factor detailed in Appendix~\ref{apdx:trust_region_derivation}, and $\circ$ denotes the Hadamard product. Here, $\mathbf F^{(t)}\in\mathbb R^{d_{\text{row}}\times d_{\text{col}}}$ denotes the matrix $\mathbf G_{:,t:}(\mathbf H_{t:,t:})^{-1}\in \mathbb R^{d_{\text{row}}\times (d_{\text{col}}-t)}$ left-padded to size $\mathbb R^{d_{\text{row}}\times d_{\text{col}}}$, and $\mathbf G\in\mathbb R^{d_{\text{row}}\times d_{\text{col}}}$ is the gradient of the weight matrix. The main challenge comes from computing $\mathbf F^{(t)}$. Naively computing $\mathbf F^{(t)}$ requires inverting $\mathbf H_{t:,t:}$ and multiplying by $\mathbf G_{:,t:}$. Repeating this process for $d_{\text{col}}$ steps incurs a complexity of $\mathcal O\!\left(\max\left\{d_{\text{col}}^4,\; d_{\text{row}} \cdot d_{\text{col}}^3\right\}\right)$, which is prohibitively expensive for LLMs. Instead, we propose to compute $\mathbf F^{(t)}$ recursively with
\begin{equation}
\label{eq:recursive_f}
\begin{aligned}
    \mathbf{F}^{(0)} = \mathbf{M}\mathbf{L}^\top, \quad
    \mathbf{F}^{(t+1)} = \mathbf{F}^{(t)} - \mathbf{M}_{:,t} \mathbf{L}^\top_{t,:},
\end{aligned}  
\end{equation}
where $\mathbf{M} = \mathbf{G}\mathbf{L} \in \mathbb{R}^{d_{\text{row}} \times d_{\text{col}}}$. The proof is deferred to Appendix~\ref{apdx:g2ptq_update}. Recursively updating $\mathbf F^{(t)}$ with Equation~\ref{eq:recursive_f} at each step reduces the complexity to $\mathcal O\!\left(\max\left\{d_{\text{col}}^3,\; d_{\text{row}} \cdot d_{\text{col}}^2\right\}\right)$, which matches the complexity of the vanilla GPTQ algorithm, making it practical for large-scale models.

\paragraph{Lazy-batch Update.}
To further improve memory bandwidth utilization and accelerate quantization, we adopt a lazy-batch update strategy following GPTQ~\citep{frantaroptq}. Instead of updating all remaining unquantized weight columns at each step, we group the columns into blocks and defer the corresponding updates. Let $Q = \{q_{\text{start}}, \dots, q_{\text{end}}\}$ denote a block of $B$ columns to be quantized, and $R$ denote the remaining unquantized columns. For simplicity, we omit the trust-region scaling factor $\boldsymbol\beta$, which can be absorbed into $\mathbf{M}$ with $\boldsymbol\beta\circ\mathbf{M}$. We define the weight update error at step $t$ as \begin{equation}
\label{eq:error_define}
    \mathbf{E}_{:,t} = \frac{\mathbf{W}_{:,t} - \mathbf{\hat{W}}_{:,t} - \mathbf{F}^{(t)}_{:,t}}{\mathbf{L}^\top_{t,t}}.
\end{equation}
Within the block $Q$, the columns are quantized sequentially, and the internal state $\mathbf{F}^{(t)}_{:,Q}$ is updated locally. Once the block is completed, the remaining columns $R$ for both $\mathbf{F}$ and $\mathbf{W}$ can be updated in a highly efficient batched manner:
\begin{equation}
\label{eq:lazy_batch}
\begin{aligned}
    \mathbf{F}^{(q_{\text{end}}+1)}_{:,R} &= \mathbf{F}^{(q_{\text{start}})}_{:,R} - \mathbf{M}_{:,Q} \mathbf{L}^\top_{Q,R}, \\
    \Delta\mathbf{W}_{:,R} &= - \left( \mathbf{E}_{:,Q} \mathbf{L}^\top_{Q,R} + B \cdot \mathbf{F}^{(q_{\text{start}})}_{:,R} - \mathbf{M}_{:,Q} \mathbf{D} \mathbf{L}^\top_{Q,R} \right),
\end{aligned}
\end{equation}
where $\mathbf{D} = \text{diag}(B-1, B-2, \dots, 0) \in \mathbb{R}^{B \times B}$. The detailed derivation is provided in Appendix~\ref{apdx:g2ptq_update}. Additionally, we implement efficient kernels for the lazy-batch update, with details in Appendix~\ref{apdx:perf_optim}.

%% file: texes/experiments.tex
\section{Experiments}
\label{sec:experiments}

\subsection{Experimental Settings}
\label{sec:exp_settings}
\paragraph{Models and Benchmarks.}
To comprehensively assess \name, we evaluate it across a diverse array of model families, comprising 13 dense models and 2 mixture-of-experts (MoE) models that range in size from 0.6B to 125B. Further details regarding the model coverage are provided in Appendix~\ref{apdx:model}.
For the downstream evaluation, we report perplexity (PPL) and KL divergence, consistent with prior work~\citep{frantaroptq,tseng2025model}. As discussed in Appendix~\ref{apdx:kl}, KL divergence serves as a reliable proxy for quantifying quantization-induced degradation. Additionally, we report the accuracy across seven commonsense QA benchmarks. Comprehensive details regarding these benchmarks are deferred to Appendix~\ref{apdx:benchmark}.

\paragraph{Quantization Settings and Baselines.}
We apply symmetric per-channel and per-token quantization for the weights and activations, respectively. We compare \name\ against RTN~\citep{nagel2021white} and several representative state-of-the-art weight quantization baselines, including GPTQ~\citep{frantaroptq}, GuidedQuant~\citep{kim2025guidedquant}, and GPTAQ~\citep{li2025gptaq}. To ensure a fair comparison, all methods employ a uniform scalar quantizer alongside established techniques such as rotation~\citep{ashkboos2024quarot}, clipping~\citep{frantaroptq,ashkboos2024quarot}, and activation reordering~\citep{frantaroptq}. Further details are provided in Appendix~\ref{apdx:qsettings_baseline}.

\paragraph{Implementation Details.}
We autotune the trust-region threshold $\alpha$ in Theorem~\ref{thm:trust_region_beta} with grid search, deferring the details to Appendix~\ref{apdx:trust-region-autotune}. As described in Section~\ref{sec:method}, \name\ refreshes the first- and second-order information once per Transformer block. We further implement \name$^*$ with the true sequential strategy~\citep{frantaroptq}, which partitions the Transformer block quantization into four stages and updates the information prior to each stage. More details regarding quantization efficiency optimization, including kernel fusion and distributed quantization, are provided in Appendix~\ref{apdx:perf_optim}.

\subsection{Main Results}
\input{tables/main/weight}
\input{tables/main/wa_moe}

\paragraph{Results on Weight-only Quantization.}
We evaluate \name\ for weight-only quantization across various model families and bit-widths. As shown in Table~\ref{tab:w_only}, the \name\ variants consistently achieve the best performance across all evaluated metrics. The improvements are most pronounced under the W2A16 setting: averaged across the three models, \name$^*$ halves the KL divergence compared to GPTAQ (from 1.83 to 0.90), translating to a 6.45\% increase in QA accuracy. This advantage persists at higher bit-widths, where \name\ further narrows the performance gap with the full-precision models. These results demonstrate the broad applicability of \name\ across diverse model architectures and bit-widths. Additional results for weight-only quantization are provided in Appendix~\ref{apdx:weight_only}.

\paragraph{Results on Weight-activation Quantization.}
We further investigate the performance of \name\ under aggressive 4-bit weight-activation quantization. It is well established that quantization error in this setting is typically dominated by activations due to the presence of severe activation outliers~\citep{frantaroptq, liu2024spinquant}. As shown in Table~\ref{tab:wa}, \name\ substantially outperforms all baselines on the LLaMA3-70B model, improving upon the QA accuracy of GPTAQ by 4.56\%. This indicates that \name\ can effectively compensate for accumulated activation quantization errors through appropriate weight updates. Additional weight-activation quantization results for other models are provided in Appendix~\ref{apdx:weight_activation}.

\paragraph{Results on MoE Quantization.}
Given that MoE has become a standard component in modern large-scale language models, evaluating the effectiveness of \name\ on MoE architectures is of significant interest. In Table~\ref{tab:moe}, we compare \name\ against GPTQ on two MoE models, Qwen3-30B-A3B and Qwen3.8-Flash-Next, under 4-bit quantization. Although GPTQ already delivers competitive results that closely approximate the full-precision model, \name\ further improves accuracy. For instance, on the 125B Qwen3.8-Flash-Next model, \name\ incurs only a 0.33\% decrease in QA accuracy, demonstrating its strong generalization capabilities to large-scale MoE models. Results at additional bit-widths and evaluations on reasoning benchmarks are deferred to Appendix~\ref{apdx:moe}.

\subsection{Discussions}
\label{sec:discussions}
\input{tables/ablations/ablations}

\paragraph{Ablation Study.}
To validate the individual components of \name, we conduct an ablation study on the Qwen3-0.6B model under 3-bit quantization in Table~\ref{tab:ablation_methods}. Implementing block-wise alignment significantly improves upon the layer-wise GPTQ baseline, increasing the average QA accuracy from 33.09\% to 37.40\%. Dynamically refreshing this information further enhances accuracy by yielding more precise Hessian estimations. Crucially, this refreshing mechanism serves as a strict prerequisite for gradient compensation; without it, the initial full-precision weights would reside at a local minimum where the gradients are inherently zero. Incorporating both information refreshing and exact gradient compensation yields an additional 1.75\% increase in QA accuracy. Finally, applying the KL divergence loss to the final Transformer block directly aligns the output distribution space, providing further accuracy gains. Furthermore, as theoretically analyzed in Section~\ref{sec:method}, trust-region scaling is critical to ensure that the exact gradient compensation step does not produce excessively large weight updates that would violate the local validity of the Taylor approximation. As demonstrated in Table~\ref{tab:ablation_methods}, omitting this scaling causes the exact gradient step to diverge, resulting in a catastrophic collapse of downstream accuracy. More ablation studies on the calibration set size, the number of output channel groups, and the exact gradient compensation are deferred to Appendix~\ref{apdx:ablations}.



\paragraph{Calibration Resources.}
During calibration, \name\ performs backward passes through individual Transformer blocks to collect first- and second-order information, which could consume higher memory and computational costs than GPTQ. Table~\ref{tab:memory_time} summarizes the memory usage and runtime of GPTQ and \name. GPTQ is highly efficient since it requires only a single forward pass to estimate the Hessian matrices. Although \name\ and \name$^*$ require additional backward passes, both remain practical in deployment. Specifically, we perform backpropagation in a block-wise manner and do not require end-to-end backpropagation, unlike prior methods~\citep{edalati2025oac,kim2025guidedquant,tseng2025model}. It takes only 2.24 hours and 12.54~GB of memory to quantize an 8B model on a single accelerator. We also provide the calibration resources required for quantizing a 100B-parameter model in Appendix~\ref{apdx:discussions}.

Due to limited space, we leave further discussions in Appendix~\ref{apdx:discussions}, including a once-for-all autotuning strategy and Hessian-based weight clipping.

%% file: tables/main/weight.tex
\begin{table*}[!t]
\begin{center}
\resizebox{0.9\linewidth}{!}
{
\begin{tabular}{c|l|ccc|ccc|ccc}
\toprule
\multirow{2}{*}{\textbf{Model}}
& \multirow{2}{*}{\textbf{Method}}
& \multicolumn{3}{c|}{\textbf{W2A16}}
& \multicolumn{3}{c|}{\textbf{W3A16}}
& \multicolumn{3}{c}{\textbf{W4A16}} \\
\cmidrule(lr){3-5}
\cmidrule(lr){6-8}
\cmidrule(lr){9-11}
& & \textbf{KL} & \textbf{PPL} & \textbf{QA}
  & \textbf{KL} & \textbf{PPL} & \textbf{QA}
  & \textbf{KL} & \textbf{PPL} & \textbf{QA} \\
\midrule

\multirow{7}{*}{
\begin{tabular}{@{}c@{}}
\textbf{Qwen3} \\
\textbf{0.6B}
\end{tabular}}
& BF16
& 0 & 11.28 & 49.28
& 0 & 11.28 & 49.28
& 0 & 11.28 & 49.28 \\
\cmidrule{2-11}
& RTN
& 1.43e+01 & 7.56e+06 & 28.95
& 7.44e+00 & 1.62e+04 & 29.68
& 5.84e-01 & 17.61 & 41.36 \\
& GPTQ
& 4.73e+00 & 1.58e+03 & 29.92
& 1.50e+00 & 55.60 & 33.09
& 4.58e-01 & 18.33 & 41.57 \\
& GuidedQ
& 2.46e+00 & 143.67 & 30.05
& 4.96e-01 & 17.76 & 38.28
& 1.42e-01 & 12.88 & 45.41 \\
& GPTAQ
& 3.33e+00 & 364.96 & 28.61
& 8.92e-01 & 30.40 & 34.58
& 3.03e-01 & 15.85 & 42.85 \\
& \cellcolor{g2ptqblue!12}\name
& \cellcolor{g2ptqblue!12}1.55e+00
& \cellcolor{g2ptqblue!12}64.99
& \cellcolor{g2ptqblue!12}31.72
& \cellcolor{g2ptqblue!12}3.87e-01
& \cellcolor{g2ptqblue!12}16.92
& \cellcolor{g2ptqblue!12}39.22
& \cellcolor{g2ptqblue!12}1.17e-01
& \cellcolor{g2ptqblue!12}12.63
& \cellcolor{g2ptqblue!12}\textbf{45.85} \\
& \cellcolor{g2ptqblue!12}\name$^{*}$
& \cellcolor{g2ptqblue!12}\textbf{1.26e+00}
& \cellcolor{g2ptqblue!12}\textbf{53.25}
& \cellcolor{g2ptqblue!12}\textbf{32.93}
& \cellcolor{g2ptqblue!12}\textbf{3.33e-01}
& \cellcolor{g2ptqblue!12}\textbf{16.12}
& \cellcolor{g2ptqblue!12}\textbf{39.83}
& \cellcolor{g2ptqblue!12}\textbf{1.05e-01}
& \cellcolor{g2ptqblue!12}\textbf{12.52}
& \cellcolor{g2ptqblue!12}45.80 \\
\midrule

\multirow{7}{*}{
\begin{tabular}{@{}c@{}}
\textbf{Qwen3.5} \\
\textbf{2B}
\end{tabular}}
& BF16
& 0 & 6.58 & 59.61
& 0 & 6.58 & 59.61
& 0 & 6.58 & 59.61 \\
\cmidrule{2-11}
& RTN
& 9.32e+00 & 5.45e+04 & 28.82
& 2.24e+00 & 84.76 & 38.09
& 1.92e-01 & 7.90 & 56.13 \\
& GPTQ
& 1.42e+00 & 39.95 & 32.18
& 2.28e-01 & 8.70 & 51.29
& 6.36e-02 & 7.00 & 56.57 \\
& GuidedQ
& 1.62e+00 & 51.08 & 32.69
& 2.38e-01 & 8.72 & 51.37
& 6.49e-02 & 7.06 & 57.67 \\
& GPTAQ
& 1.33e+00 & 38.15 & 31.50
& 2.15e-01 & 8.58 & 50.01
& 6.14e-02 & 7.07 & 56.69 \\
& \cellcolor{g2ptqblue!12}\name
& \cellcolor{g2ptqblue!12}1.05e+00
& \cellcolor{g2ptqblue!12}25.06
& \cellcolor{g2ptqblue!12}36.09
& \cellcolor{g2ptqblue!12}1.99e-01
& \cellcolor{g2ptqblue!12}8.42
& \cellcolor{g2ptqblue!12}51.87
& \cellcolor{g2ptqblue!12}5.83e-02
& \cellcolor{g2ptqblue!12}\textbf{6.98}
& \cellcolor{g2ptqblue!12}57.22 \\
& \cellcolor{g2ptqblue!12}\name$^{*}$
& \cellcolor{g2ptqblue!12}\textbf{8.42e-01}
& \cellcolor{g2ptqblue!12}\textbf{19.99}
& \cellcolor{g2ptqblue!12}\textbf{37.75}
& \cellcolor{g2ptqblue!12}\textbf{1.77e-01}
& \cellcolor{g2ptqblue!12}\textbf{8.20}
& \cellcolor{g2ptqblue!12}\textbf{53.53}
& \cellcolor{g2ptqblue!12}\textbf{5.37e-02}
& \cellcolor{g2ptqblue!12}\textbf{6.98}
& \cellcolor{g2ptqblue!12}\textbf{57.97} \\
\midrule

\multirow{7}{*}{
\begin{tabular}{@{}c@{}}
\textbf{LLaMA3} \\
\textbf{8B}
\end{tabular}}
& BF16
& 0 & 4.09 & 69.75
& 0 & 4.09 & 69.75
& 0 & 4.09 & 69.75 \\
\cmidrule{2-11}
& RTN
& 1.28e+01 & 1.48e+06 & 29.07
& 1.42e+00 & 23.77 & 41.60
& 1.45e-01 & 4.77 & 66.03 \\
& GPTQ
& 1.15e+00 & 18.11 & 38.61
& 1.69e-01 & 4.99 & 64.63
& 4.99e-02 & 4.33 & 68.44 \\
& GuidedQ
& 9.00e-01 & 13.03 & 43.16
& 1.40e-01 & 4.83 & 65.99
& 4.05e-02 & 4.29 & 68.75 \\
& GPTAQ
& 8.36e-01 & 12.60 & 42.38
& 1.44e-01 & 4.87 & 65.68
& 4.28e-02 & 4.30 & 68.72 \\
& \cellcolor{g2ptqblue!12}\name
& \cellcolor{g2ptqblue!12}8.65e-01
& \cellcolor{g2ptqblue!12}14.76
& \cellcolor{g2ptqblue!12}47.47
& \cellcolor{g2ptqblue!12}1.29e-01
& \cellcolor{g2ptqblue!12}4.78
& \cellcolor{g2ptqblue!12}66.07
& \cellcolor{g2ptqblue!12}3.94e-02
& \cellcolor{g2ptqblue!12}4.28
& \cellcolor{g2ptqblue!12}69.06 \\
& \cellcolor{g2ptqblue!12}\name$^{*}$
& \cellcolor{g2ptqblue!12}\textbf{6.12e-01}
& \cellcolor{g2ptqblue!12}\textbf{9.18}
& \cellcolor{g2ptqblue!12}\textbf{51.16}
& \cellcolor{g2ptqblue!12}\textbf{1.22e-01}
& \cellcolor{g2ptqblue!12}\textbf{4.75}
& \cellcolor{g2ptqblue!12}\textbf{66.56}
& \cellcolor{g2ptqblue!12}\textbf{3.75e-02}
& \cellcolor{g2ptqblue!12}\textbf{4.27}
& \cellcolor{g2ptqblue!12}\textbf{69.41} \\
\bottomrule
\end{tabular}
}
\end{center}
\vspace{-1ex}
\caption{Weight-only quantization results across different model families. KL, PPL, and QA metrics are reported as averages. Results for more models are provided in Appendix~\ref{apdx:weight_only}.}
\label{tab:w_only}
\end{table*}

%% file: tables/main/wa_moe.tex
\begin{figure*}[t]
    \centering

    \begin{minipage}[t]{0.48\textwidth}
        \centering
        \vspace{0pt}

        \resizebox{\linewidth}{!}{
            \begin{tabular}{c|l|ccc}
                \toprule
                \textbf{Model}
                & \textbf{Method}
                & \textbf{KL}
                & \textbf{PPL}
                & \textbf{QA} \\
                \midrule

                \multirow{7}{*}{
                \begin{tabular}{@{}c@{}}
                    \textbf{LLaMA3} \\
                    \textbf{70B}
                \end{tabular}}
                & BF16
                & 0
                & 2.77
                & 77.77 \\
                \cmidrule{2-5}
                & RTN
                & 1.83e+00
                & 21.03
                & 34.27 \\
                & GPTQ
                & 6.33e-01
                & 5.60
                & 58.08 \\
                & GuidedQ
                & 5.69e-01
                & 5.19
                & 60.39 \\
                & GPTAQ
                & 3.73e-01
                & 4.27
                & 68.38 \\
                & \cellcolor{g2ptqblue!12}\name
                & \cellcolor{g2ptqblue!12}\textbf{2.77e-01}
                & \cellcolor{g2ptqblue!12}\textbf{3.79}
                & \cellcolor{g2ptqblue!12}\textbf{72.94} \\
                & \cellcolor{g2ptqblue!12}\name\textsuperscript{*}
                & \cellcolor{g2ptqblue!12}3.33e-01
                & \cellcolor{g2ptqblue!12}4.08
                & \cellcolor{g2ptqblue!12}71.76 \\
                \bottomrule
            \end{tabular}
        }

        \captionsetup{type=table,skip=7pt}
        \caption{4-bit weight-activation quantization results. KL, PPL, and QA metrics are reported as averages. Results for more models are provided in Appendix~\ref{apdx:weight_activation}.}
        \label{tab:wa}
    \end{minipage}
    \hfill
    \begin{minipage}[t]{0.48\textwidth}
        \centering
        \vspace{0pt}

        \resizebox{0.95\linewidth}{!}{
            \begin{tabular}{c|l|ccc}
                \toprule
                \textbf{Model}
                & \textbf{Method}
                & \textbf{KL}
                & \textbf{PPL}
                & \textbf{QA} \\
                \midrule

                \multirow{4}{*}{
                \begin{tabular}{@{}c@{}}
                    \textbf{Qwen3} \\
                    \textbf{30B-A3B}
                \end{tabular}}
                & BF16
                & 0
                & 5.63
                & 73.45 \\
                \cmidrule{2-5}
                & GPTQ
                & 3.64e-02
                & 5.71
                & 71.86 \\
                & \cellcolor{g2ptqblue!12}\name
                & \cellcolor{g2ptqblue!12}2.22e-02
                & \cellcolor{g2ptqblue!12}\textbf{5.66}
                & \cellcolor{g2ptqblue!12}\textbf{72.44} \\
                & \cellcolor{g2ptqblue!12}\name\textsuperscript{*}
                & \cellcolor{g2ptqblue!12}\textbf{2.09e-02}
                & \cellcolor{g2ptqblue!12}5.67
                & \cellcolor{g2ptqblue!12}72.43 \\
                \midrule

                \multirow{3}{*}{
                \begin{tabular}{@{}c@{}}
                    \textbf{Qwen3.8} \\
                    \textbf{Flash-Next}
                \end{tabular}}
                & BF16
                & 0
                & 4.70
                & 78.28 \\
                \cmidrule{2-5}
                & GPTQ
                & 1.16e-01
                & 5.00
                & 77.55 \\
                & \cellcolor{g2ptqblue!12}\name\textsuperscript{*}
                & \cellcolor{g2ptqblue!12}\textbf{9.61e-02}
                & \cellcolor{g2ptqblue!12}\textbf{4.55}
                & \cellcolor{g2ptqblue!12}\textbf{77.95} \\
                \bottomrule
            \end{tabular}
        }

        \captionsetup{type=table,skip=5pt}
        \caption{4-bit weight-only quantization results on MoE models. KL, PPL, and QA metrics are reported as averages. Results for more bit-widths are provided in Appendix~\ref{apdx:moe}}
        \label{tab:moe}
    \end{minipage}
\end{figure*}

%% file: tables/ablations/ablations.tex
\begin{figure*}[t]
    \centering

    \newlength{\panelheight}
    \setlength{\panelheight}{0.19\textheight}

    \begin{minipage}[t]{0.495\textwidth}
        \centering

        \begin{minipage}[c][\panelheight][c]{\linewidth}
            \centering
            \resizebox{\linewidth}{!}{
                \begin{tabular}{l|ccc}
                    \toprule
                    \textbf{Method} & \textbf{KL} & \textbf{PPL} & \textbf{QA} \\
                    \midrule
                    GPTQ              & 1.50e+00 & 55.60    & 33.09 \\
                    \rowcolor{g2ptqblue!12} + Block-wise Align & 5.07e-01 & 18.73    & 37.40 \\
                    \rowcolor{g2ptqblue!12} + Refreshed Info   & 4.51e-01 & 17.27    & 38.61 \\
                    \rowcolor{g2ptqblue!12} + Exact Grad       & 3.63e-01 & 17.48    & 39.15 \\
                    \rowcolor{g2ptqblue!12} + Final KL Align               & 3.33e-01 & 16.12    & 39.83 \\
                    \rowcolor{gray!20} - Trust-region Scaling  & 2.53e+01 & 1.29e+12 & 28.97 \\
                    \bottomrule
                \end{tabular}
            }
        \end{minipage}

        \vspace{-3.5ex}
        \captionsetup{type=table, skip=0pt}
        \caption{Ablation study of \name\ on Qwen3-0.6B 3-bit weight-only quantization. The KL, PPL, and QA metrics are reported as averages.}
        \label{tab:ablation_methods}
    \end{minipage}
    \hfill
    \begin{minipage}[t]{0.465\textwidth}
        \centering

        \begin{minipage}[c][\panelheight][c]{\linewidth}
            \centering
            \resizebox{0.92\linewidth}{!}{
            \begin{tabular}{c|c|cc}
                \toprule
                \textbf{Model Size} & \textbf{Method} & \textbf{Mem (GB)} & \textbf{Time (H)} \\ 
                \midrule
                \multirow{3}{*}{0.6B} 
                & GPTQ & 0.43 & 0.08 \\
                & \name & 8.56 & 0.18 \\
                & \name\textsuperscript{*} & 8.36 & 0.33 \\
                \midrule
                \multirow{3}{*}{8B} 
                & GPTQ & 4.62 & 0.43 \\
                & \name & 14.57 & 1.49 \\
                & \name\textsuperscript{*} & 12.54 & 2.24 \\
                \bottomrule
            \end{tabular}
            }
        \end{minipage}

        \vspace{-3.5ex}
        \captionsetup{type=table, skip=0pt}
        \caption{Comparison of calibration memory consumption and runtime for GPTQ and \name\ on Qwen3-0.6B and 8B models.}
        \label{tab:memory_time}
    \end{minipage}

    \vspace{-1.5ex}
\end{figure*}

%% file: texes/conclusions.tex
\section{Conclusions}
In this paper, we introduce \name, a PTQ framework that incorporates refreshed first- and second-order information for quantization error compensation under block-wise supervision, offering global and accurate guidance for weight updates without fine-tuning. Moreover, we theoretically derive efficient implementations for the block-wise Hessian approximation and exact gradient compensation. Extensive experiments across various model families and bit-widths demonstrate that \name\ achieves better alignment with the full-precision model than the state-of-the-art baselines.

%% file: texes/appendix.tex
\clearpage
\appendix

\section{Related Work}
\label{apdx:related_work}
This section reviews two categories of outliers in LLMs and complements the discussion in Section~\ref{sec:related_work}.

\paragraph{Outlier Channels.}
In LLMs, input activation outliers are typically concentrated in a small number of fixed channels~\citep{dettmers2022gpt3,xiao2023smoothquant}, which substantially increases the difficulty of quantization. Early work~\citep{dettmers2022gpt3} achieves near-lossless 8-bit quantization accuracy by retaining outlier channels in high precision. However, such fine-grained mixed-precision strategies often reduce inference efficiency. Recent studies~\citep{xiao2023smoothquant,ashkboos2024quarot,liu2024spinquant,sun2024flatquant,van2025fptquant} instead use function-preserving transformations to smooth outlier channels, thereby improving both quantization accuracy and efficiency. SmoothQuant~\citep{xiao2023smoothquant} introduces per-channel scaling to shift quantization difficulty from activations to weights. QuaRot~\citep{ashkboos2024quarot} employs Hadamard transforms to distribute outliers across channels and achieves a breakthrough in 4-bit weight-activation quantization. SpinQuant~\citep{liu2024spinquant} uses Cayley reparameterization to learn orthogonal transformations. FlatQuant~\citep{sun2024flatquant} further improves accuracy through Kronecker-decomposed affine transformations. FPTQuant~\citep{van2025fptquant} maximizes transformation expressivity while ensuring that the transformations can be merged into model weights, achieving results competitive with FlatQuant while incurring lower inference overhead.

\paragraph{Massive Activations.}
Massive activations~\citep{sun2024massive,liu2024intactkv} are the outliers on some pivotal tokens that have orders of magnitude larger than other activations. These activations are closely related to attention sinks~\citep{xiao2023efficient}. Distortions in the representations of such pivotal tokens can severely degrade LLM accuracy. IntactKV~\citep{liu2024intactkv} preserves the first few tokens of the system prompt in a lossless representation, serving as a zero-overhead plug-in that improves quantization accuracy across different quantization settings. RotateKV~\citep{su2025rotatekv} identifies pivotal tokens based on activation magnitudes and retains them in high precision for KV cache quantization. More recently, new model architectures have been developed to produce LLMs without massive activations. A representative example is gated attention~\citep{qiu2025gated}, which eliminates massive activations by introducing input-dependent sparsity into the output of the attention module, enabling more stable model training and higher accuracy.

\section{Theoretical Derivation}
\label{apdx:derivations}

\subsection{Hessian Approximation for KL divergence and MSE Loss}
\label{apdx:hessian_aprox}

\paragraph{Proof of Theorem~\ref{thm:hessian_mse}.}
Let the block-wise MSE loss be defined as $\ell_{\text{MSE}} (\mathbf h, \mathbf y) = \frac{1}{2d}\|\mathbf h - \mathbf y\|_2^2$, where $\mathbf h \in \mathbb R^d$ is the output hidden state of the quantized model, $\mathbf y \in \mathbb R^d$ is the target, and $d$ is the hidden dimension. Let $\mathbf J = \frac{\partial \mathbf h}{\partial \hat{\mathbf W}}$ be the Jacobian of the hidden state with respect to the quantized weights $\hat{\mathbf W}$. The Hessian of the loss with respect to the weights $\hat{\mathbf W}$ can be approximated using the Generalized Gauss-Newton (GGN) approximation, yielding:
\begin{equation}
    \mathbf H_{\text{MSE}} \approx \mathbb{E}_{\mathbf x \sim q(\mathbf x)} \left[ \mathbf J^\top \nabla_{\mathbf h}^2 \ell_{\text{MSE}} \mathbf J \right].
\end{equation}
First, we compute the second derivative of the MSE loss with respect to the model output $\mathbf h$:
\begin{equation}
    \nabla_{\mathbf h}^2 \ell_{\text{MSE}} = \nabla_{\mathbf h}^2 \left( \frac{1}{2d}\|\mathbf h - \mathbf y\|_2^2 \right) = \frac{1}{d}\mathbf I.
\end{equation}
Substituting this into the GGN approximation, we obtain the approximated Hessian:
\begin{equation}
    \mathbf H_{\text{MSE}} \approx \mathbb{E}_{\mathbf x \sim q(\mathbf x)} \left[ \mathbf J^\top \left(\frac{1}{d}\mathbf I\right) \mathbf J \right] = \frac{1}{d} \mathbb{E}_{\mathbf x \sim q(\mathbf x)} [\mathbf J^\top \mathbf J].
\end{equation}
Note that this Hessian approximation is independent of the target label $\mathbf y$. Next, we compute the first derivative of the MSE loss with respect to the model weights $\hat{\mathbf W}$:
\begin{equation}
    \nabla_{\hat{\mathbf W}} \ell_{\text{MSE}} = \nabla_{\mathbf h} \ell_{\text{MSE}}\mathbf J  = \left( \frac{1}{d}(\mathbf h - \mathbf y) \right)\mathbf J = -\frac{1}{d} \boldsymbol\epsilon\mathbf J,
\end{equation}
where $\boldsymbol\epsilon = \mathbf y - \mathbf h\in\mathbb R^d$ is the output error. The expectation of the outer product of the first-order gradients under the output error distribution can be written as:
\begin{equation}
\label{eq:mse_hessian_1}
\begin{aligned}
    \mathbb{E}_{\mathbf x \sim q(\mathbf x)}\mathbb{E}_{ \boldsymbol\epsilon} &\left[  \nabla_{\hat{\mathbf W}}  \{\ell_\text{MSE}(\mathbf h,\mathbf h+\boldsymbol\epsilon)\}^\top \nabla_{\hat{\mathbf W}}  \{\ell_\text{MSE}(\mathbf h,\mathbf h+\boldsymbol\epsilon)\} \right] \\
    &= \mathbb{E}_{\mathbf x \sim q(\mathbf x)}\mathbb{E}_{ \boldsymbol\epsilon} \left[ \left( -\frac{1}{d} \boldsymbol\epsilon\mathbf J \right)^\top \left( -\frac{1}{d} \boldsymbol\epsilon\mathbf J \right) \right] \\
    &= \frac{1}{d^2} \mathbb{E}_{\mathbf x \sim q(\mathbf x)}\mathbb{E}_{ \boldsymbol\epsilon} \left[ \mathbf J^\top \boldsymbol\epsilon^\top \boldsymbol\epsilon \mathbf J \right] \\
    &= \frac{1}{d^2} \mathbb{E}_{\mathbf x \sim q(\mathbf x)} \left[ \mathbf J^\top \mathbb{E}_{ \boldsymbol\epsilon} [\boldsymbol\epsilon^\top \boldsymbol\epsilon] \mathbf J \right].
\end{aligned}
\end{equation}
If the output error $\boldsymbol\epsilon$ is sampled from a Gaussian distribution $\mathcal N(\mathbf 0, d\mathbf I)$, then the covariance matrix of the noise is $\mathbb{E}_{\boldsymbol\epsilon}[\boldsymbol\epsilon^\top \boldsymbol\epsilon] = d\mathbf I$. Substituting this back into Equation~\ref{eq:mse_hessian_1} yields:
\begin{equation}
\begin{aligned}
    \frac{1}{d^2} \mathbb{E}_{\mathbf x \sim q(\mathbf x)} \left[ \mathbf J^\top (d\mathbf I) \mathbf J \right] &= \frac{1}{d} \mathbb{E}_{\mathbf x \sim q(\mathbf x)} [\mathbf J^\top \mathbf J] \approx \mathbf H_{\text{MSE}}.
\end{aligned}
\end{equation}
Thus, the Hessian of the MSE loss can be effectively approximated by the expected outer product of the gradients when the target is augmented with Gaussian noise $\boldsymbol\epsilon \sim \mathcal N(\mathbf 0, d\mathbf I)$:
\begin{equation}
\mathbf H_{\text{MSE}}\approx \mathbb{E}_{\mathbf x \sim q(\mathbf x)}\mathbb{E}_{ \boldsymbol\epsilon \sim \mathcal N(\mathbf 0,d\mathbf I)} \left[  \nabla_{\hat{\mathbf W}}  \{\ell_\text{MSE}(\mathbf h,\mathbf h+\boldsymbol\epsilon)\}^\top \nabla_{\hat{\mathbf W}}  \{\ell_\text{MSE}(\mathbf h,\mathbf h+\boldsymbol\epsilon)\} \right].
\end{equation}

\paragraph{Proof of Theorem~\ref{thm:hessian_kl}.}
The Hessian of the KL divergence loss with respect to the quantized model weights $\hat{\mathbf W}$ is defined as:
\begin{equation}
    \mathbf H_{\text{KL}} = \mathbb{E}_{\mathbf x \sim q(\mathbf x)} \left[ \nabla_{\hat{\mathbf W}}^2 D_\text{KL}(p_{\mathbf W}(\cdot|\mathbf x) \| p_{\hat{\mathbf W}}(\cdot|\mathbf x)) \right],
\end{equation}
where the model input $\mathbf x$ is sampled from the data distribution $q(\mathbf x)$. Expanding the KL divergence, we can write the Hessian as:
\begin{equation}
\label{eq:kl_hessian_1}
\begin{aligned}
    \mathbf H_{\text{KL}} &= \mathbb{E}_{\mathbf x \sim q(\mathbf x)} \left[ -\nabla_{\hat{\mathbf W}}^2 \sum_{y \in \mathcal{C}} p_{\mathbf W}(y|\mathbf x) \log \frac{p_{\hat{\mathbf W}}(y|\mathbf x)}{p_{\mathbf W}(y|\mathbf x)} \right] \\
    &= \mathbb{E}_{\mathbf x \sim q(\mathbf x)} \left[ - \sum_{y \in \mathcal{C}} p_{\mathbf W}(y|\mathbf x) \nabla_{\hat{\mathbf W}}^2 \log p_{\hat{\mathbf W}}(y|\mathbf x) \right] \\
    &= \mathbb{E}_{\mathbf x \sim q(\mathbf x)} \mathbb{E}_{y \sim p_{\mathbf W}(y|\mathbf x)} \left[ -\nabla_{\hat{\mathbf W}}^2 \log p_{\hat{\mathbf W}}(y|\mathbf x) \right],
\end{aligned}
\end{equation}
where $\mathcal{C}$ denotes the set of all possible tokens. Next, we expand the second derivative of the negative log-likelihood inside the expectation. The Hessian of the log probability is given by:
\begin{equation}
\begin{aligned}
    -\nabla_{\hat{\mathbf W}}^2 \log p_{\hat{\mathbf W}}(y \mid \mathbf x) &= -\nabla_{\hat{\mathbf W}} \left( \frac{\nabla_{\hat{\mathbf W}} p_{\hat{\mathbf W}}(y \mid \mathbf x)}{p_{\hat{\mathbf W}}(y \mid \mathbf x)} \right) \\
    &= -\frac{p_{\hat{\mathbf W}}(y \mid \mathbf x)\nabla_{\hat{\mathbf W}}^2 p_{\hat{\mathbf W}}(y \mid \mathbf x) - \nabla_{\hat{\mathbf W}} p_{\hat{\mathbf W}}(y \mid \mathbf x)^\top \nabla_{\hat{\mathbf W}} p_{\hat{\mathbf W}}(y \mid \mathbf x)}{p_{\hat{\mathbf W}}(y \mid \mathbf x)^2} \\
    &= -\frac{\nabla_{\hat{\mathbf W}}^2 p_{\hat{\mathbf W}}(y \mid \mathbf x)}{p_{\hat{\mathbf W}}(y \mid \mathbf x)} + \nabla_{\hat{\mathbf W}} \log p_{\hat{\mathbf W}}(y \mid \mathbf x)^\top \nabla_{\hat{\mathbf W}} \log p_{\hat{\mathbf W}}(y \mid \mathbf x).
\end{aligned}
\end{equation}
Substituting this expansion back into Equation~\ref{eq:kl_hessian_1}, we obtain:
\begin{equation}
\begin{aligned}
    \mathbf H_{\text{KL}} &= \mathbb{E}_{\mathbf x \sim q(\mathbf x)} \mathbb{E}_{y \sim p_{\mathbf W}(y \mid \mathbf x)} \left[ -\frac{\nabla_{\hat{\mathbf W}}^2 p_{\hat{\mathbf W}}(y \mid \mathbf x)}{p_{\hat{\mathbf W}}(y \mid \mathbf x)} \right] \\
    &\quad + \mathbb{E}_{\mathbf x \sim q(\mathbf x)} \mathbb{E}_{y \sim p_{\mathbf W}(y \mid \mathbf x)} \left[ \nabla_{\hat{\mathbf W}} \{-\log p_{\hat{\mathbf W}}(y \mid \mathbf x)\}^\top \nabla_{\hat{\mathbf W}} \{-\log p_{\hat{\mathbf W}}(y \mid \mathbf x)\} \right].
\end{aligned}
\end{equation}
To simplify the first term, we assume that the quantized weights are close to the full-precision weights, i.e., $\hat{\mathbf W} \approx \mathbf W$. Under this assumption, the output distributions are approximately equal:
\begin{equation}
    p_{\mathbf W}(y \mid \mathbf x) \approx p_{\hat{\mathbf W}}(y \mid \mathbf x).
\end{equation}
We can then approximate the inner expectation of the first term as:
\begin{equation}
\begin{aligned}
    \mathbb{E}_{y \sim p_{\mathbf W}(y \mid \mathbf x)} \left[ -\frac{\nabla_{\hat{\mathbf W}}^2 p_{\hat{\mathbf W}}(y \mid \mathbf x)}{p_{\hat{\mathbf W}}(y \mid \mathbf x)} \right] &\approx \sum_{y \in \mathcal{C}} p_{\hat{\mathbf W}}(y \mid \mathbf x) \left( -\frac{\nabla_{\hat{\mathbf W}}^2 p_{\hat{\mathbf W}}(y \mid \mathbf x)}{p_{\hat{\mathbf W}}(y \mid \mathbf x)} \right) \\
    &= -\nabla_{\hat{\mathbf W}}^2 \sum_{y \in \mathcal{C}} p_{\hat{\mathbf W}}(y \mid \mathbf x) \\
    &= -\nabla_{\hat{\mathbf W}}^2 (1) = \mathbf 0.
\end{aligned}
\end{equation}
Since the first term vanishes, the Hessian of the KL divergence can be approximated by the Fisher information matrix, where the inputs are sampled from the data distribution and the labels are sampled from the full-precision model:
\begin{equation}
    \mathbf H_{\text{KL}} \approx \mathbb{E}_{\mathbf x \sim q(\mathbf x)}\mathbb{E}_{ y \sim p_{\mathbf W}( y|\mathbf x)} \left[  \nabla_{\hat{\mathbf W}}  \{- \log p_{\hat{\mathbf W}}( y |\mathbf x)\}^\top \nabla_{\hat{\mathbf W}}  \{-\log p_{\hat{\mathbf W}}(  y |\mathbf x)\} \right].
\end{equation}

\paragraph{Practical Hessian Computation.}
The Hessians of both the block-wise MSE and KL divergence loss can be expressed as outer products of the gradient vectors, and can be written as in Equation~\ref{eq:hessian_fisher}, where the tokens are from the same sequence. We then average the Hessians computed from different sequences to obtain the final Hessian matrix. However, computing the exact per-token gradients $\text{Diag}\left({\partial\boldsymbol\ell}/{\partial\mathbf Z_{j,:}}\right)^2$ in Equation~\ref{eq:hessian_fisher} remains computationally expensive, as it requires $n$ independent forward and backward passes. The following theorem shows that these per-token gradients can be efficiently approximated using a single forward and backward pass over the entire sequence with the cumulative loss $\mathcal{L} = \sum_{i=1}^n \ell_i$:
\begin{theorem}[Practical Hessian Computation]
\label{thm:practical_hessian_compute}
    Let $\mathcal{L} = \sum_{i=1}^n \ell_i$ be the cumulative loss, where $\ell_i$ is either the block-wise MSE or the KL divergence loss. Assume that the causal attention mechanism is highly sparse, such that the squared gradients of future tokens' losses with respect to the $k$-th token's activation is negligible (i.e., $\left(\frac{\partial \ell_m}{\partial \mathbf{Z}_{j,k}}\right)^2 \approx 0$ for $m > k$). Then, the expectation of the squared gradient of the cumulative loss approximates the expectation of the squared per-token gradient:
    \begin{equation}
        \mathbb{E} \left[ \left( \frac{\partial \mathcal{L}}{\partial \mathbf{Z}_{j,k}} \right)^2 \right] \approx \mathbb{E} \left[ \left( \frac{\partial \ell_k}{\partial \mathbf{Z}_{j,k}} \right)^2 \right].
    \end{equation}
\end{theorem}
\textit{Proof}. In causal language models, the gradient of the cumulative loss with respect to the $k$-th token's activation is the sum of the gradients from the $k$-th token and all subsequent tokens:
\begin{equation}
    \frac{\partial \mathcal{L}}{\partial \mathbf{Z}_{j,k}} = \sum_{i=1}^n \frac{\partial \ell_i}{\partial \mathbf{Z}_{j,k}} = \sum_{i=k}^n \frac{\partial \ell_i}{\partial \mathbf{Z}_{j,k}}.
\end{equation}
Squaring this aggregated gradient introduces unwanted cross terms and additional squared terms:
\begin{equation}
    \left( \sum_{i=k}^n \frac{\partial \ell_i}{\partial \mathbf{Z}_{j,k}} \right)^2 = \sum_{i=k}^n \left( \frac{\partial \ell_i}{\partial \mathbf{Z}_{j,k}} \right)^2 + 2 \sum_{k \leq m < l \leq n} \frac{\partial \ell_m}{\partial \mathbf{Z}_{j,k}} \frac{\partial \ell_l}{\partial \mathbf{Z}_{j,k}}.
\end{equation}
We first show that the cross terms vanish in expectation. Let $g_m = \frac{\partial \ell_m}{\partial \mathbf{Z}_{j,k}}$ and $g_l = \frac{\partial \ell_l}{\partial \mathbf{Z}_{j,k}}$ for $k \leq m < l \leq n$. Let $\mathcal{F}_{<l}$ denote all the information up to step $l-1$, including the input and any sampled tokens or noise. By the law of total expectation,
\begin{equation}
    \mathbb{E}[g_m g_l] = \mathbb{E}_{<l} \Big[ \mathbb{E}_l [g_mg_l \mid \mathcal{F}_{<l}] \Big] = \mathbb{E}_{<l} \Big[ g_m\mathbb{E}_l [g_l \mid \mathcal{F}_{<l}] \Big],
\end{equation}
where the inner expectation $\mathbb{E}_l$ is taken with respect to the random variable at step $l$ (either the sampled token or the noise). We now analyze this inner expectation for the two objectives. For the KL divergence loss, the expectation is taken over the target tokens $y_l \sim p_{\mathbf{W}}(y \mid \mathbf{x}_{< l})$. Assuming $\hat{\mathbf{W}} \approx \mathbf{W}$, we have
\begin{equation}
\begin{aligned}
    \mathbb{E}_{y_l \sim p_{\mathbf{W}}(y \mid \mathbf{x}_{< l})}[g_l] &\approx \mathbb{E}_{y_l \sim p_{\hat{\mathbf{W}}}(y \mid \mathbf{x}_{< l})} \left[ \frac{\partial}{\partial \mathbf{Z}_{j,k}} \{-\log p_{\hat{\mathbf{W}}}(y_l \mid \mathbf{x}_{< l})\} \right] \\
    &= -\sum_{y_l \in \mathcal{C}} p_{\hat{\mathbf{W}}}(y_l \mid \mathbf{x}_{< l}) \frac{\frac{\partial}{\partial \mathbf{Z}_{j,k}} p_{\hat{\mathbf{W}}}(y_l \mid \mathbf{x}_{< l})}{p_{\hat{\mathbf{W}}}(y_l \mid \mathbf{x}_{< l})} \\
    &= -\frac{\partial}{\partial \mathbf{Z}_{j,k}} \sum_{y_l \in \mathcal{C}} p_{\hat{\mathbf{W}}}(y_l \mid \mathbf{x}_{< l}) \\
    &= -\frac{\partial}{\partial \mathbf{Z}_{j,k}} (1) = 0.
\end{aligned}
\end{equation}
Hence, the cross terms vanish in expectation for the KL divergence loss. For the block-wise MSE loss, the expectation is taken over the independent Gaussian noise $\boldsymbol\epsilon^{(l)} \sim \mathcal{N}(\mathbf{0}, d\mathbf{I})$ injected into the $l$-th token's target. Recall that the gradient of the MSE loss for a single token $l$ is $g_l = -\frac{1}{d} \boldsymbol\epsilon^{(l)} \left(\frac{\partial \mathbf{h}^{(l)}}{\partial \mathbf{Z}_{j,k}}\right)^\top$. Since the noise is zero-mean (i.e., $\mathbb{E}[\boldsymbol\epsilon^{(l)}] = \mathbf{0}$), the conditional expectation evaluates to zero:
\begin{equation}
    \mathbb{E}_{\boldsymbol\epsilon^{(l)}}[g_l \mid \mathcal{F}_{<l}] = -\frac{1}{d} \mathbb{E}[\boldsymbol\epsilon^{(l)}]\left( \frac{\partial \mathbf{h}^{(l)}}{\partial \mathbf{Z}_{j,k}}\right)^\top = 0.
\end{equation}
Therefore, the cross terms also vanish exactly for the MSE loss. Since the cross terms vanish in expectation for both objectives, the expectation of the squared aggregated gradient reduces to the sum of the expectations of the squared per-token gradients:
\begin{equation}
    \mathbb{E} \left[ \left( \sum_{i=k}^n \frac{\partial \ell_i}{\partial \mathbf{Z}_{j,k}} \right)^2 \right] = \sum_{i=k}^n \mathbb{E} \left[ \left( \frac{\partial \ell_i}{\partial \mathbf{Z}_{j,k}} \right)^2 \right].
\end{equation}
It remains to consider the additional squared terms $g_m^2$, where $k < m \leq n$. Intuitively, in causal language models, the $k$-th token influences subsequent tokens through the attention mechanism. For the vast majority of token pairs, the attention scores are highly sparse and close to zero, implying that the gradient signal $g_m$ propagated back to the $k$-th token is minimal. Consequently, $g_m^2$ can be negligible in practice.
Combining the fact that the cross terms vanish in expectation with the assumption that the additional squared terms are negligible, we obtain
\begin{equation}
    \mathbb{E} \left[ \left( \sum_{i=1}^n \frac{\partial \ell_i}{\partial \mathbf{Z}_{j,k}} \right)^2 \right] \approx \mathbb{E} \left[ \left( \frac{\partial \ell_k}{\partial \mathbf{Z}_{j,k}} \right)^2 \right].
\end{equation}

\subsection{Discussions on Gradient Approximation}
\label{apdx:grad_aprox}
\paragraph{Gradient Approximation in FOEM.}
FOEM~\citep{zheng2026first} approximates the gradient with a first-order Taylor expansion:
\begin{equation}
    \mathbf g^{(\mathbf w)}\approx \mathbf g^{(\mathbf w_{\text{orig}})} + (\mathbf w-\mathbf w_{\text{orig}})\mathbf H,
\end{equation}
where $\mathbf g^{(\mathbf w)}, \mathbf g^{(\mathbf w_{\text{orig}})} \in \mathbb R^{d_{\text{col}}}$ denote the gradient with respect to the weight row vector of the current quantized model and the original full-precision model, respectively. Assuming that the original model has converged on the calibration set, we have $\mathbf g^{(\mathbf w_{\text{orig}})} \approx \mathbf 0$, which yields
\begin{equation}
    \mathbf g^{(\mathbf w)} \approx (\mathbf w - \mathbf w_{\text{orig}})\mathbf H.
\end{equation}
When the layer-wise MSE in Equation~\ref{eq:mse_loss} is used as the optimization objective, the gradient can be derived analytically as
\begin{equation}
\begin{aligned}
\label{eq:mse_grad_aprox}
    \nabla_{\mathbf w}\|\mathbf w\mathbf X - \mathbf w_{\text{orig}}\mathbf X\|_2^2
    &= \nabla_{\mathbf w}\|(\mathbf w - \mathbf w_{\text{orig}})\mathbf X\|_2^2 \\
    &= \nabla_{\mathbf w}\left((\mathbf w - \mathbf w_{\text{orig}})\mathbf X\mathbf X^\top(\mathbf w - \mathbf w_{\text{orig}})^\top\right) \\
    &= 2(\mathbf w - \mathbf w_{\text{orig}})\mathbf X\mathbf X^\top \\
    &= (\mathbf w - \mathbf w_{\text{orig}})\mathbf H.
\end{aligned}
\end{equation}
As shown in Equation~\ref{eq:mse_grad_aprox}, this gradient approximation is exact when the layer-wise MSE loss is adopted. Nevertheless, for other loss types, it may be subject to approximation error.

\paragraph{Limitations of the Gradient Approximation.}
The approximation error can stem from both the first-order Taylor expansion and the assumption that the original model has converged on the calibration set, depending on the choice of loss function. For the first-order Taylor approximation to be exact, the loss must be strictly quadratic with respect to the weights $\mathbf w$. This condition does not hold for global objectives such as NLL, KL divergence, and block-wise MSE loss, as the composition of nonlinearities in deep neural networks results in a highly complex, non-quadratic loss landscape. Regarding the local convergence assumption of the original model, it holds perfectly when the loss directly measures the discrepancy between the original and quantized models. However, for task-specific objectives such as NLL loss, this assumption is often violated due to the domain gap between the calibration dataset and the large-scale datasets used during model pre-training~\citep{chee2025discquant}.

\subsection{Analysis of the Exploding First-Order Weight Update}
\label{apdx:exploding_first_order}
\paragraph{Stable Weight Updates in GPTQ and FOEM.}
The weight update rule in the original GPTQ~\citep{frantaroptq} (Equation~\ref{eq:hessian_weight_update}) exhibits a highly desirable numerical property: the magnitude of the update is inherently tied to the quantization error and remains strictly invariant to the scale of the Hessian matrix. Specifically, scaling $\mathbf H$ by any non-zero constant does not alter the final update. Similarly, in FOEM~\citep{zheng2026first}, the first-order update term $\mathbf g\mathbf H^{-1}$ in Equation~\ref{eq:grad_weight_update} is approximated by the weight residual $\mathbf w - \mathbf w_{\text{orig}}$. This formulation ensures that the update magnitude is bounded by the initial weight perturbation, making it robustly scale-invariant with respect to the loss gradients.

\paragraph{Exploding Weight Updates in Exact First-order Compensation.}
As established in Theorem~\ref{thm:hessian_mse} and Theorem~\ref{thm:hessian_kl}, the Hessians of the block-wise MSE and KL divergence losses can be efficiently approximated using the expected outer products of the gradients computed from the block-wise MSE and NLL losses, respectively. Let $\mathbf H=\mathbb E\left[\nabla_{\mathbf{\hat W}}\ell'^\top\nabla_{\mathbf{\hat W}}\ell'\right]$ and $\mathbf g=\mathbb E[\nabla_{\mathbf{\hat W}}\ell]$. If we scale the losses $\ell$ and $\ell'$ by positive scalars $a$ and $b$, respectively, the squared $L_2$ norm of the first-order update step scales by a factor of $a^2b^{-4}$:

\begin{equation}
\begin{aligned}
    \left\| \mathbb{E}[\nabla_{\mathbf{\hat W}}(a\ell)] \left( \mathbb{E}\left[\nabla_{\mathbf{\hat W}}(b\ell')^\top\nabla_{\mathbf{\hat W}}(b\ell')\right] \right)^{-1} \right\|_2^2 
    &= \left\| a\mathbb{E}[\nabla_{\mathbf{\hat W}}\ell] \left( b^2\mathbb{E}\left[\nabla_{\mathbf{\hat W}}\ell'^\top\nabla_{\mathbf{\hat W}}\ell'\right] \right)^{-1} \right\|_2^2 \\
    &= a^2b^{-4} \left\| \mathbf{g}\mathbf{H}^{-1} \right\|_2^2.
\end{aligned}
\end{equation}

Unlike the scale-invariant updates in GPTQ and FOEM, this exact first-order compensation mechanism is highly sensitive to the relative scales of $\ell$ and $\ell'$. As dictated by the derived $a^2b^{-4}$ multiplier, the update magnitude is unbounded. Because these loss scales can vary dramatically across different model architectures and quantization configurations, this unbounded variation can severely destabilize the optimization process. Most critically, the presence of the $b^{-4}$ term indicates that the update norm is acutely vulnerable to the scale of $\ell'$. When the gradients of $\ell'$ are smaller in magnitude than those of $\ell$ (effectively yielding a small $b$ relative to $a$), the $b^{-4}$ factor rapidly inflates the inverse Hessian term, severely amplifying the update step and leading to catastrophic divergence. 

\subsection{Derivation of the Trust-region Scaling Factor}
\label{apdx:trust_region_derivation}
\paragraph{Proof of Theorem~\ref{thm:trust_region_beta}.}
To address the exploding first-order weight update issue described in Appendix~\ref{apdx:exploding_first_order}, we constrain the exact gradient compensation step $\Delta\mathbf{w}_{\text{grad}}$ in Equation~\ref{eq:grad_weight_update} by introducing a scaling factor $\beta \in [0, 1]$. The total weight update at a single quantization step is defined as $\Delta \mathbf{w} = \Delta \mathbf{w}_{\text{GPTQ}} + \beta \Delta \mathbf{w}_{\text{grad}}$. Under the second-order Taylor approximation, the resulting total loss change $\Delta L$ can be decomposed into three components:
\begin{equation}
\begin{aligned}
    \Delta L &= \mathbf{g}(\Delta \mathbf{w}_{\text{GPTQ}} + \beta \Delta \mathbf{w}_{\text{grad}})^\top + \frac{1}{2} (\Delta \mathbf{w}_{\text{GPTQ}} + \beta \Delta \mathbf{w}_{\text{grad}}) \mathbf{H} (\Delta \mathbf{w}_{\text{GPTQ}} + \beta \Delta \mathbf{w}_{\text{grad}})^\top \\
    &= \underbrace{\left( \mathbf{g}\Delta \mathbf{w}_{\text{GPTQ}}^\top + \frac{1}{2} \Delta \mathbf{w}_{\text{GPTQ}} \mathbf{H} \Delta \mathbf{w}_{\text{GPTQ}}^\top \right)}_{\Delta L_{\text{GPTQ}}} + \underbrace{\left( \beta \mathbf{g}\Delta \mathbf{w}_{\text{grad}}^\top + \frac{1}{2} \beta^2 \Delta \mathbf{w}_{\text{grad}} \mathbf{H} \Delta \mathbf{w}_{\text{grad}}^\top \right)}_{\Delta L_{\text{grad}}} 
    \\&\ \ \ \ + \underbrace{\beta \Delta \mathbf{w}_{\text{GPTQ}} \mathbf{H} \Delta \mathbf{w}_{\text{grad}}^\top}_{\text{Cross Term}}.
\end{aligned}
\end{equation}

By substituting the explicit solutions of $\Delta \mathbf{w}_{\text{GPTQ}}$ and $\Delta \mathbf{w}_{\text{grad}}$ in Equation~\ref{eq:grad_weight_update}, we can show that the cross term is strictly zero:
\begin{equation}
\begin{aligned}
    \Delta \mathbf{w}_{\text{GPTQ}} \mathbf{H} \Delta \mathbf{w}_{\text{grad}}^\top &= \left( \frac{\hat{\mathbf w}_t - \mathbf w_t}{\mathbf{H}^{-1}_{t,t}} \mathbf{e}_t \mathbf{H}^{-1} \right) \mathbf{H} \left( \frac{\mathbf{g} \mathbf{H}^{-1} \mathbf{e}_t^\top}{\mathbf{H}^{-1}_{t,t}} \mathbf{H}^{-1} \mathbf{e}_t^\top - \mathbf{H}^{-1} \mathbf{g}^\top \right) \\
    &= \frac{\hat{\mathbf w}_t - \mathbf w_t}{\mathbf{H}^{-1}_{t,t}} \mathbf{e}_t \left( \frac{\mathbf{g} \mathbf{H}^{-1} \mathbf{e}_t^\top}{\mathbf{H}^{-1}_{t,t}} \mathbf{H}^{-1} \mathbf{e}_t^\top - \mathbf{H}^{-1} \mathbf{g}^\top \right) \\
    &= \frac{\hat{\mathbf w}_t - \mathbf w_t}{\mathbf{H}^{-1}_{t,t}} \left( \mathbf{g} \mathbf{H}^{-1} \mathbf{e}_t^\top - \mathbf{e}_t \mathbf{H}^{-1} \mathbf{g}^\top \right) \\
    &= 0.
\end{aligned}
\end{equation}
Similarly, $\Delta L_{\text{GPTQ}}$ and $\Delta L_{\text{grad}}$ can be computed as
\begin{equation}
\begin{aligned}
    \Delta L_{\text{GPTQ}} = \frac{\hat{\mathbf w}_t - \mathbf w_t}{\mathbf{H}^{-1}_{t,t}} (\mathbf{g} \mathbf{H}^{-1} \mathbf{e}_t^\top) + \frac{1}{2} \frac{(\hat{\mathbf w}_t - \mathbf w_t)^2}{\mathbf{H}^{-1}_{t,t}},\quad
    \Delta L_{\text{grad}} = -c\beta + \frac{1}{2}c\beta^2,
\end{aligned}
\end{equation}
where $c$ is a non-negative scalar defined as
\begin{equation}
    c = \mathbf{g}\mathbf{H}^{-1}\mathbf{g}^\top - \frac{(\mathbf{g}\mathbf{H}^{-1}\mathbf{e}_t^\top)^2}{\mathbf{H}^{-1}_{t,t}} \geq 0.
\end{equation}
To ensure stability, we constrain the magnitude of $\Delta L_{\text{grad}}$ to be bounded by a budget proportional to the magnitude of $\Delta L_{\text{GPTQ}}$:
\begin{equation}
    \left| \Delta L_{\text{grad}} \right| \leq \alpha \cdot \left| \Delta L_{\text{GPTQ}} \right|.
\end{equation}
Given $\beta \in [0, 1]$ and $c \geq 0$, $\Delta L_{\text{grad}}$ is strictly non-positive, which yields
\begin{equation}
    c\beta^2 - 2c\beta + 2\alpha \cdot \left| \Delta L_{\text{GPTQ}} \right| \geq 0.
\end{equation}
Solving for $\beta$ under the constraint $\beta \in [0, 1]$ gives:
\begin{equation}
    \beta \leq 1 - \sqrt{\max\left(1 - \frac{2\alpha \cdot \left| \Delta L_{\text{GPTQ}} \right|}{c}, 0\right)}.
\end{equation}
We choose the largest feasible $\beta$ as the scaling factor to minimize the loss. When $c \leq 2\alpha \cdot \left| \Delta L_{\text{GPTQ}} \right|$, the unscaled gradient update naturally falls within the trust region, and no scaling is required (i.e., $\beta = 1$).

\paragraph{Bounding the Total Loss Change.}
Given $\Delta L = \Delta L_{\text{GPTQ}} + \Delta L_{\text{grad}}$ and $\left| \Delta L_{\text{grad}} \right| \leq \alpha \cdot \left| \Delta L_{\text{GPTQ}} \right|$, we can bound the total loss change within the following interval:
\begin{equation}
    \Delta L_{\text{GPTQ}} - \alpha \left| \Delta L_{\text{GPTQ}} \right| \leq \Delta L \leq \Delta L_{\text{GPTQ}}.
\end{equation}
This guarantees that the total loss increment after gradient compensation is at most equal to the original GPTQ update. Ideally, it reduces the GPTQ loss increment by up to a ratio of $\alpha$, ensuring the theoretical stability of the algorithm.

\paragraph{Implementation Details.}
For an efficient parallel implementation across all rows of the weight matrix, we compute the scaling factors in matrix form. Let $\mathbf{W}, \hat{\mathbf{W}} \in \mathbb{R}^{d_{\text{row}}\times d_{\text{col}}}$ denote the original and quantized weight matrices, respectively, and define the quantization error matrix as $\Delta \mathbf{W} = \hat{\mathbf{W}} - \mathbf{W}$. Let $\mathbf{G}\in\mathbb R^{d_{\text{row}}\times d_{\text{col}}}$ denote the gradient matrix.
The matrix $\Delta\mathbf{L}_{\text{GPTQ}} \in \mathbb{R}^{d_{\text{row}}\times d_{\text{col}}}$, comprising the pre-computed GPTQ loss changes $\Delta L_{\text{GPTQ}}$, is formulated as:
\begin{equation}
    \Delta\mathbf{L}_{\text{GPTQ}} = \left( \Delta \mathbf{W} \circ (\mathbf{G}\mathbf{H}^{-1}) + \frac{1}{2} (\Delta \mathbf{W})^{\circ 2} \right) \oslash \text{diag}(\mathbf{H}^{-1}),
\end{equation}
where $\circ$ denotes the Hadamard product and $(\cdot)^{\circ 2}$ denotes the element-wise square.
Similarly, the matrix $\mathbf{C} \in \mathbb{R}^{d_{\text{row}}\times d_{\text{col}}}$, containing the scalars $c$ for all weights, is computed as:
\begin{equation}
    \mathbf{C} = (\mathbf{G}\mathbf{H}^{-1} \circ \mathbf{G})\mathbf{1}^\top \mathbf{1} - (\mathbf{G}\mathbf{H}^{-1})^{\circ 2} \oslash \text{diag}(\mathbf{H}^{-1}),
\end{equation}
where $\mathbf{1}$ is a row vector of ones. For simplicity, we compute the row-wise scaling factor $\boldsymbol\beta\in\mathbb R^{d_{\text{row}}}$ by averaging the scaling factors within each row:
\begin{equation}
\label{eq:beta}
    \boldsymbol\beta = \left(\mathbf{1}\left(1 - \sqrt{\max\left(1 - 2\alpha (|\Delta\mathbf{L}_{\text{GPTQ}}| \oslash \mathbf C), 0\right)}\right)^\top\right) / d_{\text{col}},
\end{equation}
where $\oslash$ denotes Hadamard division.

\subsection{Efficient Weight Update with Gradient Compensation}
\label{apdx:g2ptq_update}

\paragraph{Proof of Equation~\ref{eq:recursive_f}.}
Let $\mathbf{L}^\top_{k,:} \in \mathbb{R}^{d_{\text{col}}}$ denote the $k$-th row of $\mathbf{L}^\top$. The inverse of the Hessian can be written as the sum of outer products of these rows:
\begin{equation}
    \mathbf{H}^{-1} = \mathbf{L} \mathbf{L}^{\top} = \sum_{k=0}^{d_{\text{col}}-1} (\mathbf{L}^\top_{k,:})^{\top} \mathbf{L}^\top_{k,:}.
\end{equation}
Let $\mathbf{g} \in \mathbb{R}^{d_{\text{col}}}$ be a single row of the gradient matrix $\mathbf{G}$. We want to compute the unpadded update vector $\mathbf{g}_{t:}(\mathbf{H}_{t:,t:})^{-1} \in \mathbb{R}^{d_{\text{col}}-t}$. Using the decomposition above, we obtain
\begin{equation}
\begin{aligned}
    \mathbf{g}_{t:}(\mathbf{H}_{t:,t:})^{-1} &= \mathbf{g}_{t:} \left( \sum_{k=t}^{d_{\text{col}}-1} (\mathbf{L}^\top_{k,t:})^{\top} \mathbf{L}^\top_{k,t:} \right) \\
    &= \sum_{k=t}^{d_{\text{col}}-1} \left[ \mathbf{g}_{t:} (\mathbf{L}^\top_{k,t:})^{\top} \right] \mathbf{L}^\top_{k,t:} \\
    &= \sum_{k=t}^{d_{\text{col}}-1} (\mathbf{g} \mathbf{L})_k \mathbf{L}^\top_{k,t:}.
\end{aligned}
\end{equation}
Applying the same row-wise argument to the full gradient matrix $\mathbf{G} \in \mathbb{R}^{d_{\text{row}} \times d_{\text{col}}}$ yields
\begin{equation}
\mathbf{G}_{:,t:}(\mathbf{H}_{t:,t:})^{-1} = \sum_{k=t}^{d_{\text{col}}-1} (\mathbf{G} \mathbf{L})_{:,k} \mathbf{L}^\top_{k,t:} = \sum_{k=t}^{d_{\text{col}}-1} \mathbf{M}_{:,k} \mathbf{L}^\top_{k,t:},
\end{equation}
Accordingly, the padded matrix $\mathbf{F}^{(t)}$ can be written as
\begin{equation}
\mathbf{F}^{(t)} = \sum_{k=t}^{d_{\text{col}}-1} \mathbf{M}_{:,k} \mathbf{L}^\top_{k,:}.
\end{equation}
This representation immediately gives the base case for $t=0$:
\begin{equation}
\mathbf{F}^{(0)} = \sum_{k=0}^{d_{\text{col}}-1} \mathbf{M}_{:,k} \mathbf{L}^\top_{k,:} = \mathbf{M}\mathbf{L}^\top.
\end{equation}
The recursive relation for $t+1$ then follows by separating the $t$-th term from the summation:
\begin{equation}
\begin{aligned}
\mathbf{F}^{(t+1)} &= \sum_{k=t+1}^{d_{\text{col}}-1} \mathbf{M}_{:,k} \mathbf{L}^\top_{k,:} \\
&= \sum_{k=t}^{d_{\text{col}}-1} \mathbf{M}_{:,k} \mathbf{L}^\top_{k,:} - \mathbf{M}_{:,t} \mathbf{L}^\top_{t,:} \\
&= \mathbf{F}^{(t)} - \mathbf{M}_{:,t} \mathbf{L}^\top_{t,:}.
\end{aligned}
\end{equation}

\paragraph{Proof of Equation~\ref{eq:lazy_batch}.}
For simplicity in the derivation, we assume the trust-region scaling factor $\boldsymbol\beta$ is absorbed into the matrix $\mathbf{M}$ with $\boldsymbol\beta\circ\mathbf{M}$. Let $Q = \{q_{\text{start}}, \dots, q_{\text{end}}\}$ be the set of column indices currently being updated in a block of size $B$, and let $R$ be the set of remaining column indices. First, we consider the batched update for $\mathbf{F}$. Applying the recursive relationship $\mathbf{F}^{(t+1)} = \mathbf{F}^{(t)} - \mathbf{M}_{:,t} \mathbf{L}^\top_{t,:}$ iteratively for all $t \in Q$ yields the update for the remaining columns $R$ after the block finishes:
\begin{equation}
    \mathbf{F}^{(q_{\text{end}}+1)}_{:,R} = \mathbf{F}^{(q_{\text{start}})}_{:,R} - \sum_{t \in Q} \mathbf{M}_{:,t} \mathbf{L}^\top_{t,R} = \mathbf{F}^{(q_{\text{start}})}_{:,R} - \mathbf{M}_{:,Q} \mathbf{L}^\top_{Q,R}.
\end{equation}
Next, we consider the updates to the weight matrix $\mathbf{W}$. The update to the remaining columns at step $t$ can be written as:
\begin{equation}
    \Delta\mathbf{W}_{:,R}^{(t)} = -\mathbf{E}_{:,t} \mathbf{L}^\top_{t,R} - \mathbf{F}^{(t)}_{:,R}.
\end{equation}
Accumulating these updates over the entire block $Q$, we obtain the total batched update for $\mathbf{W}_{:,R}$:
\begin{equation}
\label{eq:lazy_update1}
\begin{aligned}
    \Delta\mathbf{W}_{:,R} &= \sum_{t \in Q} \left( -\mathbf{E}_{:,t} \mathbf{L}^\top_{t,R} - \mathbf{F}^{(t)}_{:,R} \right) \\
&= -\mathbf{E}_{:,Q} \mathbf{L}^\top_{Q,R} - \sum_{t \in Q} \mathbf{F}^{(t)}_{:,R}.
\end{aligned}
\end{equation}
We can expand the term $\sum_{t \in Q} \mathbf{F}^{(t)}_{:,R}$ by expressing each $\mathbf{F}^{(t)}_{:,R}$ in terms of the initial block state $\mathbf{F}^{(q_{\text{start}})}_{:,R}$:
\begin{equation}
\begin{aligned}
    \sum_{t \in Q} \mathbf{F}^{(t)}_{:,R} &= \sum_{t \in Q} \left( \mathbf{F}^{(q_{\text{start}})}_{:,R} - \sum_{j=q_{\text{start}}}^{t-1} \mathbf{M}_{:,j} \mathbf{L}^\top_{j,R} \right) \\
    &= B \cdot \mathbf{F}^{(q_{\text{start}})}_{:,R} - \sum_{t \in Q} \sum_{j=q_{\text{start}}}^{t-1} \mathbf{M}_{:,j} \mathbf{L}^\top_{j,R} \\
    &= B \cdot \mathbf{F}^{(q_{\text{start}})}_{:,R} - \mathbf{M}_{:,Q} \mathbf{D} \mathbf{L}^\top_{Q,R},
\end{aligned}
\end{equation}
where $\mathbf{D} = \text{diag}(B-1, B-2, \dots, 0) \in \mathbb{R}^{B \times B}$. Finally, substituting this expression into Equation~\ref{eq:lazy_update1} gives the complete lazy-batch update formula:
\begin{equation}
\begin{aligned}
    \Delta\mathbf{W}_{:,R} &= -\mathbf{E}_{:,Q} \mathbf{L}^\top_{Q,R} - \left( B \cdot \mathbf{F}^{(q_{\text{start}})}_{:,R} - \mathbf{M}_{:,Q} \mathbf{D} \mathbf{L}^\top_{Q,R} \right) \\
    &= - \left( \mathbf{E}_{:,Q} \mathbf{L}^\top_{Q,R} + B \cdot \mathbf{F}^{(q_{\text{start}})}_{:,R} - \mathbf{M}_{:,Q} \mathbf{D} \mathbf{L}^\top_{Q,R} \right).
\end{aligned}
\end{equation}

\section{Implementation Details}
\label{apdx:impl_details}

\subsection{Pseudocode}
\label{apdx:pseudocode}
We provide the pseudocode for quantizing the entire Transformer model in Algorithm~\ref{alg:g2ptq-model}.

\subsection{Trust-region Threshold Autotuning}
\label{apdx:trust-region-autotune}
We determine the trust-region threshold $\alpha$ defined in Theorem~\ref{thm:trust_region_beta} for each Transformer block via a grid search. Specifically, for each candidate value of $\alpha$, we quantize the respective block and evaluate the resulting block-wise quantization loss with a subset of 64 calibration samples. The search is conducted over the interval $[0, 3]$ with a step size of $0.1$, and is terminated early once the loss begins to increase. The complete procedure is detailed in Algorithm~\ref{alg:alpha_autotune}.

\begin{algorithm}[t]
\begingroup
\setlength{\heavyrulewidth}{1.2pt}
\setlength{\lightrulewidth}{0.5pt}
\setlength{\aboverulesep}{1pt}
\setlength{\belowrulesep}{1.4pt}
\begin{tabular}{@{}p{0.98\textwidth}@{}}
\toprule
\textbf{Algorithm 2:} \name\ quantization for the entire Transformer model \\
\midrule\hspace*{6pt}%
\begin{minipage}{\dimexpr\linewidth-12pt\relax}
\textbf{Input:} Transformer blocks: $\text{block}^{(l)}$ ($l=1,2,\ldots,L$), model input $\mathbf{Y}$, indicator $\mathcal{I}_{\text{autotune}}$ \\
\textbf{Output:} Quantized Transformer model \\[3pt]
\phantom{0}1:\quad $\hat{\mathbf{Y}} \leftarrow \mathbf{Y}$ \\
\phantom{0}2:\quad \textbf{for} $l=1,2,\ldots,L$ \textbf{do} \\
\phantom{0}3:\quad \quad Load $\text{block}^{(l)}$ to GPU \\
\phantom{0}4:\quad \quad $\mathcal{S} \leftarrow$ Set of linear layers to be quantized in $\text{block}^{(l)}$ \\
\phantom{0}5:\quad \quad $\hat{\mathbf{Y}}' \leftarrow \hat{\mathbf{Y}}$ \\
\phantom{0}6:\quad \quad $\mathbf{Y} \leftarrow \text{block}^{(l)}(\mathbf{Y})$\text{.detach()}  \\
\phantom{0}7:\quad \quad $\hat{\mathbf{Y}} \leftarrow \text{block}^{(l)}(\hat{\mathbf{Y}}')$ \\
\phantom{0}8:\quad \quad \textbf{if} $l < L$ \textbf{then} \\
\phantom{0}9:\quad \quad \quad $\ell_H \leftarrow \ell_{\text{MSE}}\big(\hat{\mathbf{Y}}, (\hat{\mathbf{Y}} + \boldsymbol{\epsilon})\text{.detach()}\big)$ with $\boldsymbol{\epsilon} \sim \mathcal{N}(\mathbf{0}, d\mathbf{I})$ \\
10:\quad \quad \quad $\ell_G \leftarrow \ell_{\text{MSE}}(\hat{\mathbf{Y}}, \mathbf{Y})$ \\
11:\quad \quad \textbf{else} \\
12:\quad \quad \quad $\ell_H \leftarrow -\log p(y|\hat{\mathbf{Y}})$ with $y \sim p(\cdot|\mathbf{Y})$ \\
13:\quad \quad \quad $\ell_G \leftarrow D_{\text{KL}}\big(p(\cdot|\mathbf{Y}) \,\|\, p(\cdot | \hat{\mathbf{Y}})\big)$ \\
14:\quad \quad \textbf{end if} \\
15:\quad \quad Backpropagate $\ell_H$ to cache output activation gradients for each layer in $\mathcal{S}$  \\
16:\quad \quad Backpropagate $\ell_G$ to cache weight gradients for each layer in $\mathcal{S}$ \\
17:\quad \quad $\_ \leftarrow \text{block}^{(l)}(\hat{\mathbf{Y}}')$, compute Hessian via Equation~\ref{eq:hessian_fisher} for each layer in $\mathcal{S}$ \\
18:\quad \quad \textbf{if} $\mathcal{I}_{\text{autotune}}$ \textbf{then} \\
19:\quad \quad \quad Search for $\alpha^{(l)}$ via Algorithm~\ref{alg:alpha_autotune} and save it to disk \\
20:\quad \quad \textbf{else} \\
21:\quad \quad \quad Load $\alpha^{(l)}$ from disk \\
22:\quad \quad \textbf{end if} \\
23:\quad \quad Compute trust-region scaling factor $\boldsymbol\beta$ via Equation~\ref{eq:beta} \\
24:\quad \quad \textbf{for} each linear layer in $\mathcal{S}$ \textbf{do} \\
25:\quad \quad \quad Quantize the linear layer via Algorithm~\ref{alg:g2ptq} \\
26:\quad \quad \textbf{end for} \\
27:\quad \quad $\hat{\mathbf{Y}} \leftarrow \text{block}^{(l)}(\hat{\mathbf{Y}}')$  \\
28:\quad \quad Offload $\text{block}^{(l)}$ to CPU \\
29:\quad \textbf{end for}

\vspace{2pt}
\end{minipage}\\
\bottomrule
\end{tabular}
\endgroup
\captionsetup{labelformat=empty,labelsep=none,textformat=empty}
\caption{}
\vspace{-\baselineskip}
\label{alg:g2ptq-model}
\end{algorithm}

\begin{algorithm}[t]
\begingroup
\setlength{\heavyrulewidth}{1.2pt}
\setlength{\lightrulewidth}{0.5pt}
\setlength{\aboverulesep}{1pt}
\setlength{\belowrulesep}{1.4pt}
\begin{tabular}{@{}p{0.98\textwidth}@{}}
\toprule
\textbf{Algorithm 3:} Trust-region threshold autotuning for one Transformer block \\
\midrule\hspace*{6pt}%
\begin{minipage}{\dimexpr\linewidth-12pt\relax}
\textbf{Input:} Transformer block: $\text{block}^{(l)}$, search range $[\alpha_{\min}, \alpha_{\max}]$, step size $\Delta\alpha$, quantized block input $\hat{\mathbf{Y}}'$, full-precision block output $\mathbf{Y}$ \\
\textbf{Output:} Trust-region threshold $\alpha_{\text{best}}$ \\[3pt]
\phantom{0}1:\quad $\ell_{\text{best}} \leftarrow \infty$, $\alpha_{\text{best}} \leftarrow 0$ \\
\phantom{0}2:\quad $\text{state\_dict} \leftarrow \text{block}^{(l)}.\text{state\_dict}()$ \\
\phantom{0}3:\quad \textbf{for} $\alpha \in \{\alpha_{\min}, \alpha_{\min} + \Delta\alpha, \ldots, \alpha_{\max}\}$ \textbf{do} \\
\phantom{0}4:\quad \quad $\mathcal{S} \leftarrow$ Set of linear layers to be quantized in $\text{block}^{(l)}$ \\
\phantom{0}5:\quad \quad Compute trust-region scaling factor $\boldsymbol\beta$ via Equation~\ref{eq:beta} \\
\phantom{0}6:\quad \quad \textbf{for} each linear layer in $\mathcal{S}$ \textbf{do} \\
\phantom{0}7:\quad \quad \quad Quantize the linear layer via Algorithm~\ref{alg:g2ptq} \\
\phantom{0}8:\quad \quad \textbf{end for} \\
\phantom{0}9:\quad \quad $\hat{\mathbf{Y}} \leftarrow \text{block}^{(l)}(\hat{\mathbf{Y}}')$ \\
10:\quad \quad \textbf{if} $l < L$ \textbf{then} \\
11:\quad \quad \quad $\ell_G \leftarrow \ell_{\text{MSE}}(\hat{\mathbf{Y}}, \mathbf{Y})$ \\
12:\quad \quad \textbf{else} \\
13:\quad \quad \quad $\ell_G \leftarrow D_{\text{KL}}\big(p(\cdot|\mathbf{Y}) \,\|\, p(\cdot | \hat{\mathbf{Y}})\big)$ \\
14:\quad \quad \textbf{end if} \\
15:\quad \quad $\text{block}^{(l)}.\text{load\_state\_dict}(\text{state\_dict})$ \\
16:\quad \quad \textbf{if} $\ell_G < \ell_{\text{best}}$ \textbf{then} \\
17:\quad \quad \quad $\ell_{\text{best}} \leftarrow \ell_G$ \\
18:\quad \quad \quad $\alpha_{\text{best}} \leftarrow \alpha$ \\
19:\quad \quad \textbf{else} \\
20:\quad \quad \quad \textbf{break} \\
21:\quad \quad \textbf{end if} \\
22:\quad \textbf{end for}
\vspace{2pt}
\end{minipage} \\
\bottomrule
\end{tabular}
\endgroup
\captionsetup{labelformat=empty,labelsep=none,textformat=empty}
\caption{}
\vspace{-\baselineskip}
\label{alg:alpha_autotune}
\end{algorithm}

\subsection{Quantization Efficiency Optimization}
\label{apdx:perf_optim}
\paragraph{Kernel Fusion.}
As detailed in Section~\ref{sec:effi_impl}, \name\ is implemented using a lazy-batch update scheme. Within each column block, the operations in the \textit{for} loop are typically memory-bound, making them highly suitable for kernel fusion. Consequently, we implement two Triton kernels~\citep{tillet2019triton}: one for column-wise weight quantization and quantization error computation (Lines 11--12 in Algorithm~\ref{alg:g2ptq}), and another for the intra-block weight and $\mathbf F$ matrix update (Lines 13--14 in Algorithm~\ref{alg:g2ptq}). These custom kernels significantly accelerate \name\ by reducing redundant memory traffic.

\paragraph{Efficient Scaling for Large MoEs.}
We introduce further optimizations to enable the quantization process to scale efficiently to large-scale MoE models. First, we group multiple small-sized experts into a single batch and quantize them simultaneously to enhance arithmetic intensity. Second, we employ data parallelism for distributed quantization, distributing both the calibration data processing and the linear layer quantization across multiple devices. Third, we offload the hidden states of the calibration data to conserve GPU memory, utilizing asynchronous onloading and offloading to minimize overhead.

\section{Experiment Details}
\label{apdx:exp_details}
\input{tables/appendix/model}

\subsection{Model Coverage}
\label{apdx:model}
To comprehensively assess the generalizability of \name\ across different model architectures, we evaluate it on 15 models ranging in size from 0.6B to 125B parameters across various model families. As summarized in Table~\ref{tab:model_arch}, these models encompass a diverse array of modern architectures, including standard LLaMA-like dense models (LLaMA3 and Qwen3), hybrid models featuring mixed linear, full, or sparse attention mechanisms (Qwen3.5 and Qwen3.8-Flash-Next), and mixture-of-experts models (Qwen3-MoE and Qwen3.8-Flash-Next). Notably, Qwen3.8-Flash-Next also incorporates gated residuals and n-gram embeddings, rendering its architecture highly distinct from the other evaluated models.

\subsection{Discussion of the KL Divergence Metric}
\label{apdx:kl}
KL divergence serves as a reliable proxy for evaluating overall model degradation. Recent studies~\citep{dutta2024accuracy, tseng2025model, helcig2026statistically} demonstrate that KL divergence exhibits a strong correlation with downstream task accuracy, making it a more dependable metric for gauging quantization-induced misalignment than downstream metrics such as PPL and question-answering accuracy. These task-specific metrics can be noisy and occasionally fail to capture true model degradation, particularly at 4-bit precision where the drop in accuracy is often marginal. In contrast, KL divergence directly quantifies the distributional shift relative to the full-precision model, thereby providing a robust measure for assessing the preservation of a model's general capabilities.

\subsection{Benchmarks}
\label{apdx:benchmark}
For downstream evaluation, we report perplexity (PPL) and KL divergence, in line with prior work~\citep{frantaroptq,lin2024awq,dutta2024accuracy,tseng2025model}. These metrics are evaluated on datasets spanning multiple domains: the pre-training dataset WikiText2~\citep{merity2016pointer}, the chat dataset UltraChat~\citep{ding2023enhancing}, and the mathematics dataset NuminaMath1.5~\citep{numina_math_datasets}. Specifically, we utilize the standard test set for WikiText2, while for UltraChat and NuminaMath1.5, we randomly sample 128 and 256 examples, respectively, to serve as our test sets. Furthermore, we report the accuracy across seven commonsense question-answering benchmarks: ARC-Challenge, ARC-Easy~\citep{clark2018think}, C-Eval~\citep{huang2023ceval}, HellaSwag~\citep{zellers2019hellaswag}, LAMBADA~\citep{paperno2016lambada}, PIQA~\citep{bisk2020piqa}, and WinoGrande~\citep{sakaguchi2021winogrande}. To assess the long-context generation capabilities of the quantized models, we conduct evaluations on challenging reasoning benchmarks for the Qwen3.8-Flash-Next model following \citep{liu2025quantization}. These include GPQA Diamond~\citep{rein2023gpqa}, LiveCodeBench v6~\citep{jain2025livecodebench}, ArXiv-Math~\citep{dekoninck2026matharena}, and IFBench~\citep{pyatkin2025generalizing}, which collectively encompass tasks requiring domain expertise, mathematical reasoning, coding, and instruction following. Evaluations on reasoning benchmarks are executed using EvalScope~\citep{evalscope_2024} across three different random seeds. The sampling parameters are configured with a temperature of 1.0, top\_p of 0.95, top\_k of 20, and repetition\_penalty of 1.0.

\subsection{Quantization Settings and Baselines}
\label{apdx:qsettings_baseline}
We primarily investigate 2-, 3-, and 4-bit weight-only quantization, as well as 4-bit weight-activation quantization. For weights, we apply symmetric per-channel quantization. Following prior work~\cite{frantaroptq,ashkboos2024quarot,lin2024awq}, we perform a grid search over the weight clipping factors to minimize the MSE of the quantized weights. Additionally, we utilize the activation reordering technique~\cite{frantaroptq} with static grouping to better compensate for quantization errors in critical outlier channels~\cite{dettmers2022gpt3,xiao2023smoothquant}. For activations, we apply symmetric per-token quantization with a clipping ratio set to 0.9~\citep{ashkboos2024quarot}. To calibrate the dense models, we sample 1,024 sequences of length 2,048 from the NeuralMagic dataset\footnote{\url{https://huggingface.co/datasets/neuralmagic/LLM_compression_calibration}}. Due to the sparse activation patterns of MoE, we increase the number of calibration samples to 8,192 to ensure a sufficient number of calibration tokens for each expert. Prior to quantization, the models are rotated~\cite{ashkboos2024quarot} with Hadamard transforms to mitigate the impact of outliers. As baselines, we compare \name\ against RTN~\cite{nagel2021white} and several representative state-of-the-art weight quantization methods, including GPTQ~\cite{frantaroptq}, GuidedQuant~\cite{kim2025guidedquant}, and GPTAQ~\cite{li2025gptaq}. To ensure a fair comparison, all methods employ a uniform scalar quantizer. For both \name\ and GuidedQuant, we set the number of channel groups to $g=4$, following the configuration in GuidedQuant~\cite{kim2025guidedquant}.

\section{Additional Experiments}
\label{apdx:exp}

\subsection{More Ablations}
\label{apdx:ablations}
\input{tables/appendix/more_ablations}

\paragraph{Ablation on Calibration Set Size.}
Figure~\ref{fig:calib_size_ablation} illustrates the effect of calibration set size on \name\ and GPTQ. Although \name\ achieves competitive results with a small calibration set, it generally requires more calibration samples than GPTQ to attain its best accuracy. We assume that additional calibration samples can help reduce the variance of the estimated first- and second-order information.

\paragraph{Ablation on the Number of Output Channel Groups.}
In Equation~\ref{eq:avg_hessian}, the output channels are partitioned into $g$ groups, and a single averaged Hessian is shared among the channels within each group. The number of groups $g$ controls the trade-off between Hessian estimation accuracy and quantization efficiency. As shown in Table~\ref{tab:g_ablation}, setting $g=1$ already yields strong performance. Increasing $g$ further improves quantization accuracy by providing more fine-grained Hessian estimates, but this improvement comes at the cost of higher memory usage and computational overhead. In our experiments, we set $g=4$ following GuidedQuant~\citep{kim2025guidedquant}.

\paragraph{Ablation on Exact Gradient Compensation.}
As discussed in Section~\ref{sec:method} and theoretically analyzed in Appendix~\ref{apdx:grad_aprox}, prior first-order methods, such as FOEM~\citep{zheng2026first}, rely on a first-order Taylor expansion to approximate the loss gradients. While this approximation is exact and efficient for the layer-wise MSE objective, it introduces substantial errors when applied to more expressive objectives, such as the block-wise MSE or KL divergence employed in \name. To empirically validate the necessity of computing the exact gradient, we compare our exact gradient compensation against the gradient approximation approach used in FOEM. Table~\ref{tab:ablation_exact_grad} summarizes the results for the Qwen3-0.6B model under 3-bit weight-only quantization. Replacing the exact gradient with the Taylor approximation causes a noticeable degradation in both distribution alignment and downstream task performance, resulting in a 2.09\% decrease in QA accuracy. These findings confirm that utilizing the exact gradient is crucial for providing accurate first-order guidance when optimizing under block-wise objectives.

\subsection{More Discussions}
\label{apdx:discussions}
\input{tables/appendix/alpha_robust}
\input{tables/appendix/autotune_cost}
\input{tables/appendix/hessian_clip}

\paragraph{Calibration Resources for Quantizing a 100B-parameter Model.}
To demonstrate the scalability of \name, we evaluate the calibration resources required to quantize a large-scale model, specifically the 125B-parameter Qwen3.8-Flash-Next. As detailed in Section~\ref{apdx:qsettings_baseline}, we utilize 8,192 calibration samples for this MoE model, and the quantization process is distributed across 8 accelerators. Under this configuration, GPTQ requires 19.48~GB of memory per device and completes in 7.98 hours. In comparison, \name$^*$ consumes 86.72~GB of memory per device and requires 20.30 hours. Although \name$^*$ incurs a higher resource overhead due to the block-wise backpropagation, its memory footprint and execution time remain highly manageable on a standard multi-accelerator node, confirming its practicality for quantizing large-scale models.

\paragraph{You Only Autotune Once.}
As detailed in Appendix~\ref{apdx:trust-region-autotune}, the trust-region threshold $\alpha$ defined in Theorem~\ref{thm:trust_region_beta} is optimized via a grid search during the \name\ quantization process. Here, we evaluate the generalizability of this threshold across varying bit-widths and model variants. Specifically, in Table~\ref{tab:alpha_robust}, we directly apply the $\alpha$ value---originally optimized for 3-bit weight quantization on the Qwen3-0.6B model---to alternative bit-widths (2- and 4-bit) and a distinct model variant (Qwen3-0.6B-Base) without any re-tuning. The results demonstrate that \name\ maintains comparable performance without setting-specific autotuning, indicating that the trust-region threshold can be optimized once and applied universally across diverse quantization settings and model variants. Additionally, we visualize the optimal trust-region thresholds for different settings in Figure~\ref{fig:alpha_visualize}. The average absolute difference across layers remains within 0.4, demonstrating that the optimal threshold is highly stable across configurations, which further justifies the feasibility of transferring a single optimized threshold. As shown in Table~\ref{tab:autotune_cost}, this once-for-all approach yields a 1.20--1.33$\times$ speedup for \name\ quantization on an 8B model, substantially reducing the computational overhead required to adapt \name\ to new configurations.

\paragraph{Hessian-based Weight Clipping.}
Recent studies~\citep{yu2026h} leverage the diagonal of the Hessian matrix to guide the search for optimal weight clipping factors. Specifically, they employ a grid search over the weight clipping factors to minimize $\sum_{i=1}^{d_{row}}\sum_{j=1}^{d_{col}}\mathbf H_{j,j}\Delta\mathbf W_{i,j}^2$. Assuming the Hessian is an identity matrix reduces this approach to the commonly used weight MSE-based clipping strategy~\citep{frantaroptq}. In this section, we investigate the effectiveness of this Hessian-based weight clipping (H-Scale) within the \name\ framework across different model families. Table~\ref{tab:ablation_hscale} presents the 3-bit weight quantization results for the Qwen3 and Qwen3.5 model families. Interestingly, we observe a clear dichotomy in the effectiveness of Hessian-based clipping. For the Qwen3 models (0.6B, 1.7B, and 4B), applying H-Scale consistently and significantly improves both distribution alignment and downstream task accuracy. For instance, on the Qwen3-4B model, H-Scale increases the average QA accuracy by 2.12\%. Conversely, for the Qwen3.5 models (0.8B, 2B, and 4B), this clipping strategy provides no tangible benefits and even leads to marginal performance degradation. This suggests that Hessian-based weight clipping is highly model-dependent, potentially influenced by the distinct weight and activation distributions inherently learned by different model architectures. Consequently, we leave the exploration of optimal clipping method selection for future work. In this work, we continue to employ the weight MSE-based clipping strategy, as detailed in Section~\ref{sec:exp_settings}.

\subsection{Complete Results on Weight-Only Quantization}
\label{apdx:weight_only}
We provide comprehensive weight-only quantization results for 2-, 3-, and 4-bit settings in Table~\ref{tab:w2}, Table~\ref{tab:w3}, and Table~\ref{tab:w4}, respectively. Our evaluation encompasses 13 dense models featuring diverse architectures and sizes ranging from 0.6B to 70B parameters. Across these configurations, \name\ demonstrates exceptional generalizability to various model architectures and bit-widths. Specifically, \name\ variants achieve the lowest average KL divergence in 38 out of 39 evaluated quantization settings, indicating superior distributional alignment with the full-precision models. Consequently, \name\ yields better overall downstream task accuracy. For instance, under 2-bit quantization, \name\ variants attain the highest QA accuracy on 12 out of the 13 evaluated models.

\subsection{Complete Results on Weight-Activation Quantization}
\label{apdx:weight_activation}
In Table~\ref{tab:w4a4}, we present the 4-bit weight-activation quantization results for seven medium-sized dense models across three different model families. The \name\ variants consistently achieve the lowest KL divergence on six out of the seven models and attain the highest QA accuracy on every evaluated model. This further demonstrates the adaptability of \name\ to various quantization settings.

\subsection{Complete Results on MoE Quantization}
\label{apdx:moe}
We present quantization results for two MoE models, Qwen3-30B-A3B and Qwen3.8-Flash-Next, with parameter sizes ranging from 30B to 125B, in Table~\ref{tab:qwen3_30b_a3b} and Table~\ref{tab:qwen38-flash-next-w4a16}. The \name\ variants consistently achieve better distributional alignment and downstream task accuracy than the GPTQ baseline, demonstrating the ability of \name\ to scale effectively to large-scale MoE models. Furthermore, we evaluate the quantized Qwen3.8-Flash-Next models across four reasoning benchmarks. Notably, \name$^*$ better preserves reasoning capabilities under 4-bit quantization, incurring only a 0.19\% decrease in average accuracy and thereby effectively achieving lossless quantization.

\input{tables/appendix/w2}
\input{tables/appendix/w3}
\input{tables/appendix/w4}
\input{tables/appendix/w4a4}
\input{tables/appendix/moe}

%% file: tables/appendix/model.tex
\begin{table*}[!t]
\begin{center}
\resizebox{0.6\linewidth}{!}
{
\begin{tabular}{lcc}
\toprule
\textbf{Model Family} & \textbf{Attn Arch} & \textbf{MLP Arch} \\ 
\midrule
LLaMA3~\citep{grattafiori2024llama} & GQA & Dense \\
Qwen3~\citep{yang2025qwen3} & GQA & Dense \\
Qwen3.5~\citep{qwen35blog} & GQA / GDN & Dense \\
Qwen3-MoE~\citep{yang2025qwen3} & GQA & MoE \\
Qwen3.8-Flash-Next~\citep{qiu2026design} & QSA / GDN & MoE \\
\bottomrule
\end{tabular}
}
\end{center}
\vspace{-1ex}
\caption{Comparison of Attention and MLP architectures across different model families.}
\label{tab:model_arch}
\end{table*}

%% file: tables/appendix/more_ablations.tex
\begin{figure*}[!t]
\centering

\newlength{\samplepanelheight} 
\setlength{\samplepanelheight}{0.18\textheight}

\begin{subfigure}[t]{0.48\textwidth}
    \centering
    \begin{minipage}[c][\samplepanelheight][c]{\linewidth}
        \centering
        \includegraphics[width=\linewidth]{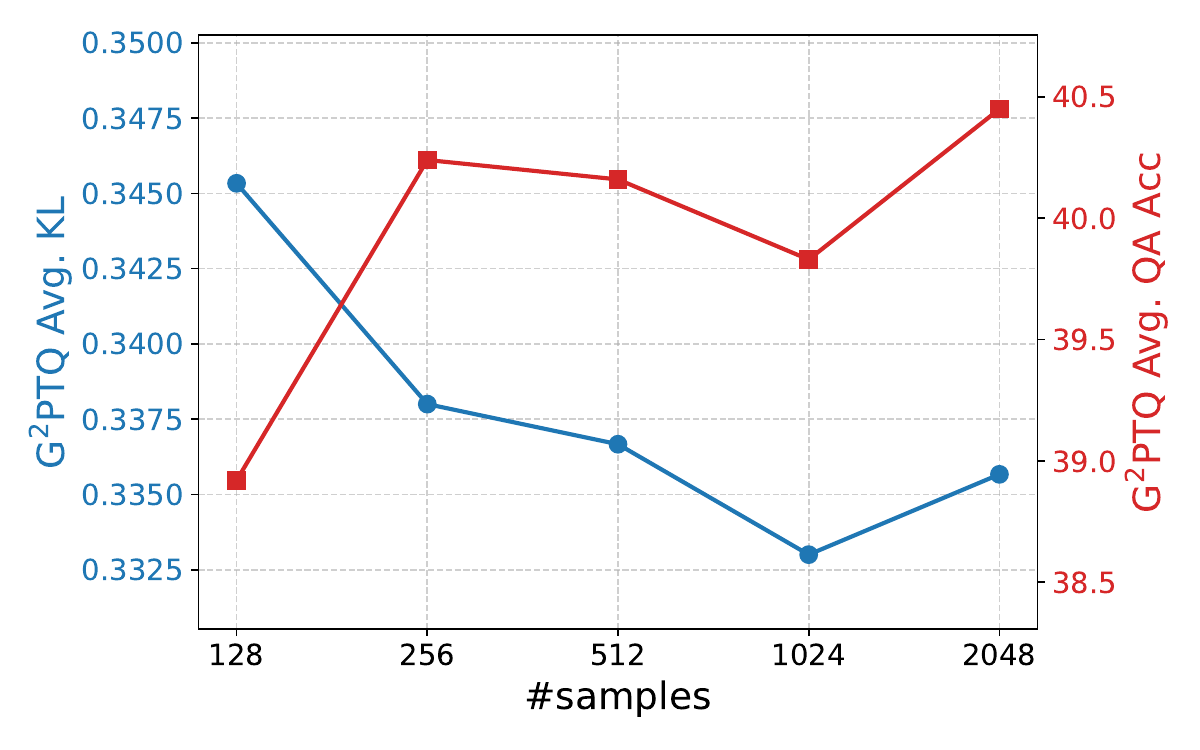}
    \end{minipage}
    \vspace{-0.5ex}
    \label{fig:sample_ablation_g2ptq}
\end{subfigure}
\hfill
\begin{subfigure}[t]{0.48\textwidth}
    \centering
    \begin{minipage}[c][\samplepanelheight][c]{\linewidth}
        \centering
        \includegraphics[width=\linewidth]{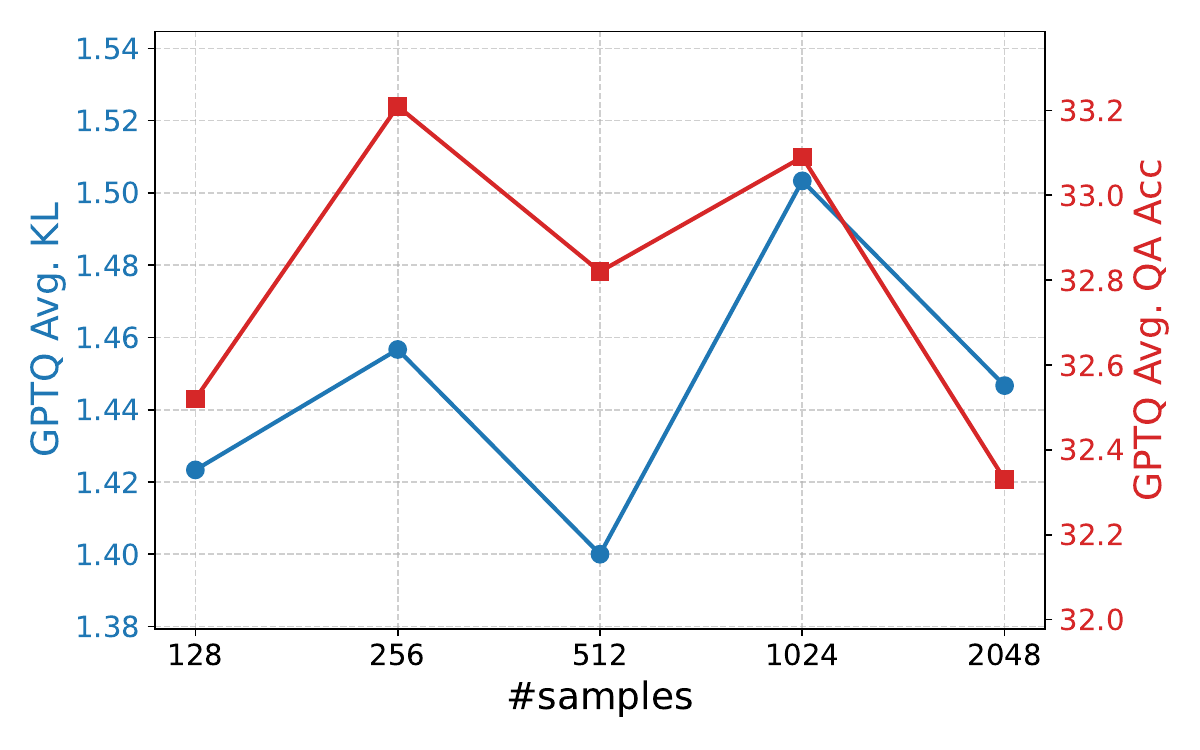}
    \end{minipage}
    \vspace{-0.5ex}
    \label{fig:sample_ablation_gptq}
\end{subfigure}

\vspace{-1.5ex}
\caption{Ablation study on the calibration set size for \name\ (left) and GPTQ (right). Experiments are conducted using 3-bit weight quantization on the Qwen3-0.6B model.}
\label{fig:calib_size_ablation}
\end{figure*}

\begin{table*}[!t]
\begin{center}
\resizebox{0.6\linewidth}{!}
{
\begin{tabular}{c|ccc|cc}
\toprule
\textbf{\ensuremath{g}} & \textbf{KL} & \textbf{PPL} & \textbf{QA} & \textbf{Mem Cost} & \textbf{Time Cost} \\ 
\midrule
1  & 3.43e-01 & 16.66 & 40.77 & 1.00x & 1.00x \\
2  & 3.39e-01 & 16.33 & 40.11 & 1.00x & 1.03x \\
4  & 3.33e-01 & 16.12 & 39.83 & 1.01x & 1.20x \\
8  & 3.31e-01 & 16.25 & 39.81 & 1.02x & 1.47x \\
16 & 3.23e-01 & 15.77 & 40.53 & 1.04x & 2.00x \\
\bottomrule
\end{tabular}
}
\end{center}
\vspace{-1ex}
\caption{Ablation study on the number of output channel groups $g$. We report the average accuracy metrics for 3-bit quantization of the Qwen3-0.6B model, alongside the relative calibration cost.}
\label{tab:g_ablation}
\end{table*}

\begin{table*}[!t]
\begin{center}
\resizebox{0.8\linewidth}{!}
{
\begin{tabular}{c|cc|cc|cc|c}
\toprule
\multirow{2}{*}{\textbf{Method}} & \multicolumn{2}{c|}{\textbf{WikiText2}} & \multicolumn{2}{c|}{\textbf{UltraChat}} & \multicolumn{2}{c|}{\textbf{NuminaMath}} & \multirow{2}{*}{\textbf{QA Avg.}} \\ 
\cmidrule(lr){2-3} \cmidrule(lr){4-5} \cmidrule(lr){6-7}
& \textbf{KL} & \textbf{PPL} & \textbf{KL} & \textbf{PPL} & \textbf{KL} & \textbf{PPL} & \\ 
\midrule
w/ Approx Grad & 7.10e-01 & 35.29 & 5.69e-01 & 12.08 & 3.32e-01 & 5.83 & 37.74 \\
\rowcolor{g2ptqblue!12} w/ Exact Grad & \textbf{4.97e-01} & \textbf{32.34} & \textbf{3.39e-01} & \textbf{10.52} & \textbf{1.63e-01} & \textbf{5.49} & \textbf{39.83} \\
\bottomrule
\end{tabular}
}
\end{center}
\vspace{-1ex}
\caption{Ablation study on exact versus approximated gradient compensation. Experiments are conducted using 3-bit weight quantization on the Qwen3-0.6B model.}
\label{tab:ablation_exact_grad}
\end{table*}

%% file: tables/appendix/alpha_robust.tex
\begin{table*}[!t]
\begin{center}
\resizebox{0.95\linewidth}{!}
{
\begin{tabular}{c|c|cc|cc|cc|c}
\toprule
\multirow{2}{*}{\textbf{Model}} & \multirow{2}{*}{\textbf{Autotune}} & \multicolumn{2}{c|}{\textbf{WikiText2}} & \multicolumn{2}{c|}{\textbf{UltraChat}} & \multicolumn{2}{c|}{\textbf{NuminaMath}} & \multirow{2}{*}{\textbf{QA Avg.}} \\ 
\cmidrule(lr){3-4} \cmidrule(lr){5-6} \cmidrule(lr){7-8}
& & \textbf{KL} & \textbf{PPL} & \textbf{KL} & \textbf{PPL} & \textbf{KL} & \textbf{PPL} & \\ 
\midrule

\multirow{2}{*}{\begin{tabular}{@{}c@{}}\textbf{Qwen3-0.6B} \\ \textbf{(W2A16)}\end{tabular}}  
& \checkmark & 1.85e+00 & 125.39 & 1.35e+00 & 26.60 & 5.86e-01 & 7.76 & 32.93 \\
& $\times$ & 1.91e+00 & 131.15 & 1.40e+00 & 28.14 & 6.02e-01 & 8.11 & 32.56 \\
\midrule

\multirow{2}{*}{\begin{tabular}{@{}c@{}}\textbf{Qwen3-0.6B} \\ \textbf{(W4A16)}\end{tabular}}  
& \checkmark & 1.60e-01 & 23.98 & 1.01e-01 & 8.56 & 5.35e-02 & 5.03 & 45.80 \\
& $\times$ & 1.59e-01 & 23.92 & 1.01e-01 & 8.56 & 5.33e-02 & 5.04 & 46.25 \\
\midrule

\multirow{2}{*}{\begin{tabular}{@{}c@{}}\textbf{Qwen3-0.6B-Base} \\ \textbf{(W3A16)}\end{tabular}}  
& \checkmark & 4.79e-01 & 20.14 & 3.42e-01 & 7.45 & 1.70e-01 & 3.75 & 44.99 \\
& $\times$ & 4.75e-01 & 20.10 & 3.39e-01 & 7.43 & 1.72e-01 & 3.75 & 45.36 \\
\bottomrule
\end{tabular}
}
\end{center}
\vspace{-1ex}
\caption{Comparison of setting-specific autotuning (\checkmark) versus applying a ``you only autotune once'' threshold ($\times$) transferred from the 3-bit Qwen3-0.6B baseline.}
\label{tab:alpha_robust}
\end{table*}

%% file: tables/appendix/autotune_cost.tex
\begin{figure*}[t]
    \centering
    
    \begin{minipage}[b]{0.49\textwidth}
        \centering
        \includegraphics[width=\linewidth]{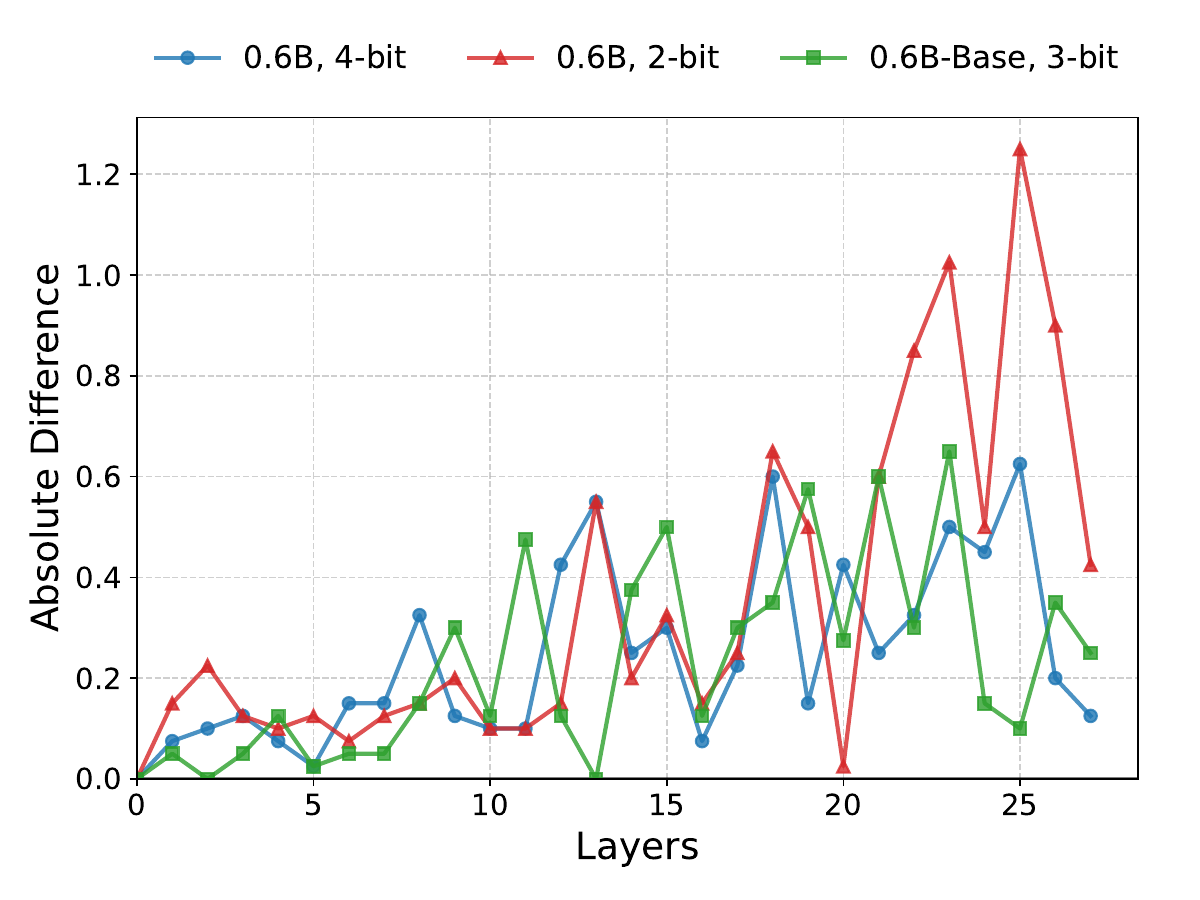}
        \captionsetup{type=figure, skip=2pt}
        \vspace{-2.5ex}
        \caption{Comparison of the autotuned trust-region threshold across various quantization settings and model variants, showing absolute differences relative to the 3-bit Qwen3-0.6B baseline.}
        \label{fig:alpha_visualize}
    \end{minipage}
    \hfill
    \begin{minipage}[b]{0.47\textwidth}
        \centering
        \resizebox{\linewidth}{!}{
            \begin{tabular}{c|c|cc}
                \toprule
                \textbf{Model Size} & \textbf{Method} & \textbf{Time (H)} & \textbf{Speedup} \\ 
                \midrule
                \multirow{2}{*}{0.6B} 
                & \name & 0.15 & 1.20x \\
                & \name\textsuperscript{*} & 0.26 & 1.27x \\
                \midrule
                \multirow{2}{*}{8B} 
                & \name & 1.24 & 1.20x \\
                & \name\textsuperscript{*} & 1.69 & 1.33x \\
                \bottomrule
            \end{tabular}
        }
        \vspace{6.5ex}
        \captionsetup{type=table}
        \caption{Calibration speedup achieved by applying the ``you only autotune once'' threshold transfer on Qwen3-0.6B and 8B models.}
        \label{tab:autotune_cost}
    \end{minipage}
    
\end{figure*}

%% file: tables/appendix/hessian_clip.tex
\begin{table*}[!t]
\begin{center}
\resizebox{0.95\linewidth}{!}
{
\begin{tabular}{c|c|cc|cc|cc|c}
\toprule
\multirow{2}{*}{\textbf{Model}} & \multirow{2}{*}{\textbf{H-Scale}} & \multicolumn{2}{c|}{\textbf{WikiText2}} & \multicolumn{2}{c|}{\textbf{UltraChat}} & \multicolumn{2}{c|}{\textbf{NuminaMath}} & \multirow{2}{*}{\textbf{QA Avg.}} \\ 
\cmidrule(lr){3-4} \cmidrule(lr){5-6} \cmidrule(lr){7-8}
& & \textbf{KL} & \textbf{PPL} & \textbf{KL} & \textbf{PPL} & \textbf{KL} & \textbf{PPL} & \\ 
\midrule

\multirow{2}{*}{\textbf{Qwen3-0.6B}}  
& $\times$ & 4.97e-01 & 32.34 & 3.39e-01 & 10.52 & 1.63e-01 & 5.49 & 39.83 \\
& \checkmark & \textbf{4.49e-01} & \textbf{32.15} & \textbf{2.91e-01} & \textbf{10.35} & \textbf{1.31e-01} & \textbf{5.48} & \textbf{41.95} \\
\midrule

\multirow{2}{*}{\textbf{Qwen3-1.7B}}  
& $\times$ & 4.73e-01 & 21.66 & 3.52e-01 & 11.76 & 1.51e-01 & \textbf{5.69} & 51.33 \\
& \checkmark & \textbf{3.89e-01} & \textbf{20.28} & \textbf{2.94e-01} & \textbf{11.50} & \textbf{1.09e-01} & 5.72 & \textbf{53.14} \\
\midrule

\multirow{2}{*}{\textbf{Qwen3-4B}}  
& $\times$ & 4.10e-01 & 18.38 & 2.43e-01 & \textbf{7.88} & 1.07e-01 & \textbf{4.84} & 60.30 \\
& \checkmark & \textbf{3.24e-01} & \textbf{16.92} & \textbf{2.10e-01} & 8.22 & \textbf{8.83e-02} & 4.92 & \textbf{62.42} \\
\midrule

\multirow{2}{*}{\textbf{Qwen3.5-0.8B}}  
& $\times$ & \textbf{3.16e-01} & 23.36 & \textbf{2.38e-01} & \textbf{6.90} & \textbf{1.09e-01} & \textbf{3.69} & \textbf{44.47} \\
& \checkmark & 3.19e-01 & \textbf{23.33} & 2.45e-01 & 6.99 & \textbf{1.09e-01} & 3.73 & 44.13 \\
\midrule

\multirow{2}{*}{\textbf{Qwen3.5-2B}}  
& $\times$ & \textbf{2.68e-01} & \textbf{16.03} & \textbf{1.75e-01} & \textbf{5.36} & \textbf{8.68e-02} & \textbf{3.20} & \textbf{53.53} \\
& \checkmark & 2.82e-01 & 16.35 & 1.82e-01 & 5.40 & 8.79e-02 & 3.22 & 52.41 \\
\midrule

\multirow{2}{*}{\textbf{Qwen3.5-4B}}  
& $\times$ & \textbf{2.38e-01} & 11.56 & \textbf{1.56e-01} & 4.74 & 7.45e-02 & \textbf{2.85} & \textbf{64.65} \\
& \checkmark & 2.45e-01 & \textbf{11.54} & 1.60e-01 & \textbf{4.72} & \textbf{7.39e-02} & 2.86 & 64.58 \\
\bottomrule
\end{tabular}
}
\end{center}
\vspace{-1ex}
\caption{Effect of Hessian-based weight clipping (H-Scale) across different model families. Experiments are conducted under 3-bit weight quantization.}
\label{tab:ablation_hscale}
\end{table*}

%% file: tables/appendix/w2.tex
\begin{table*}[!t]
\centering
\begin{adjustbox}{
  max width=\linewidth,
  max totalheight=0.95\textheight,
  keepaspectratio
}
\begin{tabular}{c|l|cc|cc|cc|cccccccc}
\toprule
\multirow{2}{*}{\textbf{Model}} &
\multirow{2}{*}{\textbf{Method}} &
\multicolumn{2}{c|}{\textbf{WikiText2}} &
\multicolumn{2}{c|}{\textbf{UltraChat}} &
\multicolumn{2}{c|}{\textbf{NuminaMath}} &
\multicolumn{8}{c}{\textbf{CommonSense QA}} \\
\cmidrule(lr){3-4}
\cmidrule(lr){5-6}
\cmidrule(lr){7-8}
\cmidrule(lr){9-16}
& & \textbf{KL} & \textbf{PPL}
& \textbf{KL} & \textbf{PPL}
& \textbf{KL} & \textbf{PPL}
& \textbf{ARC-C} & \textbf{ARC-E} & \textbf{C-eval}
& \textbf{HellaS} & \textbf{LAMB} & \textbf{PIQA}
& \textbf{Wino} & \textbf{Avg.} \\
\midrule

\multirow{7}{*}{\begin{tabular}{@{}c@{}}\textbf{Qwen3}\\\textbf{0.6B}\end{tabular}}
& BF16 & 0 & 20.96 & 0 & 8.00 & 0 & 4.88 & 34.13 & 56.06 & 43.76 & 47.29 & 40.02 & 67.36 & 56.35 & 49.28 \\
\cmidrule{2-16}
& RTN & 1.33e+01 & 5.60e+06 & 1.46e+01 & 8.74e+06 & 1.51e+01 & 8.35e+06 & 26.19 & 24.66 & 25.48 & 26.31 & 0.00 & 50.71 & 49.33 & 28.95 \\
& GPTQ & 5.44e+00 & 3.58e+03 & 5.37e+00 & 1.06e+03 & 3.39e+00 & 88.79 & 25.51 & 28.20 & 24.52 & 26.91 & 0.91 & 52.50 & 50.91 & 29.92 \\
& GuidedQ & 3.02e+00 & 312.46 & 2.86e+00 & 102.97 & 1.49e+00 & 15.59 & 23.04 & 30.05 & 23.55 & 27.86 & 1.71 & 54.03 & 50.12 & 30.05 \\
& GPTAQ & 3.88e+00 & 808.32 & 3.76e+00 & 249.25 & 2.35e+00 & 37.32 & 22.44 & 28.03 & 21.17 & 26.86 & 0.62 & 51.20 & 49.96 & 28.61 \\
& \cellcolor{g2ptqblue!12}\name
& \cellcolor{g2ptqblue!12}2.12e+00
& \cellcolor{g2ptqblue!12}147.21
& \cellcolor{g2ptqblue!12}1.74e+00
& \cellcolor{g2ptqblue!12}38.22
& \cellcolor{g2ptqblue!12}7.99e-01
& \cellcolor{g2ptqblue!12}9.53
& \cellcolor{g2ptqblue!12}22.78
& \cellcolor{g2ptqblue!12}33.63
& \cellcolor{g2ptqblue!12}23.25
& \cellcolor{g2ptqblue!12}30.18
& \cellcolor{g2ptqblue!12}4.70
& \cellcolor{g2ptqblue!12}57.13
& \cellcolor{g2ptqblue!12}50.36
& \cellcolor{g2ptqblue!12}31.72 \\
& \cellcolor{g2ptqblue!12}\name$^{*}$
& \cellcolor{g2ptqblue!12}\textbf{1.85e+00}
& \cellcolor{g2ptqblue!12}\textbf{125.39}
& \cellcolor{g2ptqblue!12}\textbf{1.35e+00}
& \cellcolor{g2ptqblue!12}\textbf{26.60}
& \cellcolor{g2ptqblue!12}\textbf{5.86e-01}
& \cellcolor{g2ptqblue!12}\textbf{7.76}
& \cellcolor{g2ptqblue!12}24.91
& \cellcolor{g2ptqblue!12}36.57
& \cellcolor{g2ptqblue!12}23.11
& \cellcolor{g2ptqblue!12}30.52
& \cellcolor{g2ptqblue!12}9.24
& \cellcolor{g2ptqblue!12}56.69
& \cellcolor{g2ptqblue!12}49.49
& \cellcolor{g2ptqblue!12}\textbf{32.93} \\
\midrule

\multirow{7}{*}{\begin{tabular}{@{}c@{}}\textbf{Qwen3}\\\textbf{1.7B}\end{tabular}}
& BF16 & 0 & 16.72 & 0 & 9.90 & 0 & 5.72 & 44.03 & 70.16 & 58.69 & 60.33 & 50.40 & 72.80 & 61.33 & 59.68 \\
\cmidrule{2-16}
& RTN & 1.60e+01 & 3.97e+07 & 1.65e+01 & 3.05e+07 & 1.69e+01 & 3.78e+07 & 26.88 & 24.16 & 26.23 & 26.31 & 0.00 & 51.80 & 47.83 & 29.03 \\
& GPTQ & 3.38e+00 & 266.11 & 3.42e+00 & 108.03 & 1.54e+00 & 12.93 & 24.40 & 28.32 & 24.52 & 28.15 & 0.82 & 51.52 & 51.14 & 29.84 \\
& GuidedQ & 3.17e+00 & 212.39 & 2.60e+00 & 54.06 & 1.14e+00 & 10.51 & 23.38 & 31.44 & 22.73 & 30.07 & 2.17 & 53.65 & 48.86 & 30.33 \\
& GPTAQ & 3.75e+00 & 573.88 & 4.44e+00 & 479.80 & 2.27e+00 & 44.90 & 24.15 & 27.53 & 24.44 & 28.22 & 1.44 & 53.10 & 49.49 & 29.77 \\
& \cellcolor{g2ptqblue!12}\name
& \cellcolor{g2ptqblue!12}1.91e+00
& \cellcolor{g2ptqblue!12}74.58
& \cellcolor{g2ptqblue!12}1.75e+00
& \cellcolor{g2ptqblue!12}34.78
& \cellcolor{g2ptqblue!12}7.38e-01
& \cellcolor{g2ptqblue!12}9.90
& \cellcolor{g2ptqblue!12}27.13
& \cellcolor{g2ptqblue!12}41.92
& \cellcolor{g2ptqblue!12}23.18
& \cellcolor{g2ptqblue!12}35.22
& \cellcolor{g2ptqblue!12}11.88
& \cellcolor{g2ptqblue!12}59.41
& \cellcolor{g2ptqblue!12}52.57
& \cellcolor{g2ptqblue!12}35.90 \\
& \cellcolor{g2ptqblue!12}\name$^{*}$
& \cellcolor{g2ptqblue!12}\textbf{1.58e+00}
& \cellcolor{g2ptqblue!12}\textbf{54.02}
& \cellcolor{g2ptqblue!12}\textbf{1.30e+00}
& \cellcolor{g2ptqblue!12}\textbf{21.31}
& \cellcolor{g2ptqblue!12}\textbf{4.97e-01}
& \cellcolor{g2ptqblue!12}\textbf{7.05}
& \cellcolor{g2ptqblue!12}26.45
& \cellcolor{g2ptqblue!12}43.56
& \cellcolor{g2ptqblue!12}22.59
& \cellcolor{g2ptqblue!12}37.91
& \cellcolor{g2ptqblue!12}14.57
& \cellcolor{g2ptqblue!12}59.41
& \cellcolor{g2ptqblue!12}54.70
& \cellcolor{g2ptqblue!12}\textbf{37.03} \\
\midrule

\multirow{7}{*}{\begin{tabular}{@{}c@{}}\textbf{Qwen3}\\\textbf{4B}\end{tabular}}
& BF16 & 0 & 13.66 & 0 & 8.03 & 0 & 5.14 & 54.18 & 78.07 & 70.43 & 68.50 & 59.50 & 74.92 & 65.67 & 67.32 \\
\cmidrule{2-16}
& RTN & 1.57e+01 & 2.28e+07 & 1.60e+01 & 1.66e+07 & 1.62e+01 & 1.66e+07 & 27.30 & 25.55 & 26.45 & 26.39 & 0.00 & 51.90 & 48.15 & 29.39 \\
& GPTQ & 2.17e+00 & 72.91 & 1.66e+00 & 17.12 & 9.10e-01 & 7.25 & 25.94 & 36.66 & 24.81 & 33.36 & 9.30 & 56.91 & 51.54 & 34.07 \\
& GuidedQ & 2.03e+00 & 68.09 & 1.56e+00 & 20.38 & 7.12e-01 & 7.28 & 26.28 & 38.93 & 25.63 & 34.79 & 9.12 & 57.18 & 50.43 & 34.62 \\
& GPTAQ & 1.89e+00 & 64.52 & 1.68e+00 & 24.87 & 7.39e-01 & 8.37 & 23.63 & 39.18 & 25.56 & 34.51 & 16.28 & 58.27 & 53.28 & 35.82 \\
& \cellcolor{g2ptqblue!12}\name
& \cellcolor{g2ptqblue!12}1.47e+00
& \cellcolor{g2ptqblue!12}39.65
& \cellcolor{g2ptqblue!12}1.05e+00
& \cellcolor{g2ptqblue!12}14.05
& \cellcolor{g2ptqblue!12}4.73e-01
& \cellcolor{g2ptqblue!12}6.47
& \cellcolor{g2ptqblue!12}32.00
& \cellcolor{g2ptqblue!12}53.66
& \cellcolor{g2ptqblue!12}30.01
& \cellcolor{g2ptqblue!12}43.63
& \cellcolor{g2ptqblue!12}24.67
& \cellcolor{g2ptqblue!12}64.96
& \cellcolor{g2ptqblue!12}54.70
& \cellcolor{g2ptqblue!12}43.38 \\
& \cellcolor{g2ptqblue!12}\name$^{*}$
& \cellcolor{g2ptqblue!12}\textbf{1.31e+00}
& \cellcolor{g2ptqblue!12}\textbf{35.40}
& \cellcolor{g2ptqblue!12}\textbf{8.86e-01}
& \cellcolor{g2ptqblue!12}\textbf{11.45}
& \cellcolor{g2ptqblue!12}\textbf{3.65e-01}
& \cellcolor{g2ptqblue!12}\textbf{5.37}
& \cellcolor{g2ptqblue!12}33.28
& \cellcolor{g2ptqblue!12}50.84
& \cellcolor{g2ptqblue!12}34.25
& \cellcolor{g2ptqblue!12}45.46
& \cellcolor{g2ptqblue!12}30.60
& \cellcolor{g2ptqblue!12}66.59
& \cellcolor{g2ptqblue!12}58.01
& \cellcolor{g2ptqblue!12}\textbf{45.58} \\
\midrule

\multirow{7}{*}{\begin{tabular}{@{}c@{}}\textbf{Qwen3}\\\textbf{8B}\end{tabular}}
& BF16 & 0 & 9.72 & 0 & 5.20 & 0 & 3.60 & 56.14 & 80.77 & 79.57 & 74.98 & 64.10 & 77.97 & 67.96 & 71.64 \\
\cmidrule{2-16}
& RTN & 1.72e+01 & 1.04e+08 & 1.80e+01 & 1.36e+08 & 1.81e+01 & 1.31e+08 & 25.85 & 25.04 & 25.78 & 26.63 & 0.00 & 51.63 & 49.33 & 29.18 \\
& GPTQ & 1.12e+00 & 21.86 & 7.74e-01 & 7.99 & 3.70e-01 & 4.51 & 32.17 & 55.72 & 37.96 & 50.66 & 39.57 & 67.19 & 56.91 & 48.60 \\
& GuidedQ & 1.14e+00 & 21.96 & 7.77e-01 & 8.07 & 3.60e-01 & 4.43 & 33.45 & 51.98 & 35.36 & 48.80 & 29.61 & 68.44 & 56.59 & 46.32 \\
& GPTAQ & 1.29e+00 & 27.88 & 1.02e+00 & 10.79 & 4.22e-01 & 5.04 & 33.02 & 54.29 & 30.31 & 44.12 & 29.83 & 65.29 & 53.75 & 44.37 \\
& \cellcolor{g2ptqblue!12}\name
& \cellcolor{g2ptqblue!12}9.43e-01
& \cellcolor{g2ptqblue!12}19.74
& \cellcolor{g2ptqblue!12}6.19e-01
& \cellcolor{g2ptqblue!12}7.64
& \cellcolor{g2ptqblue!12}2.65e-01
& \cellcolor{g2ptqblue!12}4.25
& \cellcolor{g2ptqblue!12}27.73
& \cellcolor{g2ptqblue!12}44.15
& \cellcolor{g2ptqblue!12}41.46
& \cellcolor{g2ptqblue!12}54.46
& \cellcolor{g2ptqblue!12}42.75
& \cellcolor{g2ptqblue!12}65.02
& \cellcolor{g2ptqblue!12}61.09
& \cellcolor{g2ptqblue!12}48.09 \\
& \cellcolor{g2ptqblue!12}\name$^{*}$
& \cellcolor{g2ptqblue!12}\textbf{8.28e-01}
& \cellcolor{g2ptqblue!12}\textbf{18.32}
& \cellcolor{g2ptqblue!12}\textbf{5.30e-01}
& \cellcolor{g2ptqblue!12}\textbf{7.02}
& \cellcolor{g2ptqblue!12}\textbf{2.21e-01}
& \cellcolor{g2ptqblue!12}\textbf{4.16}
& \cellcolor{g2ptqblue!12}42.66
& \cellcolor{g2ptqblue!12}68.77
& \cellcolor{g2ptqblue!12}45.99
& \cellcolor{g2ptqblue!12}55.88
& \cellcolor{g2ptqblue!12}45.12
& \cellcolor{g2ptqblue!12}71.55
& \cellcolor{g2ptqblue!12}61.48
& \cellcolor{g2ptqblue!12}\textbf{55.92} \\
\midrule

\multirow{7}{*}{\begin{tabular}{@{}c@{}}\textbf{Qwen3}\\\textbf{14B}\end{tabular}}
& BF16 & 0 & 8.65 & 0 & 4.80 & 0 & 3.36 & 60.49 & 82.74 & 82.32 & 78.80 & 67.90 & 80.03 & 73.01 & 75.04 \\
\cmidrule{2-16}
& RTN & 1.61e+01 & 3.37e+07 & 1.69e+01 & 4.88e+07 & 1.58e+01 & 1.31e+07 & 26.37 & 24.75 & 23.03 & 26.40 & 0.00 & 52.23 & 49.33 & 28.87 \\
& GPTQ & 8.34e-01 & 15.59 & 5.37e-01 & 6.09 & 2.33e-01 & \textbf{3.72} & 43.00 & 66.54 & 52.75 & 59.46 & 48.48 & 73.01 & 63.85 & 58.16 \\
& GuidedQ & 8.64e-01 & 16.02 & 5.45e-01 & 6.40 & 2.49e-01 & 3.90 & 44.80 & 69.28 & 50.67 & 58.63 & 46.40 & 72.09 & 63.22 & 57.87 \\
& GPTAQ & 8.13e-01 & 16.66 & 5.28e-01 & 6.72 & 2.14e-01 & 3.99 & 42.49 & 66.54 & 54.23 & 59.16 & 50.86 & 71.82 & 64.72 & 58.55 \\
& \cellcolor{g2ptqblue!12}\name
& \cellcolor{g2ptqblue!12}7.94e-01
& \cellcolor{g2ptqblue!12}16.21
& \cellcolor{g2ptqblue!12}4.61e-01
& \cellcolor{g2ptqblue!12}6.34
& \cellcolor{g2ptqblue!12}1.93e-01
& \cellcolor{g2ptqblue!12}3.85
& \cellcolor{g2ptqblue!12}45.31
& \cellcolor{g2ptqblue!12}72.90
& \cellcolor{g2ptqblue!12}51.93
& \cellcolor{g2ptqblue!12}63.32
& \cellcolor{g2ptqblue!12}55.77
& \cellcolor{g2ptqblue!12}73.78
& \cellcolor{g2ptqblue!12}65.98
& \cellcolor{g2ptqblue!12}\textbf{61.28} \\
& \cellcolor{g2ptqblue!12}\name$^{*}$
& \cellcolor{g2ptqblue!12}\textbf{7.00e-01}
& \cellcolor{g2ptqblue!12}\textbf{14.84}
& \cellcolor{g2ptqblue!12}\textbf{4.15e-01}
& \cellcolor{g2ptqblue!12}\textbf{5.97}
& \cellcolor{g2ptqblue!12}\textbf{1.66e-01}
& \cellcolor{g2ptqblue!12}3.74
& \cellcolor{g2ptqblue!12}45.48
& \cellcolor{g2ptqblue!12}73.53
& \cellcolor{g2ptqblue!12}46.88
& \cellcolor{g2ptqblue!12}64.08
& \cellcolor{g2ptqblue!12}53.66
& \cellcolor{g2ptqblue!12}74.70
& \cellcolor{g2ptqblue!12}67.09
& \cellcolor{g2ptqblue!12}60.77 \\
\midrule

\multirow{7}{*}{\begin{tabular}{@{}c@{}}\textbf{Qwen3}\\\textbf{32B}\end{tabular}}
& BF16 & 0 & 7.61 & 0 & 4.05 & 0 & 2.80 & 61.01 & 83.21 & 86.03 & 82.60 & 67.13 & 82.05 & 73.24 & 76.47 \\
\cmidrule{2-16}
& RTN & 1.40e+01 & 5.72e+06 & 1.41e+01 & 3.69e+06 & 1.41e+01 & 2.65e+06 & 26.79 & 25.34 & 25.56 & 26.28 & 0.00 & 51.14 & 50.04 & 29.31 \\
& GPTQ & 8.56e-01 & 13.65 & 4.25e-01 & 5.01 & 2.64e-01 & 3.24 & 47.27 & 68.98 & 54.23 & 66.28 & 56.24 & 73.99 & 62.04 & 61.29 \\
& GuidedQ & 9.06e-01 & 14.15 & 4.54e-01 & 5.23 & 2.58e-01 & 3.27 & 41.21 & 64.06 & 47.10 & 64.32 & 51.64 & 72.58 & 59.83 & 57.25 \\
& GPTAQ & 9.18e-01 & 14.78 & 4.65e-01 & 5.31 & 2.88e-01 & 3.44 & 43.26 & 61.15 & 55.87 & 66.09 & 60.70 & 73.78 & 62.35 & 60.46 \\
& \cellcolor{g2ptqblue!12}\name
& \cellcolor{g2ptqblue!12}9.84e-01
& \cellcolor{g2ptqblue!12}15.78
& \cellcolor{g2ptqblue!12}3.82e-01
& \cellcolor{g2ptqblue!12}5.03
& \cellcolor{g2ptqblue!12}2.56e-01
& \cellcolor{g2ptqblue!12}3.27
& \cellcolor{g2ptqblue!12}48.63
& \cellcolor{g2ptqblue!12}73.06
& \cellcolor{g2ptqblue!12}60.92
& \cellcolor{g2ptqblue!12}69.50
& \cellcolor{g2ptqblue!12}58.33
& \cellcolor{g2ptqblue!12}76.77
& \cellcolor{g2ptqblue!12}64.88
& \cellcolor{g2ptqblue!12}64.58 \\
& \cellcolor{g2ptqblue!12}\name$^{*}$
& \cellcolor{g2ptqblue!12}\textbf{7.72e-01}
& \cellcolor{g2ptqblue!12}\textbf{12.99}
& \cellcolor{g2ptqblue!12}\textbf{3.56e-01}
& \cellcolor{g2ptqblue!12}\textbf{4.92}
& \cellcolor{g2ptqblue!12}\textbf{2.01e-01}
& \cellcolor{g2ptqblue!12}\textbf{3.15}
& \cellcolor{g2ptqblue!12}50.17
& \cellcolor{g2ptqblue!12}74.83
& \cellcolor{g2ptqblue!12}57.73
& \cellcolor{g2ptqblue!12}70.01
& \cellcolor{g2ptqblue!12}58.20
& \cellcolor{g2ptqblue!12}75.63
& \cellcolor{g2ptqblue!12}68.11
& \cellcolor{g2ptqblue!12}\textbf{64.95} \\
\midrule

\multirow{7}{*}{\begin{tabular}{@{}c@{}}\textbf{Qwen3.5}\\\textbf{0.8B}\end{tabular}}
& BF16 & 0 & 17.18 & 0 & 5.54 & 0 & 3.40 & 37.63 & 61.49 & 48.89 & 49.54 & 43.59 & 69.31 & 57.62 & 52.58 \\
\cmidrule{2-16}
& RTN & 9.48e+00 & 1.88e+05 & 1.02e+01 & 1.35e+05 & 1.13e+01 & 2.17e+05 & 27.56 & 25.55 & 25.85 & 25.69 & 0.00 & 49.56 & 48.15 & 28.91 \\
& GPTQ & 2.24e+00 & 173.94 & 2.01e+00 & 39.33 & 8.43e-01 & 6.88 & 22.78 & 32.62 & 23.18 & 29.06 & 4.95 & 54.52 & 49.64 & 30.96 \\
& GuidedQ & 2.48e+00 & 221.11 & 2.15e+00 & 45.77 & 9.95e-01 & 8.12 & 22.10 & 33.42 & 23.70 & 28.72 & 4.99 & 54.46 & 48.22 & 30.80 \\
& GPTAQ & 2.16e+00 & 168.39 & 1.80e+00 & 32.46 & 7.64e-01 & 6.75 & 21.93 & 32.37 & 24.74 & 29.13 & 5.14 & 53.43 & 50.83 & 31.08 \\
& \cellcolor{g2ptqblue!12}\name
& \cellcolor{g2ptqblue!12}1.87e+00
& \cellcolor{g2ptqblue!12}113.10
& \cellcolor{g2ptqblue!12}1.52e+00
& \cellcolor{g2ptqblue!12}24.81
& \cellcolor{g2ptqblue!12}7.07e-01
& \cellcolor{g2ptqblue!12}6.61
& \cellcolor{g2ptqblue!12}22.95
& \cellcolor{g2ptqblue!12}38.01
& \cellcolor{g2ptqblue!12}24.00
& \cellcolor{g2ptqblue!12}29.90
& \cellcolor{g2ptqblue!12}6.91
& \cellcolor{g2ptqblue!12}56.53
& \cellcolor{g2ptqblue!12}52.72
& \cellcolor{g2ptqblue!12}33.00 \\
& \cellcolor{g2ptqblue!12}\name$^{*}$
& \cellcolor{g2ptqblue!12}\textbf{1.52e+00}
& \cellcolor{g2ptqblue!12}\textbf{79.63}
& \cellcolor{g2ptqblue!12}\textbf{1.16e+00}
& \cellcolor{g2ptqblue!12}\textbf{17.11}
& \cellcolor{g2ptqblue!12}\textbf{4.79e-01}
& \cellcolor{g2ptqblue!12}\textbf{5.27}
& \cellcolor{g2ptqblue!12}23.12
& \cellcolor{g2ptqblue!12}39.10
& \cellcolor{g2ptqblue!12}23.40
& \cellcolor{g2ptqblue!12}31.34
& \cellcolor{g2ptqblue!12}10.75
& \cellcolor{g2ptqblue!12}56.86
& \cellcolor{g2ptqblue!12}51.30
& \cellcolor{g2ptqblue!12}\textbf{33.70} \\
\midrule

\multirow{7}{*}{\begin{tabular}{@{}c@{}}\textbf{Qwen3.5}\\\textbf{2B}\end{tabular}}
& BF16 & 0 & 12.10 & 0 & 4.66 & 0 & 2.99 & 41.47 & 65.95 & 59.06 & 62.11 & 53.31 & 72.69 & 62.67 & 59.61 \\
\cmidrule{2-16}
& RTN & 8.28e+00 & 4.28e+04 & 9.59e+00 & 5.99e+04 & 1.01e+01 & 6.08e+04 & 25.51 & 26.30 & 24.07 & 25.63 & 0.00 & 50.82 & 49.41 & 28.82 \\
& GPTQ & 1.99e+00 & 94.30 & 1.59e+00 & 20.14 & 6.73e-01 & 5.42 & 22.35 & 32.87 & 24.44 & 32.09 & 9.22 & 52.77 & 51.54 & 32.18 \\
& GuidedQ & 2.26e+00 & 121.94 & 1.79e+00 & 25.14 & 8.05e-01 & 6.15 & 22.95 & 36.70 & 23.92 & 30.50 & 7.63 & 56.31 & 50.83 & 32.69 \\
& GPTAQ & 1.93e+00 & 90.42 & 1.49e+00 & 18.98 & 5.80e-01 & 5.06 & 23.46 & 31.40 & 23.48 & 31.81 & 6.71 & 54.52 & 49.09 & 31.50 \\
& \cellcolor{g2ptqblue!12}\name
& \cellcolor{g2ptqblue!12}1.55e+00
& \cellcolor{g2ptqblue!12}57.50
& \cellcolor{g2ptqblue!12}1.11e+00
& \cellcolor{g2ptqblue!12}12.95
& \cellcolor{g2ptqblue!12}4.98e-01
& \cellcolor{g2ptqblue!12}4.74
& \cellcolor{g2ptqblue!12}25.26
& \cellcolor{g2ptqblue!12}39.73
& \cellcolor{g2ptqblue!12}23.25
& \cellcolor{g2ptqblue!12}34.75
& \cellcolor{g2ptqblue!12}16.81
& \cellcolor{g2ptqblue!12}58.81
& \cellcolor{g2ptqblue!12}53.99
& \cellcolor{g2ptqblue!12}36.09 \\
& \cellcolor{g2ptqblue!12}\name$^{*}$
& \cellcolor{g2ptqblue!12}\textbf{1.29e+00}
& \cellcolor{g2ptqblue!12}\textbf{45.43}
& \cellcolor{g2ptqblue!12}\textbf{8.79e-01}
& \cellcolor{g2ptqblue!12}\textbf{10.39}
& \cellcolor{g2ptqblue!12}\textbf{3.58e-01}
& \cellcolor{g2ptqblue!12}\textbf{4.15}
& \cellcolor{g2ptqblue!12}27.90
& \cellcolor{g2ptqblue!12}41.96
& \cellcolor{g2ptqblue!12}23.18
& \cellcolor{g2ptqblue!12}37.94
& \cellcolor{g2ptqblue!12}18.16
& \cellcolor{g2ptqblue!12}59.63
& \cellcolor{g2ptqblue!12}55.49
& \cellcolor{g2ptqblue!12}\textbf{37.75} \\
\midrule

\multirow{7}{*}{\begin{tabular}{@{}c@{}}\textbf{Qwen3.5}\\\textbf{4B}\end{tabular}}
& BF16 & 0 & 9.58 & 0 & 4.33 & 0 & 2.72 & 54.27 & 75.46 & 72.51 & 73.11 & 64.74 & 77.97 & 69.93 & 69.71 \\
\cmidrule{2-16}
& RTN & 1.35e+01 & 4.91e+06 & 1.46e+01 & 6.70e+06 & 1.50e+01 & 7.07e+06 & 27.39 & 24.49 & 25.41 & 26.60 & 0.00 & 50.92 & 49.01 & 29.12 \\
& GPTQ & 1.47e+00 & 37.95 & 1.12e+00 & 9.67 & 4.43e-01 & 3.88 & 31.23 & 52.61 & 23.25 & 39.82 & 23.68 & 63.22 & 53.35 & 41.02 \\
& GuidedQ & 1.90e+00 & 59.52 & 1.39e+00 & 13.23 & 6.15e-01 & 4.66 & 27.99 & 46.13 & 23.18 & 35.24 & 14.98 & 60.34 & 52.72 & 37.23 \\
& GPTAQ & 1.61e+00 & 45.39 & 1.11e+00 & 9.91 & 4.06e-01 & 3.88 & 30.72 & 45.92 & 24.37 & 40.01 & 17.08 & 61.97 & 54.46 & 39.22 \\
& \cellcolor{g2ptqblue!12}\name
& \cellcolor{g2ptqblue!12}1.15e+00
& \cellcolor{g2ptqblue!12}27.49
& \cellcolor{g2ptqblue!12}8.14e-01
& \cellcolor{g2ptqblue!12}7.98
& \cellcolor{g2ptqblue!12}3.45e-01
& \cellcolor{g2ptqblue!12}3.63
& \cellcolor{g2ptqblue!12}34.56
& \cellcolor{g2ptqblue!12}57.28
& \cellcolor{g2ptqblue!12}26.89
& \cellcolor{g2ptqblue!12}46.78
& \cellcolor{g2ptqblue!12}35.44
& \cellcolor{g2ptqblue!12}65.67
& \cellcolor{g2ptqblue!12}55.96
& \cellcolor{g2ptqblue!12}46.08 \\
& \cellcolor{g2ptqblue!12}\name$^{*}$
& \cellcolor{g2ptqblue!12}\textbf{1.03e+00}
& \cellcolor{g2ptqblue!12}\textbf{24.70}
& \cellcolor{g2ptqblue!12}\textbf{7.06e-01}
& \cellcolor{g2ptqblue!12}\textbf{7.63}
& \cellcolor{g2ptqblue!12}\textbf{2.74e-01}
& \cellcolor{g2ptqblue!12}\textbf{3.43}
& \cellcolor{g2ptqblue!12}36.52
& \cellcolor{g2ptqblue!12}60.90
& \cellcolor{g2ptqblue!12}33.14
& \cellcolor{g2ptqblue!12}47.91
& \cellcolor{g2ptqblue!12}35.09
& \cellcolor{g2ptqblue!12}66.87
& \cellcolor{g2ptqblue!12}58.01
& \cellcolor{g2ptqblue!12}\textbf{48.35} \\
\midrule

\multirow{7}{*}{\begin{tabular}{@{}c@{}}\textbf{Qwen3.5}\\\textbf{9B}\end{tabular}}
& BF16 & 0 & 8.65 & 0 & 3.78 & 0 & 2.48 & 55.46 & 74.37 & 78.75 & 78.08 & 70.04 & 79.98 & 72.93 & 72.80 \\
\cmidrule{2-16}
& RTN & 1.42e+01 & 8.84e+06 & 1.49e+01 & 8.69e+06 & 1.53e+01 & 1.01e+07 & 25.77 & 25.88 & 26.15 & 26.32 & 0.00 & 49.95 & 50.36 & 29.20 \\
& GPTQ & 8.87e-01 & 16.90 & 6.00e-01 & 5.79 & 2.27e-01 & 3.02 & 42.83 & 65.19 & 47.70 & 55.94 & 47.97 & 70.51 & 61.48 & 55.95 \\
& GuidedQ & 1.05e+00 & 20.29 & 6.93e-01 & 6.43 & 3.17e-01 & 3.33 & 36.86 & 59.05 & 41.08 & 51.49 & 39.45 & 69.37 & 59.35 & 50.95 \\
& GPTAQ & 9.66e-01 & 19.15 & 6.14e-01 & 6.18 & 2.25e-01 & 3.03 & 41.13 & 65.53 & 45.54 & 54.16 & 44.27 & 69.48 & 61.48 & 54.51 \\
& \cellcolor{g2ptqblue!12}\name
& \cellcolor{g2ptqblue!12}7.82e-01
& \cellcolor{g2ptqblue!12}15.85
& \cellcolor{g2ptqblue!12}5.02e-01
& \cellcolor{g2ptqblue!12}5.58
& \cellcolor{g2ptqblue!12}2.09e-01
& \cellcolor{g2ptqblue!12}3.00
& \cellcolor{g2ptqblue!12}43.09
& \cellcolor{g2ptqblue!12}68.98
& \cellcolor{g2ptqblue!12}52.01
& \cellcolor{g2ptqblue!12}59.68
& \cellcolor{g2ptqblue!12}47.18
& \cellcolor{g2ptqblue!12}72.09
& \cellcolor{g2ptqblue!12}63.38
& \cellcolor{g2ptqblue!12}58.06 \\
& \cellcolor{g2ptqblue!12}\name$^{*}$
& \cellcolor{g2ptqblue!12}\textbf{7.17e-01}
& \cellcolor{g2ptqblue!12}\textbf{14.99}
& \cellcolor{g2ptqblue!12}\textbf{4.44e-01}
& \cellcolor{g2ptqblue!12}\textbf{5.27}
& \cellcolor{g2ptqblue!12}\textbf{1.74e-01}
& \cellcolor{g2ptqblue!12}\textbf{2.90}
& \cellcolor{g2ptqblue!12}43.17
& \cellcolor{g2ptqblue!12}68.81
& \cellcolor{g2ptqblue!12}49.85
& \cellcolor{g2ptqblue!12}59.36
& \cellcolor{g2ptqblue!12}51.76
& \cellcolor{g2ptqblue!12}71.76
& \cellcolor{g2ptqblue!12}65.27
& \cellcolor{g2ptqblue!12}\textbf{58.57} \\
\midrule

\multirow{7}{*}{\begin{tabular}{@{}c@{}}\textbf{Qwen3.5}\\\textbf{27B}\end{tabular}}
& BF16 & 0 & 6.90 & 0 & 3.55 & 0 & 2.28 & 61.26 & 79.67 & 85.44 & 83.34 & 75.16 & 82.05 & 78.69 & 77.94 \\
\cmidrule{2-16}
& RTN & 1.24e+01 & 1.26e+06 & 1.36e+01 & 1.97e+06 & 1.26e+01 & 5.80e+05 & 26.02 & 23.86 & 25.56 & 26.38 & 0.00 & 51.41 & 50.12 & 29.05 \\
& GPTQ & 5.43e-01 & 10.85 & 3.06e-01 & \textbf{4.29} & 1.32e-01 & 2.54 & 56.83 & 79.34 & 69.54 & 71.18 & 69.36 & 77.04 & 73.24 & \textbf{70.93} \\
& GuidedQ & 6.44e-01 & 11.82 & 3.65e-01 & 4.50 & 1.60e-01 & 2.63 & 54.61 & 79.34 & 63.22 & 68.69 & 67.82 & 76.55 & 71.19 & 68.77 \\
& GPTAQ & 6.15e-01 & 12.06 & 3.34e-01 & 4.48 & 1.32e-01 & 2.55 & 56.06 & 80.47 & 68.57 & 69.80 & 65.30 & 78.18 & 72.30 & 70.10 \\
& \cellcolor{g2ptqblue!12}\name
& \cellcolor{g2ptqblue!12}5.36e-01
& \cellcolor{g2ptqblue!12}10.89
& \cellcolor{g2ptqblue!12}2.92e-01
& \cellcolor{g2ptqblue!12}4.34
& \cellcolor{g2ptqblue!12}1.30e-01
& \cellcolor{g2ptqblue!12}2.55
& \cellcolor{g2ptqblue!12}55.72
& \cellcolor{g2ptqblue!12}80.09
& \cellcolor{g2ptqblue!12}70.80
& \cellcolor{g2ptqblue!12}71.33
& \cellcolor{g2ptqblue!12}67.22
& \cellcolor{g2ptqblue!12}76.50
& \cellcolor{g2ptqblue!12}71.03
& \cellcolor{g2ptqblue!12}70.38 \\
& \cellcolor{g2ptqblue!12}\name$^{*}$
& \cellcolor{g2ptqblue!12}\textbf{5.22e-01}
& \cellcolor{g2ptqblue!12}\textbf{10.76}
& \cellcolor{g2ptqblue!12}\textbf{2.74e-01}
& \cellcolor{g2ptqblue!12}4.30
& \cellcolor{g2ptqblue!12}\textbf{1.08e-01}
& \cellcolor{g2ptqblue!12}\textbf{2.51}
& \cellcolor{g2ptqblue!12}54.86
& \cellcolor{g2ptqblue!12}79.46
& \cellcolor{g2ptqblue!12}69.99
& \cellcolor{g2ptqblue!12}70.94
& \cellcolor{g2ptqblue!12}66.00
& \cellcolor{g2ptqblue!12}77.48
& \cellcolor{g2ptqblue!12}70.96
& \cellcolor{g2ptqblue!12}69.96 \\
\midrule

\multirow{7}{*}{\begin{tabular}{@{}c@{}}\textbf{LLaMA3}\\\textbf{8B}\end{tabular}}
& BF16 & 0 & 6.14 & 0 & 3.41 & 0 & 2.72 & 53.24 & 77.90 & 48.44 & 79.25 & 75.53 & 80.74 & 73.16 & 69.75 \\
\cmidrule{2-16}
& RTN & 1.23e+01 & 1.37e+06 & 1.29e+01 & 1.38e+06 & 1.33e+01 & 1.68e+06 & 27.22 & 25.21 & 23.03 & 25.79 & 0.00 & 52.83 & 49.41 & 29.07 \\
& GPTQ & 1.85e+00 & 39.49 & 1.18e+00 & 10.74 & 4.16e-01 & 4.09 & 27.13 & 42.72 & 24.67 & 42.00 & 17.80 & 62.89 & 53.04 & 38.61 \\
& GuidedQ & 1.48e+00 & 27.31 & 8.88e-01 & 8.02 & 3.31e-01 & 3.76 & 30.03 & 52.23 & 25.56 & 47.76 & 25.36 & 64.15 & 57.06 & 43.16 \\
& GPTAQ & 1.48e+00 & 27.27 & 7.15e-01 & 6.86 & 3.12e-01 & 3.67 & 28.24 & 48.70 & 23.77 & 47.43 & 26.55 & 63.22 & 58.72 & 42.38 \\
& \cellcolor{g2ptqblue!12}\name
& \cellcolor{g2ptqblue!12}1.72e+00
& \cellcolor{g2ptqblue!12}34.60
& \cellcolor{g2ptqblue!12}6.08e-01
& \cellcolor{g2ptqblue!12}6.18
& \cellcolor{g2ptqblue!12}2.66e-01
& \cellcolor{g2ptqblue!12}3.49
& \cellcolor{g2ptqblue!12}32.76
& \cellcolor{g2ptqblue!12}53.75
& \cellcolor{g2ptqblue!12}28.16
& \cellcolor{g2ptqblue!12}53.94
& \cellcolor{g2ptqblue!12}38.87
& \cellcolor{g2ptqblue!12}65.67
& \cellcolor{g2ptqblue!12}59.12
& \cellcolor{g2ptqblue!12}47.47 \\
& \cellcolor{g2ptqblue!12}\name$^{*}$
& \cellcolor{g2ptqblue!12}\textbf{1.09e+00}
& \cellcolor{g2ptqblue!12}\textbf{18.50}
& \cellcolor{g2ptqblue!12}\textbf{5.35e-01}
& \cellcolor{g2ptqblue!12}\textbf{5.75}
& \cellcolor{g2ptqblue!12}\textbf{2.11e-01}
& \cellcolor{g2ptqblue!12}\textbf{3.29}
& \cellcolor{g2ptqblue!12}37.71
& \cellcolor{g2ptqblue!12}62.21
& \cellcolor{g2ptqblue!12}29.20
& \cellcolor{g2ptqblue!12}56.57
& \cellcolor{g2ptqblue!12}43.92
& \cellcolor{g2ptqblue!12}69.15
& \cellcolor{g2ptqblue!12}59.35
& \cellcolor{g2ptqblue!12}\textbf{51.16} \\
\midrule

\multirow{7}{*}{\begin{tabular}{@{}c@{}}\textbf{LLaMA3}\\\textbf{70B}\end{tabular}}
& BF16 & 0 & 2.86 & 0 & 3.05 & 0 & 2.41 & 64.16 & 85.86 & 65.08 & 84.95 & 79.45 & 84.49 & 80.43 & 77.77 \\
\cmidrule{2-16}
& RTN & 1.19e+01 & 4.46e+05 & 1.22e+01 & 5.68e+05 & 1.25e+01 & 6.49e+05 & 26.19 & 25.76 & 23.03 & 26.18 & 0.00 & 52.01 & 49.88 & 29.01 \\
& GPTQ & 2.75e+00 & 45.45 & 1.75e+00 & 16.84 & 8.13e-01 & 5.28 & 21.84 & 34.47 & 22.66 & 32.12 & 9.66 & 56.47 & 52.41 & 32.80 \\
& GuidedQ & \textbf{1.20e+00} & \textbf{9.58} & \textbf{2.86e-01} & \textbf{4.00} & \textbf{1.68e-01} & \textbf{2.80} & 37.37 & 60.52 & 32.76 & 69.49 & 70.58 & 68.39 & 71.35 & 58.64 \\
& GPTAQ & 1.83e+00 & 17.91 & 5.91e-01 & 5.41 & 2.41e-01 & 2.99 & 37.54 & 64.73 & 26.08 & 47.16 & 37.03 & 69.86 & 66.14 & 49.79 \\
& \cellcolor{g2ptqblue!12}\name
& \cellcolor{g2ptqblue!12}1.93e+00
& \cellcolor{g2ptqblue!12}19.88
& \cellcolor{g2ptqblue!12}5.01e-01
& \cellcolor{g2ptqblue!12}5.00
& \cellcolor{g2ptqblue!12}2.60e-01
& \cellcolor{g2ptqblue!12}3.07
& \cellcolor{g2ptqblue!12}26.62
& \cellcolor{g2ptqblue!12}25.76
& \cellcolor{g2ptqblue!12}32.10
& \cellcolor{g2ptqblue!12}60.60
& \cellcolor{g2ptqblue!12}63.03
& \cellcolor{g2ptqblue!12}52.39
& \cellcolor{g2ptqblue!12}70.17
& \cellcolor{g2ptqblue!12}47.24 \\
& \cellcolor{g2ptqblue!12}\name$^{*}$
& \cellcolor{g2ptqblue!12}1.56e+00
& \cellcolor{g2ptqblue!12}13.68
& \cellcolor{g2ptqblue!12}4.14e-01
& \cellcolor{g2ptqblue!12}4.57
& \cellcolor{g2ptqblue!12}1.94e-01
& \cellcolor{g2ptqblue!12}2.86
& \cellcolor{g2ptqblue!12}53.33
& \cellcolor{g2ptqblue!12}79.29
& \cellcolor{g2ptqblue!12}47.55
& \cellcolor{g2ptqblue!12}69.71
& \cellcolor{g2ptqblue!12}68.39
& \cellcolor{g2ptqblue!12}76.99
& \cellcolor{g2ptqblue!12}73.01
& \cellcolor{g2ptqblue!12}\textbf{66.90} \\

\bottomrule
\end{tabular}
\end{adjustbox}
\vspace{-1ex}
\caption{2-bit weight-only quantization results across different model families.}
\label{tab:w2}
\end{table*}

%% file: tables/appendix/w3.tex
\begin{table*}[!t]
\centering
\begin{adjustbox}{
  max width=\linewidth,
  max totalheight=0.95\textheight,
  keepaspectratio
}
\begin{tabular}{c|l|cc|cc|cc|cccccccc}
\toprule
\multirow{2}{*}{\textbf{Model}} & \multirow{2}{*}{\textbf{Method}} & \multicolumn{2}{c|}{\textbf{WikiText2}} & \multicolumn{2}{c|}{\textbf{UltraChat}} & \multicolumn{2}{c|}{\textbf{NuminaMath}} & \multicolumn{8}{c}{\textbf{CommonSense QA}} \\
\cmidrule(lr){3-4} \cmidrule(lr){5-6} \cmidrule(lr){7-8} \cmidrule(lr){9-16}
& & \textbf{KL} & \textbf{PPL} & \textbf{KL} & \textbf{PPL} & \textbf{KL} & \textbf{PPL} & \textbf{ARC-C} & \textbf{ARC-E} & \textbf{C-eval} & \textbf{HellaS} & \textbf{LAMB} & \textbf{PIQA} & \textbf{Wino} & \textbf{Avg.} \\
\midrule

\multirow{7}{*}{\begin{tabular}{@{}c@{}}\textbf{Qwen3}\\\textbf{0.6B}\end{tabular}}
& BF16 & 0 & 20.96 & 0 & 8.00 & 0 & 4.88 & 34.13 & 56.06 & 43.76 & 47.29 & 40.02 & 67.36 & 56.35 & 49.28 \\
\cmidrule{2-16}
& RTN & 7.99e+00 & 3.51e+04 & 7.89e+00 & 1.18e+04 & 6.43e+00 & 1.62e+03 & 25.09 & 27.31 & 25.19 & 27.84 & 0.21 & 53.43 & 48.70 & 29.68 \\
& GPTQ & 1.82e+00 & 128.54 & 1.61e+00 & 27.86 & 1.08e+00 & 10.40 & 26.28 & 31.82 & 23.55 & 33.34 & 9.08 & 56.37 & 51.22 & 33.09 \\
& GuidedQ & 6.86e-01 & 35.71 & 5.06e-01 & 11.76 & 2.95e-01 & 5.80 & 28.41 & 40.74 & 22.73 & 39.37 & 21.13 & 62.35 & 53.20 & 38.28 \\
& GPTAQ & 1.15e+00 & 65.81 & 9.73e-01 & 17.92 & 5.53e-01 & 7.48 & 25.17 & 33.67 & 24.52 & 35.90 & 14.48 & 58.22 & 50.12 & 34.58 \\
& \cellcolor{g2ptqblue!12}\name & \cellcolor{g2ptqblue!12}5.58e-01 & \cellcolor{g2ptqblue!12}33.80 & \cellcolor{g2ptqblue!12}4.02e-01 & \cellcolor{g2ptqblue!12}11.27 & \cellcolor{g2ptqblue!12}2.02e-01 & \cellcolor{g2ptqblue!12}5.68 & \cellcolor{g2ptqblue!12}26.88 & \cellcolor{g2ptqblue!12}43.35 & \cellcolor{g2ptqblue!12}23.63 & \cellcolor{g2ptqblue!12}39.55 & \cellcolor{g2ptqblue!12}25.03 & \cellcolor{g2ptqblue!12}62.24 & \cellcolor{g2ptqblue!12}53.83 & \cellcolor{g2ptqblue!12}39.22 \\
& \cellcolor{g2ptqblue!12}\name\textsuperscript{*} & \cellcolor{g2ptqblue!12}\textbf{4.97e-01} & \cellcolor{g2ptqblue!12}\textbf{32.34} & \cellcolor{g2ptqblue!12}\textbf{3.39e-01} & \cellcolor{g2ptqblue!12}\textbf{10.52} & \cellcolor{g2ptqblue!12}\textbf{1.63e-01} & \cellcolor{g2ptqblue!12}\textbf{5.49} & \cellcolor{g2ptqblue!12}28.24 & \cellcolor{g2ptqblue!12}42.55 & \cellcolor{g2ptqblue!12}24.15 & \cellcolor{g2ptqblue!12}39.77 & \cellcolor{g2ptqblue!12}26.49 & \cellcolor{g2ptqblue!12}63.33 & \cellcolor{g2ptqblue!12}54.30 & \cellcolor{g2ptqblue!12}\textbf{39.83} \\
\midrule

\multirow{7}{*}{\begin{tabular}{@{}c@{}}\textbf{Qwen3}\\\textbf{1.7B}\end{tabular}}
& BF16 & 0 & 16.72 & 0 & 9.90 & 0 & 5.72 & 44.03 & 70.16 & 58.69 & 60.33 & 50.40 & 72.80 & 61.33 & 59.68 \\
\cmidrule{2-16}
& RTN & 1.02e+01 & 2.01e+05 & 8.63e+00 & 1.72e+04 & 1.01e+01 & 5.03e+04 & 26.02 & 28.62 & 24.29 & 27.96 & 0.10 & 52.39 & 50.51 & 29.98 \\
& GPTQ & 1.37e+00 & 70.28 & 8.77e-01 & 14.70 & 5.86e-01 & 9.31 & 29.18 & 42.59 & 32.84 & 46.55 & 22.20 & 61.48 & 54.14 & 41.28 \\
& GuidedQ & 7.49e-01 & 26.52 & 5.05e-01 & 11.39 & 2.64e-01 & 6.16 & 34.13 & 54.92 & 33.73 & 50.92 & 31.42 & 66.05 & 57.85 & 47.00 \\
& GPTAQ & 8.31e-01 & 37.51 & 6.19e-01 & 16.02 & 3.19e-01 & 8.24 & 32.17 & 42.47 & 37.30 & 49.88 & 31.01 & 65.13 & 57.38 & 45.05 \\
& \cellcolor{g2ptqblue!12}\name & \cellcolor{g2ptqblue!12}5.06e-01 & \cellcolor{g2ptqblue!12}22.61 & \cellcolor{g2ptqblue!12}4.13e-01 & \cellcolor{g2ptqblue!12}13.16 & \cellcolor{g2ptqblue!12}1.88e-01 & \cellcolor{g2ptqblue!12}6.44 & \cellcolor{g2ptqblue!12}34.64 & \cellcolor{g2ptqblue!12}54.92 & \cellcolor{g2ptqblue!12}38.86 & \cellcolor{g2ptqblue!12}53.00 & \cellcolor{g2ptqblue!12}35.46 & \cellcolor{g2ptqblue!12}67.68 & \cellcolor{g2ptqblue!12}57.06 & \cellcolor{g2ptqblue!12}48.80 \\
& \cellcolor{g2ptqblue!12}\name\textsuperscript{*} & \cellcolor{g2ptqblue!12}\textbf{4.73e-01} & \cellcolor{g2ptqblue!12}\textbf{21.66} & \cellcolor{g2ptqblue!12}\textbf{3.52e-01} & \cellcolor{g2ptqblue!12}\textbf{11.76} & \cellcolor{g2ptqblue!12}\textbf{1.51e-01} & \cellcolor{g2ptqblue!12}\textbf{5.69} & \cellcolor{g2ptqblue!12}37.29 & \cellcolor{g2ptqblue!12}64.35 & \cellcolor{g2ptqblue!12}37.96 & \cellcolor{g2ptqblue!12}53.23 & \cellcolor{g2ptqblue!12}39.84 & \cellcolor{g2ptqblue!12}68.99 & \cellcolor{g2ptqblue!12}57.62 & \cellcolor{g2ptqblue!12}\textbf{51.33} \\
\midrule

\multirow{7}{*}{\begin{tabular}{@{}c@{}}\textbf{Qwen3}\\\textbf{4B}\end{tabular}}
& BF16 & 0 & 13.66 & 0 & 8.03 & 0 & 5.14 & 54.18 & 78.07 & 70.43 & 68.50 & 59.50 & 74.92 & 65.67 & 67.32 \\
\cmidrule{2-16}
& RTN & 2.09e+00 & 67.93 & 1.54e+00 & 14.78 & 1.42e+00 & 9.76 & 27.82 & 41.12 & 26.08 & 41.30 & 9.37 & 60.01 & 52.57 & 36.90 \\
& GPTQ & 6.15e-01 & 19.07 & 4.45e-01 & 7.15 & 2.93e-01 & 4.63 & 43.86 & 70.24 & 40.56 & 58.23 & 48.57 & 69.37 & 59.98 & 55.83 \\
& GuidedQ & 4.78e-01 & 19.05 & 3.21e-01 & 8.10 & 1.70e-01 & 4.91 & 42.06 & 65.03 & 49.93 & 60.36 & 47.37 & 71.76 & 61.88 & 56.91 \\
& GPTAQ & 6.03e-01 & 22.91 & 4.79e-01 & 10.39 & 2.51e-01 & 5.99 & 43.26 & 66.62 & 45.91 & 56.52 & 48.61 & 70.84 & 58.88 & 55.81 \\
& \cellcolor{g2ptqblue!12}\name & \cellcolor{g2ptqblue!12}4.31e-01 & \cellcolor{g2ptqblue!12}\textbf{18.30} & \cellcolor{g2ptqblue!12}2.68e-01 & \cellcolor{g2ptqblue!12}8.25 & \cellcolor{g2ptqblue!12}1.28e-01 & \cellcolor{g2ptqblue!12}5.05 & \cellcolor{g2ptqblue!12}45.56 & \cellcolor{g2ptqblue!12}69.32 & \cellcolor{g2ptqblue!12}53.64 & \cellcolor{g2ptqblue!12}61.41 & \cellcolor{g2ptqblue!12}52.05 & \cellcolor{g2ptqblue!12}71.27 & \cellcolor{g2ptqblue!12}62.12 & \cellcolor{g2ptqblue!12}59.34 \\
& \cellcolor{g2ptqblue!12}\name\textsuperscript{*} & \cellcolor{g2ptqblue!12}\textbf{4.10e-01} & \cellcolor{g2ptqblue!12}18.38 & \cellcolor{g2ptqblue!12}\textbf{2.43e-01} & \cellcolor{g2ptqblue!12}\textbf{7.88} & \cellcolor{g2ptqblue!12}\textbf{1.07e-01} & \cellcolor{g2ptqblue!12}\textbf{4.84} & \cellcolor{g2ptqblue!12}46.16 & \cellcolor{g2ptqblue!12}66.25 & \cellcolor{g2ptqblue!12}59.51 & \cellcolor{g2ptqblue!12}61.70 & \cellcolor{g2ptqblue!12}52.18 & \cellcolor{g2ptqblue!12}72.03 & \cellcolor{g2ptqblue!12}64.25 & \cellcolor{g2ptqblue!12}\textbf{60.30} \\
\midrule

\multirow{7}{*}{\begin{tabular}{@{}c@{}}\textbf{Qwen3}\\\textbf{8B}\end{tabular}}
& BF16 & 0 & 9.72 & 0 & 5.20 & 0 & 3.60 & 56.14 & 80.77 & 79.57 & 74.98 & 64.10 & 77.97 & 67.96 & 71.64 \\
\cmidrule{2-16}
& RTN & 1.01e+00 & 22.61 & 8.44e-01 & 8.18 & 6.89e-01 & 5.24 & 35.92 & 51.89 & 52.75 & 54.78 & 38.15 & 70.18 & 58.96 & 51.80 \\
& GPTQ & 2.21e-01 & 11.03 & 1.41e-01 & 5.49 & 6.75e-02 & 3.71 & 52.73 & 76.77 & 70.73 & 70.59 & 61.87 & 75.79 & 65.82 & 67.76 \\
& GuidedQ & 2.21e-01 & 11.24 & 1.37e-01 & 5.55 & 6.65e-02 & 3.81 & 50.34 & 77.53 & 71.77 & 69.72 & 62.02 & 75.95 & 63.85 & 67.31 \\
& GPTAQ & 2.27e-01 & 11.54 & 1.48e-01 & 5.69 & 6.82e-02 & 3.89 & 50.77 & 77.02 & 71.17 & 68.94 & 61.96 & 75.79 & 67.32 & 67.57 \\
& \cellcolor{g2ptqblue!12}\name & \cellcolor{g2ptqblue!12}\textbf{2.01e-01} & \cellcolor{g2ptqblue!12}11.04 & \cellcolor{g2ptqblue!12}1.28e-01 & \cellcolor{g2ptqblue!12}5.53 & \cellcolor{g2ptqblue!12}6.55e-02 & \cellcolor{g2ptqblue!12}3.78 & \cellcolor{g2ptqblue!12}53.24 & \cellcolor{g2ptqblue!12}79.17 & \cellcolor{g2ptqblue!12}71.40 & \cellcolor{g2ptqblue!12}70.64 & \cellcolor{g2ptqblue!12}62.41 & \cellcolor{g2ptqblue!12}76.01 & \cellcolor{g2ptqblue!12}68.43 & \cellcolor{g2ptqblue!12}\textbf{68.76} \\
& \cellcolor{g2ptqblue!12}\name\textsuperscript{*} & \cellcolor{g2ptqblue!12}\textbf{2.01e-01} & \cellcolor{g2ptqblue!12}\textbf{10.99} & \cellcolor{g2ptqblue!12}\textbf{1.19e-01} & \cellcolor{g2ptqblue!12}\textbf{5.38} & \cellcolor{g2ptqblue!12}\textbf{5.39e-02} & \cellcolor{g2ptqblue!12}\textbf{3.72} & \cellcolor{g2ptqblue!12}53.41 & \cellcolor{g2ptqblue!12}78.75 & \cellcolor{g2ptqblue!12}71.47 & \cellcolor{g2ptqblue!12}71.00 & \cellcolor{g2ptqblue!12}62.76 & \cellcolor{g2ptqblue!12}76.01 & \cellcolor{g2ptqblue!12}67.64 & \cellcolor{g2ptqblue!12}68.72 \\
\midrule

\multirow{7}{*}{\begin{tabular}{@{}c@{}}\textbf{Qwen3}\\\textbf{14B}\end{tabular}}
& BF16 & 0 & 8.65 & 0 & 4.80 & 0 & 3.36 & 60.49 & 82.74 & 82.32 & 78.80 & 67.90 & 80.03 & 73.01 & 75.04 \\
\cmidrule{2-16}
& RTN & 8.37e-01 & 17.96 & 6.56e-01 & 6.73 & 4.69e-01 & 4.17 & 44.45 & 66.25 & 65.16 & 65.28 & 57.05 & 72.63 & 64.72 & 62.22 \\
& GPTQ & 1.67e-01 & 9.68 & 9.99e-02 & 4.91 & 4.20e-02 & \textbf{3.40} & 58.11 & 81.57 & 77.34 & 75.49 & 70.15 & 79.22 & 71.74 & \textbf{73.37} \\
& GuidedQ & 1.71e-01 & 9.80 & 9.89e-02 & 5.01 & 4.19e-02 & 3.47 & 57.00 & 81.48 & 76.52 & 75.26 & 68.35 & 78.56 & 70.88 & 72.58 \\
& GPTAQ & 1.67e-01 & 9.97 & 9.93e-02 & 4.99 & 3.78e-02 & 3.45 & 56.83 & 79.67 & 77.56 & 75.16 & 69.47 & 77.69 & 70.80 & 72.45 \\
& \cellcolor{g2ptqblue!12}\name & \cellcolor{g2ptqblue!12}1.56e-01 & \cellcolor{g2ptqblue!12}9.78 & \cellcolor{g2ptqblue!12}9.05e-02 & \cellcolor{g2ptqblue!12}4.96 & \cellcolor{g2ptqblue!12}3.89e-02 & \cellcolor{g2ptqblue!12}3.43 & \cellcolor{g2ptqblue!12}58.70 & \cellcolor{g2ptqblue!12}81.48 & \cellcolor{g2ptqblue!12}78.08 & \cellcolor{g2ptqblue!12}75.69 & \cellcolor{g2ptqblue!12}67.77 & \cellcolor{g2ptqblue!12}78.56 & \cellcolor{g2ptqblue!12}70.80 & \cellcolor{g2ptqblue!12}73.01 \\
& \cellcolor{g2ptqblue!12}\name\textsuperscript{*} & \cellcolor{g2ptqblue!12}\textbf{1.54e-01} & \cellcolor{g2ptqblue!12}\textbf{9.72} & \cellcolor{g2ptqblue!12}\textbf{8.30e-02} & \cellcolor{g2ptqblue!12}\textbf{4.92} & \cellcolor{g2ptqblue!12}\textbf{3.34e-02} & \cellcolor{g2ptqblue!12}3.42 & \cellcolor{g2ptqblue!12}57.94 & \cellcolor{g2ptqblue!12}81.78 & \cellcolor{g2ptqblue!12}77.79 & \cellcolor{g2ptqblue!12}76.06 & \cellcolor{g2ptqblue!12}66.76 & \cellcolor{g2ptqblue!12}78.67 & \cellcolor{g2ptqblue!12}72.22 & \cellcolor{g2ptqblue!12}73.03 \\
\midrule

\multirow{7}{*}{\begin{tabular}{@{}c@{}}\textbf{Qwen3}\\\textbf{32B}\end{tabular}}
& BF16 & 0 & 7.61 & 0 & 4.05 & 0 & 2.80 & 61.01 & 83.21 & 86.03 & 82.60 & 67.13 & 82.05 & 73.24 & 76.47 \\
\cmidrule{2-16}
& RTN & 8.53e-01 & 13.74 & 5.26e-01 & 5.17 & 4.84e-01 & 3.67 & 49.32 & 67.85 & 67.16 & 70.16 & 52.42 & 75.41 & 61.01 & 63.33 \\
& GPTQ & 2.22e-01 & 8.32 & 8.79e-02 & 4.14 & 6.78e-02 & 2.87 & 59.22 & 81.61 & 81.43 & 79.65 & 71.57 & 80.03 & 73.64 & \textbf{75.31} \\
& GuidedQ & 2.32e-01 & 8.44 & 8.75e-02 & 4.18 & 6.67e-02 & 2.89 & 61.01 & 82.79 & 80.91 & 79.85 & 69.28 & 79.71 & 70.80 & 74.91 \\
& GPTAQ & 3.42e-01 & 9.29 & 1.77e-01 & 4.36 & 1.33e-01 & 3.02 & 57.51 & 78.28 & 78.75 & 78.64 & 71.67 & 80.36 & 72.06 & 73.90 \\
& \cellcolor{g2ptqblue!12}\name & \cellcolor{g2ptqblue!12}2.08e-01 & \cellcolor{g2ptqblue!12}8.31 & \cellcolor{g2ptqblue!12}8.36e-02 & \cellcolor{g2ptqblue!12}4.16 & \cellcolor{g2ptqblue!12}6.57e-02 & \cellcolor{g2ptqblue!12}2.84 & \cellcolor{g2ptqblue!12}60.24 & \cellcolor{g2ptqblue!12}82.58 & \cellcolor{g2ptqblue!12}81.80 & \cellcolor{g2ptqblue!12}80.19 & \cellcolor{g2ptqblue!12}68.12 & \cellcolor{g2ptqblue!12}80.30 & \cellcolor{g2ptqblue!12}73.16 & \cellcolor{g2ptqblue!12}75.20 \\
& \cellcolor{g2ptqblue!12}\name\textsuperscript{*} & \cellcolor{g2ptqblue!12}\textbf{2.06e-01} & \cellcolor{g2ptqblue!12}\textbf{8.27} & \cellcolor{g2ptqblue!12}\textbf{7.76e-02} & \cellcolor{g2ptqblue!12}\textbf{4.09} & \cellcolor{g2ptqblue!12}\textbf{5.88e-02} & \cellcolor{g2ptqblue!12}\textbf{2.82} & \cellcolor{g2ptqblue!12}59.90 & \cellcolor{g2ptqblue!12}81.27 & \cellcolor{g2ptqblue!12}80.68 & \cellcolor{g2ptqblue!12}80.21 & \cellcolor{g2ptqblue!12}67.75 & \cellcolor{g2ptqblue!12}80.63 & \cellcolor{g2ptqblue!12}73.95 & \cellcolor{g2ptqblue!12}74.91 \\
\midrule

\multirow{7}{*}{\begin{tabular}{@{}c@{}}\textbf{Qwen3.5}\\\textbf{0.8B}\end{tabular}}
& BF16 & 0 & 17.18 & 0 & 5.54 & 0 & 3.40 & 37.63 & 61.49 & 48.89 & 49.54 & 43.59 & 69.31 & 57.62 & 52.58 \\
\cmidrule{2-16}
& RTN & 2.36e+00 & 199.06 & 1.72e+00 & 29.21 & 1.58e+00 & 14.40 & 22.61 & 37.08 & 23.40 & 32.02 & 7.30 & 56.75 & 49.80 & 32.71 \\
& GPTQ & 4.04e-01 & 25.65 & 3.37e-01 & 7.57 & 1.65e-01 & 3.79 & 31.06 & 52.10 & 28.16 & 42.37 & 27.32 & 64.42 & 54.22 & 42.81 \\
& GuidedQ & 4.23e-01 & 26.85 & 3.41e-01 & 7.63 & 1.73e-01 & 3.86 & 31.57 & 51.09 & 30.24 & 41.94 & 28.80 & 64.80 & 55.56 & 43.43 \\
& GPTAQ & 3.81e-01 & 25.92 & 3.08e-01 & 7.48 & 1.45e-01 & 3.83 & 30.29 & 50.84 & 31.58 & 41.93 & 27.77 & 63.55 & 54.78 & 42.96 \\
& \cellcolor{g2ptqblue!12}\name & \cellcolor{g2ptqblue!12}3.57e-01 & \cellcolor{g2ptqblue!12}24.31 & \cellcolor{g2ptqblue!12}2.82e-01 & \cellcolor{g2ptqblue!12}7.16 & \cellcolor{g2ptqblue!12}1.35e-01 & \cellcolor{g2ptqblue!12}3.78 & \cellcolor{g2ptqblue!12}29.61 & \cellcolor{g2ptqblue!12}50.84 & \cellcolor{g2ptqblue!12}28.45 & \cellcolor{g2ptqblue!12}43.07 & \cellcolor{g2ptqblue!12}29.05 & \cellcolor{g2ptqblue!12}64.31 & \cellcolor{g2ptqblue!12}54.14 & \cellcolor{g2ptqblue!12}42.78 \\
& \cellcolor{g2ptqblue!12}\name\textsuperscript{*} & \cellcolor{g2ptqblue!12}\textbf{3.16e-01} & \cellcolor{g2ptqblue!12}\textbf{23.36} & \cellcolor{g2ptqblue!12}\textbf{2.38e-01} & \cellcolor{g2ptqblue!12}\textbf{6.90} & \cellcolor{g2ptqblue!12}\textbf{1.09e-01} & \cellcolor{g2ptqblue!12}\textbf{3.69} & \cellcolor{g2ptqblue!12}31.14 & \cellcolor{g2ptqblue!12}55.77 & \cellcolor{g2ptqblue!12}27.12 & \cellcolor{g2ptqblue!12}43.10 & \cellcolor{g2ptqblue!12}32.19 & \cellcolor{g2ptqblue!12}65.78 & \cellcolor{g2ptqblue!12}56.20 & \cellcolor{g2ptqblue!12}\textbf{44.47} \\
\midrule

\multirow{7}{*}{\begin{tabular}{@{}c@{}}\textbf{Qwen3.5}\\\textbf{2B}\end{tabular}}
& BF16 & 0 & 12.10 & 0 & 4.66 & 0 & 2.99 & 41.47 & 65.95 & 59.06 & 62.11 & 53.31 & 72.69 & 62.67 & 59.61 \\
\cmidrule{2-16}
& RTN & 2.73e+00 & 199.38 & 2.31e+00 & 40.50 & 1.69e+00 & 14.40 & 28.07 & 48.36 & 25.71 & 34.99 & 13.72 & 62.13 & 53.67 & 38.09 \\
& GPTQ & 3.25e-01 & 17.20 & 2.44e-01 & 5.65 & 1.14e-01 & 3.24 & 34.90 & 51.60 & 43.76 & 54.23 & 43.06 & 70.24 & 61.25 & 51.29 \\
& GuidedQ & 3.34e-01 & 17.16 & 2.51e-01 & 5.70 & 1.29e-01 & 3.31 & 34.73 & 52.57 & 40.79 & 54.04 & 47.91 & 69.80 & 59.75 & 51.37 \\
& GPTAQ & 3.15e-01 & 16.88 & 2.23e-01 & 5.60 & 1.07e-01 & 3.27 & 37.97 & 54.63 & 34.10 & 53.96 & 40.46 & 69.21 & 59.75 & 50.01 \\
& \cellcolor{g2ptqblue!12}\name & \cellcolor{g2ptqblue!12}2.93e-01 & \cellcolor{g2ptqblue!12}16.49 & \cellcolor{g2ptqblue!12}2.02e-01 & \cellcolor{g2ptqblue!12}5.51 & \cellcolor{g2ptqblue!12}1.02e-01 & \cellcolor{g2ptqblue!12}3.26 & \cellcolor{g2ptqblue!12}37.29 & \cellcolor{g2ptqblue!12}57.49 & \cellcolor{g2ptqblue!12}40.94 & \cellcolor{g2ptqblue!12}54.98 & \cellcolor{g2ptqblue!12}42.67 & \cellcolor{g2ptqblue!12}69.48 & \cellcolor{g2ptqblue!12}60.22 & \cellcolor{g2ptqblue!12}51.87 \\
& \cellcolor{g2ptqblue!12}\name\textsuperscript{*} & \cellcolor{g2ptqblue!12}\textbf{2.68e-01} & \cellcolor{g2ptqblue!12}\textbf{16.03} & \cellcolor{g2ptqblue!12}\textbf{1.75e-01} & \cellcolor{g2ptqblue!12}\textbf{5.36} & \cellcolor{g2ptqblue!12}\textbf{8.68e-02} & \cellcolor{g2ptqblue!12}\textbf{3.20} & \cellcolor{g2ptqblue!12}38.57 & \cellcolor{g2ptqblue!12}61.45 & \cellcolor{g2ptqblue!12}39.67 & \cellcolor{g2ptqblue!12}55.44 & \cellcolor{g2ptqblue!12}47.14 & \cellcolor{g2ptqblue!12}71.33 & \cellcolor{g2ptqblue!12}61.09 & \cellcolor{g2ptqblue!12}\textbf{53.53} \\
\midrule

\multirow{7}{*}{\begin{tabular}{@{}c@{}}\textbf{Qwen3.5}\\\textbf{4B}\end{tabular}}
& BF16 & 0 & 9.58 & 0 & 4.33 & 0 & 2.72 & 54.27 & 75.46 & 72.51 & 73.11 & 64.74 & 77.97 & 69.93 & 69.71 \\
\cmidrule{2-16}
& RTN & 7.06e-01 & 19.09 & 5.59e-01 & 5.87 & 4.12e-01 & 3.75 & 41.55 & 63.26 & 53.86 & 61.95 & 47.27 & 72.52 & 61.80 & 57.46 \\
& GPTQ & 2.94e-01 & 12.73 & 2.01e-01 & \textbf{4.67} & 9.25e-02 & 2.88 & 47.70 & 68.90 & 63.15 & 67.53 & 61.89 & 75.24 & 64.88 & 64.18 \\
& GuidedQ & 3.10e-01 & 12.86 & 2.19e-01 & 4.83 & 1.09e-01 & 2.93 & 49.66 & 73.06 & 62.04 & 67.16 & 61.23 & 75.52 & 65.82 & \textbf{64.93} \\
& GPTAQ & 3.10e-01 & 13.27 & 1.88e-01 & 4.81 & 8.87e-02 & 2.89 & 48.04 & 70.92 & 60.33 & 67.03 & 59.46 & 75.57 & 65.90 & 63.89 \\
& \cellcolor{g2ptqblue!12}\name & \cellcolor{g2ptqblue!12}2.61e-01 & \cellcolor{g2ptqblue!12}12.04 & \cellcolor{g2ptqblue!12}1.74e-01 & \cellcolor{g2ptqblue!12}4.72 & \cellcolor{g2ptqblue!12}8.43e-02 & \cellcolor{g2ptqblue!12}2.87 & \cellcolor{g2ptqblue!12}49.06 & \cellcolor{g2ptqblue!12}72.64 & \cellcolor{g2ptqblue!12}63.82 & \cellcolor{g2ptqblue!12}67.55 & \cellcolor{g2ptqblue!12}59.40 & \cellcolor{g2ptqblue!12}76.06 & \cellcolor{g2ptqblue!12}65.82 & \cellcolor{g2ptqblue!12}64.91 \\
& \cellcolor{g2ptqblue!12}\name\textsuperscript{*} & \cellcolor{g2ptqblue!12}\textbf{2.38e-01} & \cellcolor{g2ptqblue!12}\textbf{11.56} & \cellcolor{g2ptqblue!12}\textbf{1.56e-01} & \cellcolor{g2ptqblue!12}4.74 & \cellcolor{g2ptqblue!12}\textbf{7.45e-02} & \cellcolor{g2ptqblue!12}\textbf{2.85} & \cellcolor{g2ptqblue!12}48.81 & \cellcolor{g2ptqblue!12}68.43 & \cellcolor{g2ptqblue!12}63.97 & \cellcolor{g2ptqblue!12}67.56 & \cellcolor{g2ptqblue!12}60.08 & \cellcolor{g2ptqblue!12}75.57 & \cellcolor{g2ptqblue!12}68.11 & \cellcolor{g2ptqblue!12}64.65 \\
\midrule

\multirow{7}{*}{\begin{tabular}{@{}c@{}}\textbf{Qwen3.5}\\\textbf{9B}\end{tabular}}
& BF16 & 0 & 8.65 & 0 & 3.78 & 0 & 2.48 & 55.46 & 74.37 & 78.75 & 78.08 & 70.04 & 79.98 & 72.93 & 72.80 \\
\cmidrule{2-16}
& RTN & 1.66e+00 & 37.80 & 1.56e+00 & 12.72 & 1.38e+00 & 8.96 & 32.17 & 59.93 & 41.53 & 41.89 & 19.35 & 66.21 & 55.64 & 45.25 \\
& GPTQ & 1.96e-01 & 9.36 & 1.05e-01 & 4.01 & 5.03e-02 & 2.57 & 52.99 & 76.52 & 74.44 & 74.24 & 68.45 & 79.33 & 71.03 & \textbf{71.00} \\
& GuidedQ & 2.06e-01 & 9.76 & 1.10e-01 & 4.02 & 6.14e-02 & 2.60 & 51.96 & 75.76 & 73.33 & 74.36 & 68.19 & 78.73 & 69.93 & 70.32 \\
& GPTAQ & 1.99e-01 & 9.47 & 1.03e-01 & 4.05 & 4.93e-02 & 2.57 & 53.84 & 77.10 & 73.40 & 74.00 & 68.31 & 79.16 & 70.48 & 70.90 \\
& \cellcolor{g2ptqblue!12}\name & \cellcolor{g2ptqblue!12}1.94e-01 & \cellcolor{g2ptqblue!12}\textbf{9.18} & \cellcolor{g2ptqblue!12}9.40e-02 & \cellcolor{g2ptqblue!12}4.03 & \cellcolor{g2ptqblue!12}4.80e-02 & \cellcolor{g2ptqblue!12}2.57 & \cellcolor{g2ptqblue!12}52.39 & \cellcolor{g2ptqblue!12}74.62 & \cellcolor{g2ptqblue!12}74.22 & \cellcolor{g2ptqblue!12}74.56 & \cellcolor{g2ptqblue!12}67.42 & \cellcolor{g2ptqblue!12}79.22 & \cellcolor{g2ptqblue!12}71.43 & \cellcolor{g2ptqblue!12}70.55 \\
& \cellcolor{g2ptqblue!12}\name\textsuperscript{*} & \cellcolor{g2ptqblue!12}\textbf{1.89e-01} & \cellcolor{g2ptqblue!12}9.22 & \cellcolor{g2ptqblue!12}\textbf{8.75e-02} & \cellcolor{g2ptqblue!12}\textbf{3.97} & \cellcolor{g2ptqblue!12}\textbf{4.26e-02} & \cellcolor{g2ptqblue!12}\textbf{2.55} & \cellcolor{g2ptqblue!12}51.45 & \cellcolor{g2ptqblue!12}75.51 & \cellcolor{g2ptqblue!12}74.07 & \cellcolor{g2ptqblue!12}74.45 & \cellcolor{g2ptqblue!12}67.55 & \cellcolor{g2ptqblue!12}79.38 & \cellcolor{g2ptqblue!12}71.43 & \cellcolor{g2ptqblue!12}70.55 \\
\midrule

\multirow{7}{*}{\begin{tabular}{@{}c@{}}\textbf{Qwen3.5}\\\textbf{27B}\end{tabular}}
& BF16 & 0 & 6.90 & 0 & 3.55 & 0 & 2.28 & 61.26 & 79.67 & 85.44 & 83.34 & 75.16 & 82.05 & 78.69 & 77.94 \\
\cmidrule{2-16}
& RTN & 6.13e-01 & 10.76 & 5.09e-01 & 4.63 & 4.38e-01 & 3.26 & 55.03 & 77.15 & 70.51 & 71.36 & 54.53 & 77.26 & 70.17 & 68.00 \\
& GPTQ & 1.38e-01 & 7.58 & 6.10e-02 & 3.65 & 3.39e-02 & \textbf{2.33} & 60.75 & 80.30 & 82.54 & 81.27 & 76.25 & 82.10 & 78.22 & 77.35 \\
& GuidedQ & 1.48e-01 & 7.59 & 6.55e-02 & 3.66 & 3.64e-02 & 2.34 & 61.01 & 81.06 & 83.28 & 80.94 & 75.72 & 81.07 & 77.43 & 77.22 \\
& GPTAQ & 1.43e-01 & 7.60 & 6.44e-02 & 3.66 & 3.50e-02 & 2.34 & 61.86 & 81.78 & 82.62 & 80.47 & 75.88 & 81.61 & 78.37 & 77.51 \\
& \cellcolor{g2ptqblue!12}\name & \cellcolor{g2ptqblue!12}\textbf{1.36e-01} & \cellcolor{g2ptqblue!12}\textbf{7.42} & \cellcolor{g2ptqblue!12}5.95e-02 & \cellcolor{g2ptqblue!12}3.65 & \cellcolor{g2ptqblue!12}3.62e-02 & \cellcolor{g2ptqblue!12}\textbf{2.33} & \cellcolor{g2ptqblue!12}62.20 & \cellcolor{g2ptqblue!12}83.38 & \cellcolor{g2ptqblue!12}82.84 & \cellcolor{g2ptqblue!12}81.27 & \cellcolor{g2ptqblue!12}75.08 & \cellcolor{g2ptqblue!12}82.21 & \cellcolor{g2ptqblue!12}77.19 & \cellcolor{g2ptqblue!12}\textbf{77.74} \\
& \cellcolor{g2ptqblue!12}\name\textsuperscript{*} & \cellcolor{g2ptqblue!12}\textbf{1.36e-01} & \cellcolor{g2ptqblue!12}7.52 & \cellcolor{g2ptqblue!12}\textbf{5.58e-02} & \cellcolor{g2ptqblue!12}\textbf{3.64} & \cellcolor{g2ptqblue!12}\textbf{3.13e-02} & \cellcolor{g2ptqblue!12}\textbf{2.33} & \cellcolor{g2ptqblue!12}61.35 & \cellcolor{g2ptqblue!12}80.01 & \cellcolor{g2ptqblue!12}82.10 & \cellcolor{g2ptqblue!12}81.06 & \cellcolor{g2ptqblue!12}74.36 & \cellcolor{g2ptqblue!12}81.61 & \cellcolor{g2ptqblue!12}77.74 & \cellcolor{g2ptqblue!12}76.89 \\
\midrule

\multirow{7}{*}{\begin{tabular}{@{}c@{}}\textbf{LLaMA3}\\\textbf{8B}\end{tabular}}
& BF16 & 0 & 6.14 & 0 & 3.41 & 0 & 2.72 & 53.24 & 77.90 & 48.44 & 79.25 & 75.53 & 80.74 & 73.16 & 69.75 \\
\cmidrule{2-16}
& RTN & 2.14e+00 & 53.31 & 1.26e+00 & 11.61 & 8.60e-01 & 6.40 & 27.99 & 43.35 & 22.66 & 47.90 & 28.35 & 60.94 & 59.98 & 41.60 \\
& GPTQ & 2.82e-01 & 8.17 & 1.45e-01 & 3.88 & 7.86e-02 & 2.92 & 47.27 & 73.02 & 36.55 & 73.78 & 72.71 & 77.20 & 71.90 & 64.63 \\
& GuidedQ & 2.38e-01 & 7.81 & 1.18e-01 & 3.79 & 6.42e-02 & 2.89 & 49.15 & 76.39 & 39.08 & 75.16 & 73.12 & 77.31 & 71.74 & 65.99 \\
& GPTAQ & 2.53e-01 & 7.92 & 1.13e-01 & 3.79 & 6.70e-02 & 2.89 & 48.04 & 75.13 & 42.12 & 74.22 & 71.30 & 77.42 & 71.51 & 65.68 \\
& \cellcolor{g2ptqblue!12}\name & \cellcolor{g2ptqblue!12}\textbf{2.27e-01} & \cellcolor{g2ptqblue!12}\textbf{7.71} & \cellcolor{g2ptqblue!12}9.73e-02 & \cellcolor{g2ptqblue!12}3.74 & \cellcolor{g2ptqblue!12}6.16e-02 & \cellcolor{g2ptqblue!12}2.88 & \cellcolor{g2ptqblue!12}48.21 & \cellcolor{g2ptqblue!12}74.83 & \cellcolor{g2ptqblue!12}40.86 & \cellcolor{g2ptqblue!12}75.10 & \cellcolor{g2ptqblue!12}72.87 & \cellcolor{g2ptqblue!12}78.56 & \cellcolor{g2ptqblue!12}72.06 & \cellcolor{g2ptqblue!12}66.07 \\
& \cellcolor{g2ptqblue!12}\name\textsuperscript{*} & \cellcolor{g2ptqblue!12}2.28e-01 & \cellcolor{g2ptqblue!12}\textbf{7.71} & \cellcolor{g2ptqblue!12}\textbf{8.97e-02} & \cellcolor{g2ptqblue!12}\textbf{3.71} & \cellcolor{g2ptqblue!12}\textbf{4.95e-02} & \cellcolor{g2ptqblue!12}\textbf{2.84} & \cellcolor{g2ptqblue!12}49.74 & \cellcolor{g2ptqblue!12}75.59 & \cellcolor{g2ptqblue!12}42.42 & \cellcolor{g2ptqblue!12}75.33 & \cellcolor{g2ptqblue!12}72.00 & \cellcolor{g2ptqblue!12}78.02 & \cellcolor{g2ptqblue!12}72.85 & \cellcolor{g2ptqblue!12}\textbf{66.56} \\
\midrule

\multirow{7}{*}{\begin{tabular}{@{}c@{}}\textbf{LLaMA3}\\\textbf{70B}\end{tabular}}
& BF16 & 0 & 2.86 & 0 & 3.05 & 0 & 2.41 & 64.16 & 85.86 & 65.08 & 84.95 & 79.45 & 84.49 & 80.43 & 77.77 \\
\cmidrule{2-16}
& RTN & 4.56e+00 & 285.78 & 3.77e+00 & 127.47 & 2.93e+00 & 44.40 & 21.67 & 30.89 & 23.40 & 28.85 & 12.03 & 55.55 & 49.25 & 31.66 \\
& GPTQ & 6.51e-01 & 5.48 & 9.83e-02 & 3.31 & 7.75e-02 & 2.57 & 57.51 & 82.83 & 52.67 & 81.73 & 78.79 & 82.21 & 78.93 & 73.52 \\
& GuidedQ & 5.05e-01 & 4.74 & 4.61e-02 & 3.18 & 3.48e-02 & 2.47 & 62.03 & 85.06 & 60.18 & 83.43 & 79.84 & 83.41 & 81.53 & \textbf{76.50} \\
& GPTAQ & 5.78e-01 & 5.09 & 5.79e-02 & 3.21 & 4.76e-02 & 2.49 & 61.43 & 83.12 & 58.10 & 81.64 & 78.42 & 82.81 & 80.43 & 75.14 \\
& \cellcolor{g2ptqblue!12}\name & \cellcolor{g2ptqblue!12}\textbf{4.93e-01} & \cellcolor{g2ptqblue!12}\textbf{4.69} & \cellcolor{g2ptqblue!12}\textbf{4.06e-02} & \cellcolor{g2ptqblue!12}\textbf{3.16} & \cellcolor{g2ptqblue!12}3.22e-02 & \cellcolor{g2ptqblue!12}2.47 & \cellcolor{g2ptqblue!12}61.77 & \cellcolor{g2ptqblue!12}84.22 & \cellcolor{g2ptqblue!12}58.84 & \cellcolor{g2ptqblue!12}83.62 & \cellcolor{g2ptqblue!12}79.18 & \cellcolor{g2ptqblue!12}83.03 & \cellcolor{g2ptqblue!12}79.24 & \cellcolor{g2ptqblue!12}75.70 \\
& \cellcolor{g2ptqblue!12}\name\textsuperscript{*} & \cellcolor{g2ptqblue!12}5.35e-01 & \cellcolor{g2ptqblue!12}4.89 & \cellcolor{g2ptqblue!12}4.21e-02 & \cellcolor{g2ptqblue!12}3.17 & \cellcolor{g2ptqblue!12}\textbf{2.87e-02} & \cellcolor{g2ptqblue!12}\textbf{2.46} & \cellcolor{g2ptqblue!12}62.29 & \cellcolor{g2ptqblue!12}84.89 & \cellcolor{g2ptqblue!12}60.25 & \cellcolor{g2ptqblue!12}83.45 & \cellcolor{g2ptqblue!12}78.32 & \cellcolor{g2ptqblue!12}83.13 & \cellcolor{g2ptqblue!12}79.32 & \cellcolor{g2ptqblue!12}75.95 \\

\bottomrule
\end{tabular}
\end{adjustbox}
\vspace{-1ex}
\caption{3-bit weight-only quantization results across different model families.}
\label{tab:w3}
\end{table*}

%% file: tables/appendix/w4.tex
\begin{table*}[!t]
\centering
\begin{adjustbox}{
  max width=\linewidth,
  max totalheight=0.95\textheight,
  keepaspectratio
}
\begin{tabular}{c|l|cc|cc|cc|cccccccc}
\toprule
\multirow{2}{*}{\textbf{Model}} & \multirow{2}{*}{\textbf{Method}} & \multicolumn{2}{c|}{\textbf{WikiText2}} & \multicolumn{2}{c|}{\textbf{UltraChat}} & \multicolumn{2}{c|}{\textbf{NuminaMath}} & \multicolumn{8}{c}{\textbf{CommonSense QA}} \\
\cmidrule(lr){3-4} \cmidrule(lr){5-6} \cmidrule(lr){7-8} \cmidrule(lr){9-16}
& & \textbf{KL} & \textbf{PPL} & \textbf{KL} & \textbf{PPL} & \textbf{KL} & \textbf{PPL} & \textbf{ARC-C} & \textbf{ARC-E} & \textbf{C-eval} & \textbf{HellaS} & \textbf{LAMB} & \textbf{PIQA} & \textbf{Wino} & \textbf{Avg.} \\
\midrule

\multirow{7}{*}{\begin{tabular}{@{}c@{}}\textbf{Qwen3} \\ \textbf{0.6B}\end{tabular}}
& BF16 & 0 & 20.96 & 0 & 8.00 & 0 & 4.88 & 34.13 & 56.06 & 43.76 & 47.29 & 40.02 & 67.36 & 56.35 & 49.28 \\
\cmidrule{2-16}
& RTN & 7.09e-01 & 34.72 & 6.01e-01 & 12.01 & 4.43e-01 & 6.11 & 28.75 & 44.19 & 33.21 & 42.25 & 24.10 & 63.00 & 53.99 & 41.36 \\
& GPTQ & 5.97e-01 & 37.97 & 4.61e-01 & 11.25 & 3.16e-01 & 5.76 & 27.65 & 42.80 & 34.99 & 41.22 & 26.43 & 63.28 & 54.62 & 41.57 \\
& GuidedQ & 2.09e-01 & 24.71 & 1.38e-01 & 8.80 & 7.86e-02 & 5.12 & 31.06 & 49.45 & 39.97 & 44.63 & 32.18 & 65.83 & 54.78 & 45.41 \\
& GPTAQ & 4.17e-01 & 31.84 & 3.03e-01 & 10.15 & 1.88e-01 & 5.57 & 28.58 & 43.86 & 34.62 & 42.71 & 31.44 & 64.58 & 54.14 & 42.85 \\
& \cellcolor{g2ptqblue!12}\name & \cellcolor{g2ptqblue!12}1.73e-01 & \cellcolor{g2ptqblue!12}24.02 & \cellcolor{g2ptqblue!12}1.16e-01 & \cellcolor{g2ptqblue!12}8.77 & \cellcolor{g2ptqblue!12}6.14e-02 & \cellcolor{g2ptqblue!12}5.09 & \cellcolor{g2ptqblue!12}32.76 & \cellcolor{g2ptqblue!12}51.73 & \cellcolor{g2ptqblue!12}37.96 & \cellcolor{g2ptqblue!12}44.80 & \cellcolor{g2ptqblue!12}32.93 & \cellcolor{g2ptqblue!12}66.10 & \cellcolor{g2ptqblue!12}54.70 & \cellcolor{g2ptqblue!12}\textbf{45.85} \\
& \cellcolor{g2ptqblue!12}\name\textsuperscript{*} & \cellcolor{g2ptqblue!12}\textbf{1.60e-01} & \cellcolor{g2ptqblue!12}\textbf{23.98} & \cellcolor{g2ptqblue!12}\textbf{1.01e-01} & \cellcolor{g2ptqblue!12}\textbf{8.56} & \cellcolor{g2ptqblue!12}\textbf{5.35e-02} & \cellcolor{g2ptqblue!12}\textbf{5.03} & \cellcolor{g2ptqblue!12}30.12 & \cellcolor{g2ptqblue!12}49.37 & \cellcolor{g2ptqblue!12}39.82 & \cellcolor{g2ptqblue!12}44.82 & \cellcolor{g2ptqblue!12}34.45 & \cellcolor{g2ptqblue!12}65.72 & \cellcolor{g2ptqblue!12}56.27 & \cellcolor{g2ptqblue!12}45.80 \\
\midrule

\multirow{7}{*}{\begin{tabular}{@{}c@{}}\textbf{Qwen3} \\ \textbf{1.7B}\end{tabular}}
& BF16 & 0 & 16.72 & 0 & 9.90 & 0 & 5.72 & 44.03 & 70.16 & 58.69 & 60.33 & 50.40 & 72.80 & 61.33 & 59.68 \\
\cmidrule{2-16}
& RTN & 1.39e+00 & 77.06 & 9.21e-01 & 12.33 & 9.08e-01 & 8.94 & 29.35 & 42.00 & 40.34 & 50.35 & 23.73 & 63.38 & 52.41 & 43.08 \\
& GPTQ & 3.82e-01 & 24.92 & 2.87e-01 & 10.69 & 1.85e-01 & 6.17 & 35.58 & 55.43 & 52.38 & 55.47 & 35.86 & 67.46 & 56.83 & 51.29 \\
& GuidedQ & 2.99e-01 & 19.59 & 2.26e-01 & 10.90 & 1.30e-01 & 5.75 & 40.96 & 70.16 & 49.78 & 57.40 & 40.60 & 70.89 & 59.98 & \textbf{55.68} \\
& GPTAQ & 2.58e-01 & 21.45 & 2.01e-01 & 11.31 & 1.08e-01 & 6.32 & 38.23 & 63.89 & 53.71 & 56.01 & 39.78 & 70.02 & 58.64 & 54.33 \\
& \cellcolor{g2ptqblue!12}\name & \cellcolor{g2ptqblue!12}2.26e-01 & \cellcolor{g2ptqblue!12}18.94 & \cellcolor{g2ptqblue!12}1.78e-01 & \cellcolor{g2ptqblue!12}11.28 & \cellcolor{g2ptqblue!12}9.12e-02 & \cellcolor{g2ptqblue!12}5.81 & \cellcolor{g2ptqblue!12}38.65 & \cellcolor{g2ptqblue!12}66.08 & \cellcolor{g2ptqblue!12}43.76 & \cellcolor{g2ptqblue!12}57.58 & \cellcolor{g2ptqblue!12}42.85 & \cellcolor{g2ptqblue!12}70.35 & \cellcolor{g2ptqblue!12}59.43 & \cellcolor{g2ptqblue!12}54.10 \\
& \cellcolor{g2ptqblue!12}\name\textsuperscript{*} & \cellcolor{g2ptqblue!12}\textbf{2.11e-01} & \cellcolor{g2ptqblue!12}\textbf{18.11} & \cellcolor{g2ptqblue!12}\textbf{1.62e-01} & \cellcolor{g2ptqblue!12}\textbf{10.43} & \cellcolor{g2ptqblue!12}\textbf{8.56e-02} & \cellcolor{g2ptqblue!12}\textbf{5.65} & \cellcolor{g2ptqblue!12}40.44 & \cellcolor{g2ptqblue!12}65.70 & \cellcolor{g2ptqblue!12}46.66 & \cellcolor{g2ptqblue!12}57.77 & \cellcolor{g2ptqblue!12}44.83 & \cellcolor{g2ptqblue!12}70.18 & \cellcolor{g2ptqblue!12}60.06 & \cellcolor{g2ptqblue!12}55.09 \\
\midrule

\multirow{7}{*}{\begin{tabular}{@{}c@{}}\textbf{Qwen3} \\ \textbf{4B}\end{tabular}}
& BF16 & 0 & 13.66 & 0 & 8.03 & 0 & 5.14 & 54.18 & 78.07 & 70.43 & 68.50 & 59.50 & 74.92 & 65.67 & 67.32 \\
\cmidrule{2-16}
& RTN & 3.60e-01 & 16.48 & 3.14e-01 & \textbf{6.43} & 2.38e-01 & \textbf{4.67} & 45.65 & 65.70 & 62.63 & 64.51 & 49.99 & 73.29 & 61.64 & 60.49 \\
& GPTQ & 2.34e-01 & \textbf{15.10} & 1.75e-01 & 6.95 & 1.20e-01 & 4.80 & 48.38 & 73.15 & 62.78 & 66.05 & 57.87 & 73.67 & 64.96 & 63.84 \\
& GuidedQ & \textbf{1.92e-01} & 16.70 & 1.25e-01 & 7.67 & 7.53e-02 & 4.95 & 48.21 & 72.14 & 65.08 & 65.93 & 57.71 & 74.21 & 64.17 & 63.92 \\
& GPTAQ & 2.18e-01 & 17.00 & 1.56e-01 & 8.51 & 9.47e-02 & 5.55 & 50.00 & 75.25 & 63.89 & 64.92 & 58.41 & 73.94 & 62.83 & 64.18 \\
& \cellcolor{g2ptqblue!12}\name & \cellcolor{g2ptqblue!12}1.94e-01 & \cellcolor{g2ptqblue!12}16.08 & \cellcolor{g2ptqblue!12}1.14e-01 & \cellcolor{g2ptqblue!12}7.78 & \cellcolor{g2ptqblue!12}5.77e-02 & \cellcolor{g2ptqblue!12}5.12 & \cellcolor{g2ptqblue!12}50.00 & \cellcolor{g2ptqblue!12}73.86 & \cellcolor{g2ptqblue!12}62.70 & \cellcolor{g2ptqblue!12}65.65 & \cellcolor{g2ptqblue!12}58.49 & \cellcolor{g2ptqblue!12}74.43 & \cellcolor{g2ptqblue!12}64.48 & \cellcolor{g2ptqblue!12}64.23 \\
& \cellcolor{g2ptqblue!12}\name\textsuperscript{*} & \cellcolor{g2ptqblue!12}1.95e-01 & \cellcolor{g2ptqblue!12}16.13 & \cellcolor{g2ptqblue!12}\textbf{1.05e-01} & \cellcolor{g2ptqblue!12}7.55 & \cellcolor{g2ptqblue!12}\textbf{4.86e-02} & \cellcolor{g2ptqblue!12}4.87 & \cellcolor{g2ptqblue!12}50.77 & \cellcolor{g2ptqblue!12}75.84 & \cellcolor{g2ptqblue!12}66.20 & \cellcolor{g2ptqblue!12}65.79 & \cellcolor{g2ptqblue!12}58.20 & \cellcolor{g2ptqblue!12}74.70 & \cellcolor{g2ptqblue!12}64.64 & \cellcolor{g2ptqblue!12}\textbf{65.16} \\
\midrule

\multirow{7}{*}{\begin{tabular}{@{}c@{}}\textbf{Qwen3} \\ \textbf{8B}\end{tabular}}
& BF16 & 0 & 9.72 & 0 & 5.20 & 0 & 3.60 & 56.14 & 80.77 & 79.57 & 74.98 & 64.10 & 77.97 & 67.96 & 71.64 \\
\cmidrule{2-16}
& RTN & 2.16e-01 & 10.66 & 1.86e-01 & \textbf{5.19} & 1.22e-01 & \textbf{3.54} & 50.68 & 73.95 & 72.96 & 71.60 & 55.50 & 77.31 & 68.43 & 67.20 \\
& GPTQ & 6.60e-02 & \textbf{9.92} & 4.02e-02 & 5.27 & 1.98e-02 & 3.64 & 54.35 & 79.21 & 76.60 & 73.69 & 62.95 & 77.26 & 69.06 & 70.45 \\
& GuidedQ & 6.43e-02 & 10.07 & 3.66e-02 & 5.30 & \textbf{1.87e-02} & 3.64 & 55.97 & 79.59 & 76.30 & 73.77 & 63.98 & 77.37 & 67.88 & 70.69 \\
& GPTAQ & 6.67e-02 & 10.19 & 4.11e-02 & 5.35 & 2.02e-02 & 3.67 & 55.29 & 79.46 & 76.52 & 73.20 & 63.03 & 77.15 & 69.38 & 70.58 \\
& \cellcolor{g2ptqblue!12}\name & \cellcolor{g2ptqblue!12}6.38e-02 & \cellcolor{g2ptqblue!12}10.05 & \cellcolor{g2ptqblue!12}3.85e-02 & \cellcolor{g2ptqblue!12}5.27 & \cellcolor{g2ptqblue!12}2.15e-02 & \cellcolor{g2ptqblue!12}3.63 & \cellcolor{g2ptqblue!12}55.80 & \cellcolor{g2ptqblue!12}80.30 & \cellcolor{g2ptqblue!12}76.23 & \cellcolor{g2ptqblue!12}73.97 & \cellcolor{g2ptqblue!12}63.32 & \cellcolor{g2ptqblue!12}77.09 & \cellcolor{g2ptqblue!12}68.51 & \cellcolor{g2ptqblue!12}70.75 \\
& \cellcolor{g2ptqblue!12}\name\textsuperscript{*} & \cellcolor{g2ptqblue!12}\textbf{6.15e-02} & \cellcolor{g2ptqblue!12}10.03 & \cellcolor{g2ptqblue!12}\textbf{3.63e-02} & \cellcolor{g2ptqblue!12}5.21 & \cellcolor{g2ptqblue!12}2.07e-02 & \cellcolor{g2ptqblue!12}3.62 & \cellcolor{g2ptqblue!12}56.91 & \cellcolor{g2ptqblue!12}80.93 & \cellcolor{g2ptqblue!12}78.16 & \cellcolor{g2ptqblue!12}73.51 & \cellcolor{g2ptqblue!12}64.64 & \cellcolor{g2ptqblue!12}77.80 & \cellcolor{g2ptqblue!12}68.98 & \cellcolor{g2ptqblue!12}\textbf{71.56} \\
\midrule

\multirow{7}{*}{\begin{tabular}{@{}c@{}}\textbf{Qwen3} \\ \textbf{14B}\end{tabular}}
& BF16 & 0 & 8.65 & 0 & 4.80 & 0 & 3.36 & 60.49 & 82.74 & 82.32 & 78.80 & 67.90 & 80.03 & 73.01 & 75.04 \\
\cmidrule{2-16}
& RTN & 1.86e-01 & 9.80 & 1.60e-01 & 4.96 & 1.10e-01 & 3.44 & 60.32 & 79.76 & 78.16 & 76.27 & 65.36 & 77.58 & 67.64 & 72.16 \\
& GPTQ & 4.60e-02 & 8.92 & 2.75e-02 & \textbf{4.81} & 1.20e-02 & 3.38 & 58.79 & 81.06 & 81.43 & 78.04 & 67.96 & 79.82 & 71.27 & 74.05 \\
& GuidedQ & 4.62e-02 & 8.88 & 2.65e-02 & 4.83 & 1.18e-02 & 3.39 & 60.67 & 81.86 & 81.43 & 78.10 & 67.61 & 79.92 & 72.45 & 74.58 \\
& GPTAQ & 4.52e-02 & 8.98 & 2.57e-02 & 4.83 & 1.09e-02 & 3.40 & 61.18 & 82.66 & 80.83 & 77.76 & 67.15 & 79.49 & 72.45 & 74.50 \\
& \cellcolor{g2ptqblue!12}\name & \cellcolor{g2ptqblue!12}4.30e-02 & \cellcolor{g2ptqblue!12}8.87 & \cellcolor{g2ptqblue!12}2.44e-02 & \cellcolor{g2ptqblue!12}4.83 & \cellcolor{g2ptqblue!12}1.09e-02 & \cellcolor{g2ptqblue!12}3.37 & \cellcolor{g2ptqblue!12}61.18 & \cellcolor{g2ptqblue!12}82.79 & \cellcolor{g2ptqblue!12}80.98 & \cellcolor{g2ptqblue!12}78.52 & \cellcolor{g2ptqblue!12}67.34 & \cellcolor{g2ptqblue!12}79.33 & \cellcolor{g2ptqblue!12}72.06 & \cellcolor{g2ptqblue!12}\textbf{74.60} \\
& \cellcolor{g2ptqblue!12}\name\textsuperscript{*} & \cellcolor{g2ptqblue!12}\textbf{4.24e-02} & \cellcolor{g2ptqblue!12}\textbf{8.83} & \cellcolor{g2ptqblue!12}\textbf{2.27e-02} & \cellcolor{g2ptqblue!12}\textbf{4.81} & \cellcolor{g2ptqblue!12}\textbf{9.47e-03} & \cellcolor{g2ptqblue!12}\textbf{3.36} & \cellcolor{g2ptqblue!12}59.98 & \cellcolor{g2ptqblue!12}81.94 & \cellcolor{g2ptqblue!12}81.13 & \cellcolor{g2ptqblue!12}78.17 & \cellcolor{g2ptqblue!12}67.13 & \cellcolor{g2ptqblue!12}80.20 & \cellcolor{g2ptqblue!12}72.38 & \cellcolor{g2ptqblue!12}74.42 \\
\midrule

\multirow{7}{*}{\begin{tabular}{@{}c@{}}\textbf{Qwen3} \\ \textbf{32B}\end{tabular}}
& BF16 & 0 & 7.61 & 0 & 4.05 & 0 & 2.80 & 61.01 & 83.21 & 86.03 & 82.60 & 67.13 & 82.05 & 73.24 & 76.47 \\
\cmidrule{2-16}
& RTN & 2.49e-01 & 8.54 & 1.49e-01 & 4.27 & 1.45e-01 & 2.99 & 56.14 & 76.94 & 81.05 & 80.44 & 62.78 & 80.09 & 70.72 & 72.59 \\
& GPTQ & 8.16e-02 & 7.91 & 2.61e-02 & 4.08 & 2.66e-02 & 2.85 & 59.39 & 83.21 & 85.44 & 82.05 & 68.06 & 81.45 & 71.27 & 75.84 \\
& GuidedQ & 8.23e-02 & 7.86 & \textbf{2.60e-02} & 4.08 & \textbf{2.53e-02} & 2.84 & 61.60 & 83.92 & 85.66 & 81.89 & 67.79 & 81.01 & 71.67 & 76.22 \\
& GPTAQ & 1.22e-01 & 8.11 & 4.99e-02 & 4.11 & 5.36e-02 & 2.88 & 60.41 & 82.32 & 84.32 & 81.31 & 68.64 & 81.12 & 72.30 & 75.77 \\
& \cellcolor{g2ptqblue!12}\name & \cellcolor{g2ptqblue!12}\textbf{7.45e-02} & \cellcolor{g2ptqblue!12}7.78 & \cellcolor{g2ptqblue!12}2.63e-02 & \cellcolor{g2ptqblue!12}4.07 & \cellcolor{g2ptqblue!12}3.03e-02 & \cellcolor{g2ptqblue!12}2.84 & \cellcolor{g2ptqblue!12}61.35 & \cellcolor{g2ptqblue!12}82.62 & \cellcolor{g2ptqblue!12}85.81 & \cellcolor{g2ptqblue!12}81.99 & \cellcolor{g2ptqblue!12}67.48 & \cellcolor{g2ptqblue!12}81.61 & \cellcolor{g2ptqblue!12}72.45 & \cellcolor{g2ptqblue!12}76.19 \\
& \cellcolor{g2ptqblue!12}\name\textsuperscript{*} & \cellcolor{g2ptqblue!12}7.78e-02 & \cellcolor{g2ptqblue!12}\textbf{7.76} & \cellcolor{g2ptqblue!12}2.63e-02 & \cellcolor{g2ptqblue!12}\textbf{4.04} & \cellcolor{g2ptqblue!12}3.15e-02 & \cellcolor{g2ptqblue!12}\textbf{2.82} & \cellcolor{g2ptqblue!12}62.12 & \cellcolor{g2ptqblue!12}83.33 & \cellcolor{g2ptqblue!12}85.51 & \cellcolor{g2ptqblue!12}82.04 & \cellcolor{g2ptqblue!12}67.28 & \cellcolor{g2ptqblue!12}81.34 & \cellcolor{g2ptqblue!12}73.56 & \cellcolor{g2ptqblue!12}\textbf{76.45} \\
\midrule

\multirow{7}{*}{\begin{tabular}{@{}c@{}}\textbf{Qwen3.5} \\ \textbf{0.8B}\end{tabular}}
& BF16 & 0 & 17.18 & 0 & 5.54 & 0 & 3.40 & 37.63 & 61.49 & 48.89 & 49.54 & 43.59 & 69.31 & 57.62 & 52.58 \\
\cmidrule{2-16}
& RTN & 2.74e-01 & 23.24 & 2.25e-01 & 6.64 & 1.82e-01 & 3.70 & 34.30 & 58.67 & 35.88 & 46.62 & 32.12 & 67.52 & 56.99 & 47.44 \\
& GPTQ & 1.13e-01 & 18.76 & 9.03e-02 & 5.99 & 5.15e-02 & \textbf{3.47} & 35.07 & 59.68 & 42.94 & 47.93 & 38.25 & 68.01 & 58.01 & 49.98 \\
& GuidedQ & 1.12e-01 & 19.38 & 8.88e-02 & 6.01 & 5.10e-02 & 3.49 & 35.15 & 60.27 & 43.02 & 47.90 & 40.09 & 68.06 & 57.30 & 50.26 \\
& GPTAQ & 1.07e-01 & 19.50 & 8.46e-02 & 6.01 & 4.67e-02 & 3.50 & 35.58 & 58.59 & 43.54 & 47.75 & 38.89 & 68.44 & 58.17 & 50.14 \\
& \cellcolor{g2ptqblue!12}\name & \cellcolor{g2ptqblue!12}1.01e-01 & \cellcolor{g2ptqblue!12}18.73 & \cellcolor{g2ptqblue!12}7.90e-02 & \cellcolor{g2ptqblue!12}5.98 & \cellcolor{g2ptqblue!12}4.39e-02 & \cellcolor{g2ptqblue!12}3.50 & \cellcolor{g2ptqblue!12}34.64 & \cellcolor{g2ptqblue!12}58.29 & \cellcolor{g2ptqblue!12}41.08 & \cellcolor{g2ptqblue!12}48.12 & \cellcolor{g2ptqblue!12}40.91 & \cellcolor{g2ptqblue!12}68.66 & \cellcolor{g2ptqblue!12}59.35 & \cellcolor{g2ptqblue!12}50.15 \\
& \cellcolor{g2ptqblue!12}\name\textsuperscript{*} & \cellcolor{g2ptqblue!12}\textbf{9.19e-02} & \cellcolor{g2ptqblue!12}\textbf{18.21} & \cellcolor{g2ptqblue!12}\textbf{6.74e-02} & \cellcolor{g2ptqblue!12}\textbf{5.90} & \cellcolor{g2ptqblue!12}\textbf{3.53e-02} & \cellcolor{g2ptqblue!12}3.49 & \cellcolor{g2ptqblue!12}35.49 & \cellcolor{g2ptqblue!12}60.98 & \cellcolor{g2ptqblue!12}42.12 & \cellcolor{g2ptqblue!12}48.15 & \cellcolor{g2ptqblue!12}41.68 & \cellcolor{g2ptqblue!12}68.50 & \cellcolor{g2ptqblue!12}57.93 & \cellcolor{g2ptqblue!12}\textbf{50.69} \\
\midrule

\multirow{7}{*}{\begin{tabular}{@{}c@{}}\textbf{Qwen3.5} \\ \textbf{2B}\end{tabular}}
& BF16 & 0 & 12.10 & 0 & 4.66 & 0 & 2.99 & 41.47 & 65.95 & 59.06 & 62.11 & 53.31 & 72.69 & 62.67 & 59.61 \\
\cmidrule{2-16}
& RTN & 2.20e-01 & 15.05 & 1.84e-01 & 5.27 & 1.71e-01 & 3.37 & 40.27 & 64.77 & 45.99 & 59.47 & 47.04 & 72.52 & 62.83 & 56.13 \\
& GPTQ & 8.74e-02 & 13.07 & 6.38e-02 & 4.87 & 3.97e-02 & \textbf{3.07} & 38.31 & 60.86 & 55.35 & 60.18 & 48.69 & 71.27 & 61.33 & 56.57 \\
& GuidedQ & 8.85e-02 & 13.21 & 6.17e-02 & 4.88 & 4.46e-02 & 3.09 & 40.10 & 64.23 & 53.71 & 60.56 & 51.19 & 72.36 & 61.56 & 57.67 \\
& GPTAQ & 8.50e-02 & 13.23 & 5.86e-02 & 4.88 & 4.07e-02 & 3.09 & 38.65 & 59.68 & 55.79 & 60.00 & 49.25 & 71.93 & 61.56 & 56.69 \\
& \cellcolor{g2ptqblue!12}\name & \cellcolor{g2ptqblue!12}8.01e-02 & \cellcolor{g2ptqblue!12}\textbf{12.98} & \cellcolor{g2ptqblue!12}5.57e-02 & \cellcolor{g2ptqblue!12}4.87 & \cellcolor{g2ptqblue!12}3.90e-02 & \cellcolor{g2ptqblue!12}3.09 & \cellcolor{g2ptqblue!12}40.02 & \cellcolor{g2ptqblue!12}62.29 & \cellcolor{g2ptqblue!12}53.64 & \cellcolor{g2ptqblue!12}60.48 & \cellcolor{g2ptqblue!12}50.53 & \cellcolor{g2ptqblue!12}72.36 & \cellcolor{g2ptqblue!12}61.25 & \cellcolor{g2ptqblue!12}57.22 \\
& \cellcolor{g2ptqblue!12}\name\textsuperscript{*} & \cellcolor{g2ptqblue!12}\textbf{7.71e-02} & \cellcolor{g2ptqblue!12}13.05 & \cellcolor{g2ptqblue!12}\textbf{5.00e-02} & \cellcolor{g2ptqblue!12}\textbf{4.81} & \cellcolor{g2ptqblue!12}\textbf{3.40e-02} & \cellcolor{g2ptqblue!12}3.08 & \cellcolor{g2ptqblue!12}40.27 & \cellcolor{g2ptqblue!12}63.01 & \cellcolor{g2ptqblue!12}54.98 & \cellcolor{g2ptqblue!12}60.59 & \cellcolor{g2ptqblue!12}51.23 & \cellcolor{g2ptqblue!12}72.14 & \cellcolor{g2ptqblue!12}63.54 & \cellcolor{g2ptqblue!12}\textbf{57.97} \\
\midrule

\multirow{7}{*}{\begin{tabular}{@{}c@{}}\textbf{Qwen3.5} \\ \textbf{4B}\end{tabular}}
& BF16 & 0 & 9.58 & 0 & 4.33 & 0 & 2.72 & 54.27 & 75.46 & 72.51 & 73.11 & 64.74 & 77.97 & 69.93 & 69.71 \\
\cmidrule{2-16}
& RTN & 1.73e-01 & 10.94 & 1.55e-01 & 4.70 & 1.19e-01 & 2.94 & 50.09 & 67.51 & 69.54 & 70.91 & 61.07 & 76.28 & 67.25 & 66.09 \\
& GPTQ & 9.68e-02 & 10.60 & 6.59e-02 & 4.54 & 3.55e-02 & \textbf{2.76} & 54.86 & 75.59 & 70.43 & 71.68 & 63.34 & 77.20 & 70.09 & 69.03 \\
& GuidedQ & 9.80e-02 & 10.51 & 6.64e-02 & \textbf{4.45} & 4.23e-02 & 2.80 & 55.12 & 76.64 & 69.84 & 71.75 & 63.17 & 77.86 & 68.98 & 69.05 \\
& GPTAQ & 9.20e-02 & 10.53 & 6.02e-02 & 4.52 & 3.39e-02 & 2.78 & 54.95 & 77.06 & 71.17 & 71.62 & 62.74 & 77.42 & 68.59 & 69.08 \\
& \cellcolor{g2ptqblue!12}\name & \cellcolor{g2ptqblue!12}9.63e-02 & \cellcolor{g2ptqblue!12}10.64 & \cellcolor{g2ptqblue!12}5.49e-02 & \cellcolor{g2ptqblue!12}\textbf{4.45} & \cellcolor{g2ptqblue!12}3.23e-02 & \cellcolor{g2ptqblue!12}\textbf{2.76} & \cellcolor{g2ptqblue!12}54.61 & \cellcolor{g2ptqblue!12}76.98 & \cellcolor{g2ptqblue!12}70.06 & \cellcolor{g2ptqblue!12}71.69 & \cellcolor{g2ptqblue!12}62.82 & \cellcolor{g2ptqblue!12}77.91 & \cellcolor{g2ptqblue!12}68.90 & \cellcolor{g2ptqblue!12}69.00 \\
& \cellcolor{g2ptqblue!12}\name\textsuperscript{*} & \cellcolor{g2ptqblue!12}\textbf{8.20e-02} & \cellcolor{g2ptqblue!12}\textbf{10.27} & \cellcolor{g2ptqblue!12}\textbf{5.17e-02} & \cellcolor{g2ptqblue!12}4.49 & \cellcolor{g2ptqblue!12}\textbf{3.18e-02} & \cellcolor{g2ptqblue!12}\textbf{2.76} & \cellcolor{g2ptqblue!12}54.61 & \cellcolor{g2ptqblue!12}76.56 & \cellcolor{g2ptqblue!12}71.17 & \cellcolor{g2ptqblue!12}71.53 & \cellcolor{g2ptqblue!12}62.62 & \cellcolor{g2ptqblue!12}77.20 & \cellcolor{g2ptqblue!12}70.09 & \cellcolor{g2ptqblue!12}\textbf{69.11} \\
\midrule

\multirow{7}{*}{\begin{tabular}{@{}c@{}}\textbf{Qwen3.5} \\ \textbf{9B}\end{tabular}}
& BF16 & 0 & 8.65 & 0 & 3.78 & 0 & 2.48 & 55.46 & 74.37 & 78.75 & 78.08 & 70.04 & 79.98 & 72.93 & 72.80 \\
\cmidrule{2-16}
& RTN & 3.33e-01 & 9.29 & 2.82e-01 & 4.17 & 1.96e-01 & 2.80 & 51.02 & 73.53 & 71.03 & 74.47 & 62.86 & 78.45 & 70.56 & 68.85 \\
& GPTQ & 6.83e-02 & 8.84 & 2.84e-02 & 3.83 & 1.50e-02 & 2.50 & 53.92 & 74.37 & 78.01 & 77.15 & 70.42 & 79.98 & 72.93 & \textbf{72.40} \\
& GuidedQ & 7.31e-02 & 8.76 & 3.10e-02 & 3.85 & 1.77e-02 & 2.50 & 54.52 & 75.21 & 77.19 & 76.89 & 69.14 & 80.14 & 73.48 & 72.37 \\
& GPTAQ & 6.91e-02 & 8.77 & 2.74e-02 & 3.84 & 1.48e-02 & 2.50 & 55.03 & 74.71 & 77.56 & 76.89 & 69.24 & 80.03 & 72.14 & 72.23 \\
& \cellcolor{g2ptqblue!12}\name & \cellcolor{g2ptqblue!12}6.99e-02 & \cellcolor{g2ptqblue!12}\textbf{8.63} & \cellcolor{g2ptqblue!12}2.72e-02 & \cellcolor{g2ptqblue!12}3.83 & \cellcolor{g2ptqblue!12}1.57e-02 & \cellcolor{g2ptqblue!12}2.50 & \cellcolor{g2ptqblue!12}55.38 & \cellcolor{g2ptqblue!12}75.08 & \cellcolor{g2ptqblue!12}76.89 & \cellcolor{g2ptqblue!12}77.04 & \cellcolor{g2ptqblue!12}69.51 & \cellcolor{g2ptqblue!12}80.30 & \cellcolor{g2ptqblue!12}72.06 & \cellcolor{g2ptqblue!12}72.32 \\
& \cellcolor{g2ptqblue!12}\name\textsuperscript{*} & \cellcolor{g2ptqblue!12}\textbf{6.59e-02} & \cellcolor{g2ptqblue!12}8.79 & \cellcolor{g2ptqblue!12}\textbf{2.53e-02} & \cellcolor{g2ptqblue!12}\textbf{3.82} & \cellcolor{g2ptqblue!12}\textbf{1.40e-02} & \cellcolor{g2ptqblue!12}\textbf{2.49} & \cellcolor{g2ptqblue!12}53.84 & \cellcolor{g2ptqblue!12}74.07 & \cellcolor{g2ptqblue!12}78.45 & \cellcolor{g2ptqblue!12}76.88 & \cellcolor{g2ptqblue!12}68.74 & \cellcolor{g2ptqblue!12}80.25 & \cellcolor{g2ptqblue!12}72.93 & \cellcolor{g2ptqblue!12}72.17 \\
\midrule

\multirow{7}{*}{\begin{tabular}{@{}c@{}}\textbf{Qwen3.5} \\ \textbf{27B}\end{tabular}}
& BF16 & 0 & 6.90 & 0 & 3.55 & 0 & 2.28 & 61.26 & 79.67 & 85.44 & 83.34 & 75.16 & 82.05 & 78.69 & 77.94 \\
\cmidrule{2-16}
& RTN & 1.28e-01 & 7.39 & 6.56e-02 & 3.66 & 4.20e-02 & 2.32 & 59.98 & 80.89 & 84.03 & 82.22 & 73.16 & 81.45 & 76.87 & 76.94 \\
& GPTQ & \textbf{5.27e-02} & 7.01 & 1.89e-02 & \textbf{3.56} & 1.19e-02 & \textbf{2.29} & 61.77 & 80.35 & 84.70 & 82.97 & 75.59 & 82.21 & 78.85 & 78.06 \\
& GuidedQ & 5.91e-02 & 7.07 & 2.01e-02 & 3.57 & 1.33e-02 & \textbf{2.29} & 62.03 & 80.22 & 85.07 & 82.86 & 75.78 & 82.64 & 78.77 & \textbf{78.20} \\
& GPTAQ & 5.39e-02 & 7.08 & 1.91e-02 & 3.57 & 1.21e-02 & \textbf{2.29} & 61.09 & 80.13 & 84.47 & 82.98 & 75.30 & 82.37 & 78.77 & 77.87 \\
& \cellcolor{g2ptqblue!12}\name & \cellcolor{g2ptqblue!12}5.32e-02 & \cellcolor{g2ptqblue!12}\textbf{6.99} & \cellcolor{g2ptqblue!12}1.84e-02 & \cellcolor{g2ptqblue!12}\textbf{3.56} & \cellcolor{g2ptqblue!12}1.28e-02 & \cellcolor{g2ptqblue!12}\textbf{2.29} & \cellcolor{g2ptqblue!12}60.92 & \cellcolor{g2ptqblue!12}80.26 & \cellcolor{g2ptqblue!12}84.77 & \cellcolor{g2ptqblue!12}82.82 & \cellcolor{g2ptqblue!12}75.22 & \cellcolor{g2ptqblue!12}82.10 & \cellcolor{g2ptqblue!12}78.14 & \cellcolor{g2ptqblue!12}77.75 \\
& \cellcolor{g2ptqblue!12}\name\textsuperscript{*} & \cellcolor{g2ptqblue!12}5.37e-02 & \cellcolor{g2ptqblue!12}7.07 & \cellcolor{g2ptqblue!12}\textbf{1.79e-02} & \cellcolor{g2ptqblue!12}3.57 & \cellcolor{g2ptqblue!12}\textbf{1.15e-02} & \cellcolor{g2ptqblue!12}\textbf{2.29} & \cellcolor{g2ptqblue!12}61.77 & \cellcolor{g2ptqblue!12}79.00 & \cellcolor{g2ptqblue!12}84.55 & \cellcolor{g2ptqblue!12}82.81 & \cellcolor{g2ptqblue!12}75.26 & \cellcolor{g2ptqblue!12}82.59 & \cellcolor{g2ptqblue!12}78.45 & \cellcolor{g2ptqblue!12}77.78 \\
\midrule

\multirow{7}{*}{\begin{tabular}{@{}c@{}}\textbf{LLaMA3} \\ \textbf{8B}\end{tabular}}
& BF16 & 0 & 6.14 & 0 & 3.41 & 0 & 2.72 & 53.24 & 77.90 & 48.44 & 79.25 & 75.53 & 80.74 & 73.16 & 69.75 \\
\cmidrule{2-16}
& RTN & 2.04e-01 & 7.51 & 1.24e-01 & 3.80 & 1.08e-01 & 3.01 & 49.66 & 74.41 & 40.34 & 76.57 & 70.74 & 79.16 & 71.35 & 66.03 \\
& GPTQ & 8.42e-02 & 6.66 & 3.62e-02 & 3.53 & 2.92e-02 & 2.80 & 51.71 & 76.18 & 45.54 & 77.53 & 74.99 & 79.92 & 73.24 & 68.44 \\
& GuidedQ & 7.04e-02 & 6.59 & 2.88e-02 & 3.50 & 2.22e-02 & 2.78 & 51.28 & 76.89 & 46.51 & 78.15 & 75.98 & 79.98 & 72.45 & 68.75 \\
& GPTAQ & 7.54e-02 & 6.61 & 2.90e-02 & 3.50 & 2.41e-02 & 2.78 & 52.65 & 76.43 & 46.66 & 77.71 & 75.20 & 79.76 & 72.61 & 68.72 \\
& \cellcolor{g2ptqblue!12}\name & \cellcolor{g2ptqblue!12}\textbf{6.86e-02} & \cellcolor{g2ptqblue!12}6.57 & \cellcolor{g2ptqblue!12}2.67e-02 & \cellcolor{g2ptqblue!12}3.50 & \cellcolor{g2ptqblue!12}2.28e-02 & \cellcolor{g2ptqblue!12}\textbf{2.77} & \cellcolor{g2ptqblue!12}52.22 & \cellcolor{g2ptqblue!12}76.77 & \cellcolor{g2ptqblue!12}47.55 & \cellcolor{g2ptqblue!12}78.16 & \cellcolor{g2ptqblue!12}75.47 & \cellcolor{g2ptqblue!12}79.76 & \cellcolor{g2ptqblue!12}73.48 & \cellcolor{g2ptqblue!12}69.06 \\
& \cellcolor{g2ptqblue!12}\name\textsuperscript{*} & \cellcolor{g2ptqblue!12}6.89e-02 & \cellcolor{g2ptqblue!12}\textbf{6.56} & \cellcolor{g2ptqblue!12}\textbf{2.45e-02} & \cellcolor{g2ptqblue!12}\textbf{3.48} & \cellcolor{g2ptqblue!12}\textbf{1.90e-02} & \cellcolor{g2ptqblue!12}\textbf{2.77} & \cellcolor{g2ptqblue!12}52.30 & \cellcolor{g2ptqblue!12}78.11 & \cellcolor{g2ptqblue!12}48.96 & \cellcolor{g2ptqblue!12}77.95 & \cellcolor{g2ptqblue!12}74.89 & \cellcolor{g2ptqblue!12}79.92 & \cellcolor{g2ptqblue!12}73.72 & \cellcolor{g2ptqblue!12}\textbf{69.41} \\
\midrule

\multirow{7}{*}{\begin{tabular}{@{}c@{}}\textbf{LLaMA3} \\ \textbf{70B}\end{tabular}}
& BF16 & 0 & 2.86 & 0 & 3.05 & 0 & 2.41 & 64.16 & 85.86 & 65.08 & 84.95 & 79.45 & 84.49 & 80.43 & 77.77 \\
\cmidrule{2-16}
& RTN & 1.65e+00 & 15.00 & 7.62e-01 & 6.34 & 5.36e-01 & 4.01 & 41.04 & 68.52 & 24.37 & 61.47 & 45.31 & 71.82 & 67.32 & 54.26 \\
& GPTQ & 2.42e-01 & 3.63 & 2.16e-02 & 3.10 & 2.23e-02 & 2.45 & 62.80 & 84.97 & 62.18 & 84.69 & 79.10 & 84.60 & 80.03 & 76.91 \\
& GuidedQ & 1.76e-01 & 3.40 & 1.39e-02 & 3.09 & 1.19e-02 & 2.43 & 63.65 & 85.23 & 62.70 & 84.45 & 79.57 & 83.68 & 79.95 & 77.03 \\
& GPTAQ & 2.02e-01 & 3.49 & 1.64e-02 & 3.09 & 1.43e-02 & 2.43 & 63.14 & 85.23 & 63.00 & 84.80 & 79.33 & 83.84 & 79.87 & 77.03 \\
& \cellcolor{g2ptqblue!12}\name & \cellcolor{g2ptqblue!12}\textbf{1.69e-01} & \cellcolor{g2ptqblue!12}\textbf{3.38} & \cellcolor{g2ptqblue!12}\textbf{1.25e-02} & \cellcolor{g2ptqblue!12}\textbf{3.08} & \cellcolor{g2ptqblue!12}1.14e-02 & \cellcolor{g2ptqblue!12}2.43 & \cellcolor{g2ptqblue!12}62.97 & \cellcolor{g2ptqblue!12}85.52 & \cellcolor{g2ptqblue!12}63.60 & \cellcolor{g2ptqblue!12}84.62 & \cellcolor{g2ptqblue!12}79.06 & \cellcolor{g2ptqblue!12}84.28 & \cellcolor{g2ptqblue!12}79.87 & \cellcolor{g2ptqblue!12}77.13 \\
& \cellcolor{g2ptqblue!12}\name\textsuperscript{*} & \cellcolor{g2ptqblue!12}1.93e-01 & \cellcolor{g2ptqblue!12}3.46 & \cellcolor{g2ptqblue!12}1.33e-02 & \cellcolor{g2ptqblue!12}3.09 & \cellcolor{g2ptqblue!12}\textbf{1.09e-02} & \cellcolor{g2ptqblue!12}\textbf{2.42} & \cellcolor{g2ptqblue!12}63.31 & \cellcolor{g2ptqblue!12}85.52 & \cellcolor{g2ptqblue!12}63.08 & \cellcolor{g2ptqblue!12}84.76 & \cellcolor{g2ptqblue!12}79.18 & \cellcolor{g2ptqblue!12}84.33 & \cellcolor{g2ptqblue!12}81.06 & \cellcolor{g2ptqblue!12}\textbf{77.32} \\

\bottomrule
\end{tabular}
\end{adjustbox}
\vspace{-1ex}
\caption{4-bit weight-only quantization results across different model families.}
\label{tab:w4}
\end{table*}

%% file: tables/appendix/w4a4.tex
\begin{table*}[!t]
\begin{center}
\resizebox{1.0\linewidth}{!}
{
\begin{tabular}{c|l|cc|cc|cc|cccccccc}
\toprule
\multirow{2}{*}{\textbf{Model}} & \multirow{2}{*}{\textbf{Method}} & \multicolumn{2}{c|}{\textbf{WikiText2}} & \multicolumn{2}{c|}{\textbf{UltraChat}} & \multicolumn{2}{c|}{\textbf{NuminaMath}} & \multicolumn{8}{c}{\textbf{CommonSense QA}} \\
\cmidrule(lr){3-4} \cmidrule(lr){5-6} \cmidrule(lr){7-8} \cmidrule(lr){9-16}
& & \textbf{KL} & \textbf{PPL} & \textbf{KL} & \textbf{PPL} & \textbf{KL} & \textbf{PPL} & \textbf{ARC-C} & \textbf{ARC-E} & \textbf{C-eval} & \textbf{HellaS} & \textbf{LAMB} & \textbf{PIQA} & \textbf{Wino} & \textbf{Avg.} \\
\midrule

\multirow{7}{*}{\begin{tabular}{@{}c@{}}\textbf{Qwen3} \\ \textbf{8B}\end{tabular}}
& BF16 & 0 & 9.72 & 0 & 5.20 & 0 & 3.60 & 56.14 & 80.77 & 79.57 & 74.98 & 64.10 & 77.97 & 67.96 & 71.64 \\
\cmidrule{2-16}
& RTN & 4.16e-01 & 12.33 & 3.42e-01 & 5.58 & 2.41e-01 & 3.81 & 44.62 & 64.44 & 64.86 & 66.48 & 48.52 & 73.88 & 64.01 & 60.97 \\
& GPTQ & 2.49e-01 & 11.32 & 1.85e-01 & 5.56 & 1.24e-01 & 3.81 & 52.05 & 75.38 & 69.24 & 69.33 & 57.87 & 75.68 & 64.64 & 66.31 \\
& GuidedQ & 2.54e-01 & 11.44 & 1.87e-01 & 5.59 & 1.25e-01 & 3.83 & 49.83 & 75.29 & 67.16 & 69.21 & 58.33 & 74.81 & 63.54 & 65.45 \\
& GPTAQ & 2.56e-01 & 11.81 & 1.86e-01 & 5.82 & 1.23e-01 & 4.00 & 48.12 & 72.94 & 69.02 & 68.04 & 57.91 & 74.54 & 63.93 & 64.93 \\
& \cellcolor{g2ptqblue!12}\name & \cellcolor{g2ptqblue!12}2.25e-01 & \cellcolor{g2ptqblue!12}11.02 & \cellcolor{g2ptqblue!12}1.64e-01 & \cellcolor{g2ptqblue!12}5.43 & \cellcolor{g2ptqblue!12}1.11e-01 & \cellcolor{g2ptqblue!12}3.76 & \cellcolor{g2ptqblue!12}49.23 & \cellcolor{g2ptqblue!12}74.79 & \cellcolor{g2ptqblue!12}69.17 & \cellcolor{g2ptqblue!12}69.80 & \cellcolor{g2ptqblue!12}61.09 & \cellcolor{g2ptqblue!12}74.54 & \cellcolor{g2ptqblue!12}64.01 & \cellcolor{g2ptqblue!12}66.09 \\
& \cellcolor{g2ptqblue!12}\name\textsuperscript{*} & \cellcolor{g2ptqblue!12}\textbf{2.21e-01} & \cellcolor{g2ptqblue!12}\textbf{10.98} & \cellcolor{g2ptqblue!12}\textbf{1.61e-01} & \cellcolor{g2ptqblue!12}\textbf{5.32} & \cellcolor{g2ptqblue!12}\textbf{1.03e-01} & \cellcolor{g2ptqblue!12}\textbf{3.74} & \cellcolor{g2ptqblue!12}49.57 & \cellcolor{g2ptqblue!12}75.76 & \cellcolor{g2ptqblue!12}68.95 & \cellcolor{g2ptqblue!12}69.58 & \cellcolor{g2ptqblue!12}60.28 & \cellcolor{g2ptqblue!12}75.73 & \cellcolor{g2ptqblue!12}65.98 & \cellcolor{g2ptqblue!12}\textbf{66.55} \\
\midrule

\multirow{7}{*}{\begin{tabular}{@{}c@{}}\textbf{Qwen3} \\ \textbf{14B}\end{tabular}}
& BF16 & 0 & 8.65 & 0 & 4.80 & 0 & 3.36 & 60.49 & 82.74 & 82.32 & 78.80 & 67.90 & 80.03 & 73.01 & 75.04 \\
\cmidrule{2-16}
& RTN & 3.55e-01 & 11.02 & 2.87e-01 & 5.11 & 2.04e-01 & 3.55 & 52.90 & 74.79 & 72.51 & 72.18 & 60.86 & 76.71 & 67.01 & 68.14 \\
& GPTQ & 1.99e-01 & 9.83 & 1.44e-01 & 4.91 & 8.98e-02 & \textbf{3.39} & 54.86 & 76.98 & 76.97 & 74.55 & 63.32 & 77.53 & 67.40 & 70.23 \\
& GuidedQ & 1.98e-01 & 9.78 & 1.45e-01 & 4.91 & 9.16e-02 & 3.40 & 56.48 & 76.47 & 74.89 & 74.78 & 63.67 & 77.53 & 67.64 & 70.21 \\
& GPTAQ & 1.87e-01 & 9.97 & 1.28e-01 & 5.07 & 7.96e-02 & 3.54 & 55.12 & 77.61 & 76.08 & 74.43 & 63.48 & 77.09 & 68.82 & 70.38 \\
& \cellcolor{g2ptqblue!12}\name & \cellcolor{g2ptqblue!12}1.79e-01 & \cellcolor{g2ptqblue!12}9.69 & \cellcolor{g2ptqblue!12}1.22e-01 & \cellcolor{g2ptqblue!12}4.88 & \cellcolor{g2ptqblue!12}7.62e-02 & \cellcolor{g2ptqblue!12}3.47 & \cellcolor{g2ptqblue!12}54.86 & \cellcolor{g2ptqblue!12}78.28 & \cellcolor{g2ptqblue!12}74.89 & \cellcolor{g2ptqblue!12}75.39 & \cellcolor{g2ptqblue!12}64.99 & \cellcolor{g2ptqblue!12}78.67 & \cellcolor{g2ptqblue!12}68.67 & \cellcolor{g2ptqblue!12}70.82 \\
& \cellcolor{g2ptqblue!12}\name\textsuperscript{*} & \cellcolor{g2ptqblue!12}\textbf{1.73e-01} & \cellcolor{g2ptqblue!12}\textbf{9.53} & \cellcolor{g2ptqblue!12}\textbf{1.20e-01} & \cellcolor{g2ptqblue!12}\textbf{4.82} & \cellcolor{g2ptqblue!12}\textbf{7.27e-02} & \cellcolor{g2ptqblue!12}3.45 & \cellcolor{g2ptqblue!12}54.35 & \cellcolor{g2ptqblue!12}78.79 & \cellcolor{g2ptqblue!12}76.30 & \cellcolor{g2ptqblue!12}75.69 & \cellcolor{g2ptqblue!12}64.54 & \cellcolor{g2ptqblue!12}78.07 & \cellcolor{g2ptqblue!12}69.30 & \cellcolor{g2ptqblue!12}\textbf{71.01} \\
\midrule

\multirow{7}{*}{\begin{tabular}{@{}c@{}}\textbf{Qwen3} \\ \textbf{32B}\end{tabular}}
& BF16 & 0 & 7.61 & 0 & 4.05 & 0 & 2.80 & 61.01 & 83.21 & 86.03 & 82.60 & 67.13 & 82.05 & 73.24 & 76.47 \\
\cmidrule{2-16}
& RTN & 4.59e-01 & 9.83 & 2.76e-01 & 4.48 & 2.58e-01 & 3.18 & 51.79 & 71.38 & 73.40 & 76.77 & 58.18 & 76.55 & 64.56 & 67.52 \\
& GPTQ & 3.34e-01 & 9.11 & 1.67e-01 & 4.32 & 1.58e-01 & 3.01 & 56.83 & 77.02 & 77.86 & 77.72 & 63.79 & 78.89 & 69.30 & 71.63 \\
& GuidedQ & 3.39e-01 & 9.11 & 1.68e-01 & 4.31 & 1.59e-01 & 3.01 & 56.91 & 77.27 & 77.56 & 77.52 & 63.69 & 77.26 & 66.38 & 70.94 \\
& GPTAQ & 3.34e-01 & 9.23 & 1.57e-01 & 4.27 & 1.49e-01 & 3.06 & 52.73 & 76.56 & 77.12 & 77.50 & 65.57 & 77.75 & 66.30 & 70.50 \\
& \cellcolor{g2ptqblue!12}\name & \cellcolor{g2ptqblue!12}2.72e-01 & \cellcolor{g2ptqblue!12}8.68 & \cellcolor{g2ptqblue!12}1.36e-01 & \cellcolor{g2ptqblue!12}4.21 & \cellcolor{g2ptqblue!12}1.23e-01 & \cellcolor{g2ptqblue!12}2.94 & \cellcolor{g2ptqblue!12}57.42 & \cellcolor{g2ptqblue!12}80.01 & \cellcolor{g2ptqblue!12}80.46 & \cellcolor{g2ptqblue!12}79.62 & \cellcolor{g2ptqblue!12}66.52 & \cellcolor{g2ptqblue!12}79.54 & \cellcolor{g2ptqblue!12}70.01 & \cellcolor{g2ptqblue!12}\textbf{73.37} \\
& \cellcolor{g2ptqblue!12}\name\textsuperscript{*} & \cellcolor{g2ptqblue!12}\textbf{2.61e-01} & \cellcolor{g2ptqblue!12}\textbf{8.57} & \cellcolor{g2ptqblue!12}\textbf{1.32e-01} & \cellcolor{g2ptqblue!12}\textbf{4.19} & \cellcolor{g2ptqblue!12}\textbf{1.16e-01} & \cellcolor{g2ptqblue!12}\textbf{2.93} & \cellcolor{g2ptqblue!12}56.14 & \cellcolor{g2ptqblue!12}80.13 & \cellcolor{g2ptqblue!12}81.05 & \cellcolor{g2ptqblue!12}79.45 & \cellcolor{g2ptqblue!12}65.30 & \cellcolor{g2ptqblue!12}80.79 & \cellcolor{g2ptqblue!12}70.09 & \cellcolor{g2ptqblue!12}73.28 \\
\midrule

\multirow{7}{*}{\begin{tabular}{@{}c@{}}\textbf{Qwen3.5} \\ \textbf{9B}\end{tabular}}
& BF16 & 0 & 8.65 & 0 & 3.78 & 0 & 2.48 & 55.46 & 74.37 & 78.75 & 78.08 & 70.04 & 79.98 & 72.93 & 72.80 \\
\cmidrule{2-16}
& RTN & 2.76e+00 & 108.18 & 2.62e+00 & 37.41 & 1.50e+00 & 10.03 & 25.09 & 48.48 & 27.04 & 33.14 & 11.60 & 59.30 & 52.01 & 36.67 \\
& GPTQ & 9.87e-01 & 19.21 & 6.85e-01 & 6.10 & 4.22e-01 & 3.56 & 42.15 & 66.29 & 42.87 & 59.69 & 49.68 & 71.33 & 59.43 & 55.92 \\
& GuidedQ & 9.99e-01 & 19.36 & 6.92e-01 & 6.14 & 4.34e-01 & 3.62 & 41.55 & 64.48 & 41.60 & 59.40 & 49.72 & 70.89 & 61.33 & 55.57 \\
& GPTAQ & 9.18e-01 & 19.04 & 4.84e-01 & 5.44 & 2.85e-01 & 3.17 & 42.15 & 66.58 & 46.21 & 61.65 & 52.57 & 70.95 & 61.96 & 57.44 \\
& \cellcolor{g2ptqblue!12}\name & \cellcolor{g2ptqblue!12}7.57e-01 & \cellcolor{g2ptqblue!12}\textbf{15.22} & \cellcolor{g2ptqblue!12}4.07e-01 & \cellcolor{g2ptqblue!12}5.07 & \cellcolor{g2ptqblue!12}2.63e-01 & \cellcolor{g2ptqblue!12}3.11 & \cellcolor{g2ptqblue!12}44.45 & \cellcolor{g2ptqblue!12}70.50 & \cellcolor{g2ptqblue!12}49.11 & \cellcolor{g2ptqblue!12}65.77 & \cellcolor{g2ptqblue!12}56.37 & \cellcolor{g2ptqblue!12}73.01 & \cellcolor{g2ptqblue!12}64.17 & \cellcolor{g2ptqblue!12}60.48 \\
& \cellcolor{g2ptqblue!12}\name\textsuperscript{*} & \cellcolor{g2ptqblue!12}\textbf{7.48e-01} & \cellcolor{g2ptqblue!12}15.28 & \cellcolor{g2ptqblue!12}\textbf{3.86e-01} & \cellcolor{g2ptqblue!12}\textbf{4.90} & \cellcolor{g2ptqblue!12}\textbf{2.37e-01} & \cellcolor{g2ptqblue!12}\textbf{3.04} & \cellcolor{g2ptqblue!12}46.50 & \cellcolor{g2ptqblue!12}69.74 & \cellcolor{g2ptqblue!12}50.82 & \cellcolor{g2ptqblue!12}64.98 & \cellcolor{g2ptqblue!12}56.28 & \cellcolor{g2ptqblue!12}73.61 & \cellcolor{g2ptqblue!12}64.17 & \cellcolor{g2ptqblue!12}\textbf{60.87} \\
\midrule

\multirow{7}{*}{\begin{tabular}{@{}c@{}}\textbf{Qwen3.5} \\ \textbf{27B}\end{tabular}}
& BF16 & 0 & 6.90 & 0 & 3.55 & 0 & 2.28 & 61.26 & 79.67 & 85.44 & 83.34 & 75.16 & 82.05 & 78.69 & 77.94 \\
\cmidrule{2-16}
& RTN & 1.01e+00 & 17.05 & 6.59e-01 & 5.53 & 2.72e-01 & 2.79 & 47.27 & 70.03 & 57.65 & 65.06 & 48.90 & 71.00 & 63.61 & 60.50 \\
& GPTQ & 6.77e-01 & 12.32 & 4.31e-01 & 4.64 & 1.85e-01 & 2.58 & 50.68 & 75.97 & 66.94 & 73.00 & 63.52 & 76.44 & 68.82 & 67.91 \\
& GuidedQ & 6.91e-01 & 12.56 & 4.35e-01 & 4.68 & 1.87e-01 & 2.59 & 51.96 & 74.58 & 65.60 & 72.72 & 64.08 & 75.24 & 68.98 & 67.59 \\
& GPTAQ & 5.23e-01 & 11.00 & 2.54e-01 & 4.18 & 1.41e-01 & 2.55 & 54.35 & 75.34 & 67.16 & 73.60 & 68.56 & 76.77 & 72.06 & 69.69 \\
& \cellcolor{g2ptqblue!12}\name & \cellcolor{g2ptqblue!12}4.40e-01 & \cellcolor{g2ptqblue!12}9.89 & \cellcolor{g2ptqblue!12}\textbf{2.30e-01} & \cellcolor{g2ptqblue!12}\textbf{4.00} & \cellcolor{g2ptqblue!12}1.17e-01 & \cellcolor{g2ptqblue!12}\textbf{2.49} & \cellcolor{g2ptqblue!12}54.52 & \cellcolor{g2ptqblue!12}78.49 & \cellcolor{g2ptqblue!12}72.88 & \cellcolor{g2ptqblue!12}76.02 & \cellcolor{g2ptqblue!12}65.55 & \cellcolor{g2ptqblue!12}78.40 & \cellcolor{g2ptqblue!12}72.14 & \cellcolor{g2ptqblue!12}\textbf{71.14} \\
& \cellcolor{g2ptqblue!12}\name\textsuperscript{*} & \cellcolor{g2ptqblue!12}\textbf{4.35e-01} & \cellcolor{g2ptqblue!12}\textbf{9.80} & \cellcolor{g2ptqblue!12}\textbf{2.30e-01} & \cellcolor{g2ptqblue!12}4.01 & \cellcolor{g2ptqblue!12}\textbf{1.14e-01} & \cellcolor{g2ptqblue!12}2.50 & \cellcolor{g2ptqblue!12}54.69 & \cellcolor{g2ptqblue!12}75.88 & \cellcolor{g2ptqblue!12}71.77 & \cellcolor{g2ptqblue!12}75.01 & \cellcolor{g2ptqblue!12}65.57 & \cellcolor{g2ptqblue!12}76.88 & \cellcolor{g2ptqblue!12}71.98 & \cellcolor{g2ptqblue!12}70.25 \\
\midrule

\multirow{7}{*}{\begin{tabular}{@{}c@{}}\textbf{LLaMA3} \\ \textbf{8B}\end{tabular}}
& BF16 & 0 & 6.14 & 0 & 3.41 & 0 & 2.72 & 53.24 & 77.90 & 48.44 & 79.25 & 75.53 & 80.74 & 73.16 & 69.75 \\
\cmidrule{2-16}
& RTN & 4.19e-01 & 9.35 & 2.87e-01 & 4.44 & 2.21e-01 & 3.37 & 42.06 & 66.12 & 32.62 & 72.31 & 62.24 & 75.35 & 66.22 & 59.56 \\
& GPTQ & 2.83e-01 & 8.15 & 1.65e-01 & 3.97 & 1.31e-01 & 3.09 & 45.31 & 70.54 & 37.37 & 74.16 & 69.09 & 76.22 & 69.30 & 63.14 \\
& GuidedQ & 2.68e-01 & 8.03 & 1.55e-01 & 3.94 & 1.23e-01 & 3.06 & 47.61 & 72.77 & 38.56 & 74.62 & 70.35 & 76.71 & 68.98 & 64.23 \\
& GPTAQ & \textbf{2.34e-01} & \textbf{7.75} & \textbf{1.21e-01} & \textbf{3.83} & 1.04e-01 & 3.00 & 50.00 & 74.96 & 38.86 & 74.72 & 70.37 & 77.53 & 69.53 & 65.14 \\
& \cellcolor{g2ptqblue!12}\name & \cellcolor{g2ptqblue!12}2.36e-01 & \cellcolor{g2ptqblue!12}\textbf{7.75} & \cellcolor{g2ptqblue!12}1.29e-01 & \cellcolor{g2ptqblue!12}3.85 & \cellcolor{g2ptqblue!12}1.05e-01 & \cellcolor{g2ptqblue!12}2.99 & \cellcolor{g2ptqblue!12}45.48 & \cellcolor{g2ptqblue!12}75.08 & \cellcolor{g2ptqblue!12}40.12 & \cellcolor{g2ptqblue!12}75.20 & \cellcolor{g2ptqblue!12}73.06 & \cellcolor{g2ptqblue!12}77.75 & \cellcolor{g2ptqblue!12}68.75 & \cellcolor{g2ptqblue!12}65.06 \\
& \cellcolor{g2ptqblue!12}\name\textsuperscript{*} & \cellcolor{g2ptqblue!12}2.39e-01 & \cellcolor{g2ptqblue!12}7.79 & \cellcolor{g2ptqblue!12}1.28e-01 & \cellcolor{g2ptqblue!12}3.84 & \cellcolor{g2ptqblue!12}\textbf{1.00e-01} & \cellcolor{g2ptqblue!12}\textbf{2.97} & \cellcolor{g2ptqblue!12}47.53 & \cellcolor{g2ptqblue!12}74.75 & \cellcolor{g2ptqblue!12}42.12 & \cellcolor{g2ptqblue!12}74.32 & \cellcolor{g2ptqblue!12}70.97 & \cellcolor{g2ptqblue!12}78.35 & \cellcolor{g2ptqblue!12}69.14 & \cellcolor{g2ptqblue!12}\textbf{65.31} \\
\midrule

\multirow{7}{*}{\begin{tabular}{@{}c@{}}\textbf{LLaMA3} \\ \textbf{70B}\end{tabular}}
& BF16 & 0 & 2.86 & 0 & 3.05 & 0 & 2.41 & 64.16 & 85.86 & 65.08 & 84.95 & 79.45 & 84.49 & 80.43 & 77.77 \\
\cmidrule{2-16}
& RTN & 2.60e+00 & 39.28 & 1.72e+00 & 16.29 & 1.16e+00 & 7.53 & 21.93 & 37.71 & 23.11 & 34.13 & 16.53 & 57.34 & 49.17 & 34.27 \\
& GPTQ & 1.12e+00 & 8.87 & 4.76e-01 & 4.75 & 3.04e-01 & 3.19 & 43.86 & 66.67 & 26.97 & 68.44 & 60.55 & 75.08 & 64.96 & 58.08 \\
& GuidedQ & 1.03e+00 & 8.03 & 3.95e-01 & 4.40 & 2.82e-01 & 3.13 & 44.11 & 69.19 & 29.64 & 70.00 & 64.43 & 74.97 & 70.40 & 60.39 \\
& GPTAQ & 8.42e-01 & 6.64 & 1.29e-01 & 3.42 & 1.47e-01 & 2.74 & 50.00 & 77.86 & 44.06 & 76.36 & 75.37 & 79.22 & 75.77 & 68.38 \\
& \cellcolor{g2ptqblue!12}\name & \cellcolor{g2ptqblue!12}\textbf{6.49e-01} & \cellcolor{g2ptqblue!12}\textbf{5.47} & \cellcolor{g2ptqblue!12}\textbf{9.09e-02} & \cellcolor{g2ptqblue!12}\textbf{3.31} & \cellcolor{g2ptqblue!12}\textbf{9.22e-02} & \cellcolor{g2ptqblue!12}\textbf{2.59} & \cellcolor{g2ptqblue!12}55.97 & \cellcolor{g2ptqblue!12}81.06 & \cellcolor{g2ptqblue!12}53.86 & \cellcolor{g2ptqblue!12}82.14 & \cellcolor{g2ptqblue!12}78.50 & \cellcolor{g2ptqblue!12}81.50 & \cellcolor{g2ptqblue!12}77.58 & \cellcolor{g2ptqblue!12}\textbf{72.94} \\
& \cellcolor{g2ptqblue!12}\name\textsuperscript{*} & \cellcolor{g2ptqblue!12}7.74e-01 & \cellcolor{g2ptqblue!12}6.23 & \cellcolor{g2ptqblue!12}1.07e-01 & \cellcolor{g2ptqblue!12}3.36 & \cellcolor{g2ptqblue!12}1.19e-01 & \cellcolor{g2ptqblue!12}2.66 & \cellcolor{g2ptqblue!12}54.61 & \cellcolor{g2ptqblue!12}81.31 & \cellcolor{g2ptqblue!12}52.60 & \cellcolor{g2ptqblue!12}80.64 & \cellcolor{g2ptqblue!12}76.89 & \cellcolor{g2ptqblue!12}81.12 & \cellcolor{g2ptqblue!12}75.14 & \cellcolor{g2ptqblue!12}71.76 \\

\bottomrule
\end{tabular}
}
\end{center}
\vspace{-1ex}
\caption{4-bit weight-activation quantization results across different model families.}
\label{tab:w4a4}
\end{table*}

%% file: tables/appendix/moe.tex
\begin{table*}[!t]
\begin{center}
\resizebox{1.0\linewidth}{!}
{
\begin{tabular}{c|l|cc|cc|cc|cccccccc}
\toprule
\multirow{2}{*}{\textbf{Model}} & \multirow{2}{*}{\textbf{Method}} & \multicolumn{2}{c|}{\textbf{WikiText2}} & \multicolumn{2}{c|}{\textbf{UltraChat}} & \multicolumn{2}{c|}{\textbf{NuminaMath}} & \multicolumn{8}{c}{\textbf{CommonSense QA}} \\
\cmidrule(lr){3-4} \cmidrule(lr){5-6} \cmidrule(lr){7-8} \cmidrule(lr){9-16}
& & \textbf{KL} & \textbf{PPL} & \textbf{KL} & \textbf{PPL} & \textbf{KL} & \textbf{PPL} & \textbf{ARC-C} & \textbf{ARC-E} & \textbf{C-eval} & \textbf{HellaS} & \textbf{LAMB} & \textbf{PIQA} & \textbf{Wino} & \textbf{Avg.} \\
\midrule

\multirow{4}{*}{\begin{tabular}{@{}c@{}}\textbf{Qwen3} \\ \textbf{30B-A3B} \\ \textbf{(W2A16)}\end{tabular}}
& BF16 & 0 & 8.70 & 0 & 4.79 & 0 & 3.41 & 56.66 & 78.96 & 84.32 & 77.74 & 64.80 & 80.58 & 71.11 & 73.45 \\
\cmidrule{2-16}
& GPTQ & 9.83e-01 & 20.13 & 6.24e-01 & 6.84 & 4.42e-01 & 4.65 & 38.99 & 56.78 & 44.73 & 57.62 & 49.82 & 70.24 & 56.12 & 53.47 \\
& \cellcolor{g2ptqblue!12}\name
& \cellcolor{g2ptqblue!12}4.93e-01
& \cellcolor{g2ptqblue!12}13.89
& \cellcolor{g2ptqblue!12}2.86e-01
& \cellcolor{g2ptqblue!12}5.68
& \cellcolor{g2ptqblue!12}2.48e-01
& \cellcolor{g2ptqblue!12}4.18
& \cellcolor{g2ptqblue!12}{45.48}
& \cellcolor{g2ptqblue!12}{67.26}
& \cellcolor{g2ptqblue!12}59.73
& \cellcolor{g2ptqblue!12}66.61
& \cellcolor{g2ptqblue!12}53.79
& \cellcolor{g2ptqblue!12}{76.06}
& \cellcolor{g2ptqblue!12}63.93
& \cellcolor{g2ptqblue!12}61.84 \\
& \cellcolor{g2ptqblue!12}\name$^{*}$
& \cellcolor{g2ptqblue!12}\textbf{4.27e-01}
& \cellcolor{g2ptqblue!12}\textbf{12.73}
& \cellcolor{g2ptqblue!12}\textbf{2.13e-01}
& \cellcolor{g2ptqblue!12}\textbf{5.23}
& \cellcolor{g2ptqblue!12}\textbf{2.09e-01}
& \cellcolor{g2ptqblue!12}\textbf{4.04}
& \cellcolor{g2ptqblue!12}44.45
& \cellcolor{g2ptqblue!12}66.08
& \cellcolor{g2ptqblue!12}{63.08}
& \cellcolor{g2ptqblue!12}{68.35}
& \cellcolor{g2ptqblue!12}{55.31}
& \cellcolor{g2ptqblue!12}75.90
& \cellcolor{g2ptqblue!12}{65.11}
& \cellcolor{g2ptqblue!12}\textbf{62.61} \\
\midrule

\multirow{4}{*}{\begin{tabular}{@{}c@{}}\textbf{Qwen3} \\ \textbf{30B-A3B} \\ \textbf{(W3A16)}\end{tabular}}
& BF16 & 0 & 8.70 & 0 & 4.79 & 0 & 3.41 & 56.66 & 78.96 & 84.32 & 77.74 & 64.80 & 80.58 & 71.11 & 73.45 \\
\cmidrule{2-16}
& GPTQ & 1.72e-01 & 10.03 & 1.18e-01 & 5.16 & 7.41e-02 & 3.57 & 52.22 & 75.42 & 73.77 & 73.87 & 62.24 & 78.07 & 68.35 & 69.13 \\
& \cellcolor{g2ptqblue!12}\name
& \cellcolor{g2ptqblue!12}1.03e-01
& \cellcolor{g2ptqblue!12}9.49
& \cellcolor{g2ptqblue!12}5.39e-02
& \cellcolor{g2ptqblue!12}4.84
& \cellcolor{g2ptqblue!12}4.33e-02
& \cellcolor{g2ptqblue!12}3.52
& \cellcolor{g2ptqblue!12}{54.35}
& \cellcolor{g2ptqblue!12}{78.41}
& \cellcolor{g2ptqblue!12}78.01
& \cellcolor{g2ptqblue!12}{75.52}
& \cellcolor{g2ptqblue!12}{63.83}
& \cellcolor{g2ptqblue!12}79.16
& \cellcolor{g2ptqblue!12}{70.48}
& \cellcolor{g2ptqblue!12}\textbf{71.39} \\
& \cellcolor{g2ptqblue!12}\name$^{*}$
& \cellcolor{g2ptqblue!12}\textbf{9.66e-02}
& \cellcolor{g2ptqblue!12}\textbf{9.41}
& \cellcolor{g2ptqblue!12}\textbf{4.65e-02}
& \cellcolor{g2ptqblue!12}\textbf{4.78}
& \cellcolor{g2ptqblue!12}\textbf{4.01e-02}
& \cellcolor{g2ptqblue!12}\textbf{3.51}
& \cellcolor{g2ptqblue!12}53.07
& \cellcolor{g2ptqblue!12}76.85
& \cellcolor{g2ptqblue!12}{78.31}
& \cellcolor{g2ptqblue!12}75.49
& \cellcolor{g2ptqblue!12}62.80
& \cellcolor{g2ptqblue!12}{79.22}
& \cellcolor{g2ptqblue!12}69.38
& \cellcolor{g2ptqblue!12}70.73 \\
\midrule

\multirow{4}{*}{\begin{tabular}{@{}c@{}}\textbf{Qwen3} \\ \textbf{30B-A3B} \\ \textbf{(W4A16)}\end{tabular}}
& BF16 & 0 & 8.70 & 0 & 4.79 & 0 & 3.41 & 56.66 & 78.96 & 84.32 & 77.74 & 64.80 & 80.58 & 71.11 & 73.45 \\
\cmidrule{2-16}
& GPTQ & 5.08e-02 & 8.87 & 3.53e-02 & 4.81 & 2.31e-02 & 3.45 & 53.92 & 78.37 & 82.17 & 76.37 & {63.85} & 79.49 & 68.82 & 71.86 \\
& \cellcolor{g2ptqblue!12}\name
& \cellcolor{g2ptqblue!12}3.36e-02
& \cellcolor{g2ptqblue!12}\textbf{8.84}
& \cellcolor{g2ptqblue!12}1.83e-02
& \cellcolor{g2ptqblue!12}\textbf{4.74}
& \cellcolor{g2ptqblue!12}1.46e-02
& \cellcolor{g2ptqblue!12}\textbf{3.40}
& \cellcolor{g2ptqblue!12}{55.63}
& \cellcolor{g2ptqblue!12}{78.54}
& \cellcolor{g2ptqblue!12}{82.62}
& \cellcolor{g2ptqblue!12}76.96
& \cellcolor{g2ptqblue!12}63.79
& \cellcolor{g2ptqblue!12}80.03
& \cellcolor{g2ptqblue!12}69.53
& \cellcolor{g2ptqblue!12}\textbf{72.44} \\
& \cellcolor{g2ptqblue!12}\name$^{*}$
& \cellcolor{g2ptqblue!12}\textbf{3.26e-02}
& \cellcolor{g2ptqblue!12}\textbf{8.84}
& \cellcolor{g2ptqblue!12}\textbf{1.64e-02}
& \cellcolor{g2ptqblue!12}4.75
& \cellcolor{g2ptqblue!12}\textbf{1.38e-02}
& \cellcolor{g2ptqblue!12}3.41
& \cellcolor{g2ptqblue!12}54.95
& \cellcolor{g2ptqblue!12}78.28
& \cellcolor{g2ptqblue!12}82.39
& \cellcolor{g2ptqblue!12}{77.14}
& \cellcolor{g2ptqblue!12}63.50
& \cellcolor{g2ptqblue!12}{80.14}
& \cellcolor{g2ptqblue!12}{70.64}
& \cellcolor{g2ptqblue!12}72.43 \\
\midrule

\multirow{4}{*}{\begin{tabular}{@{}c@{}}\textbf{Qwen3} \\ \textbf{30B-A3B} \\ \textbf{(W4A4)}\end{tabular}}
& BF16 & 0 & 8.70 & 0 & 4.79 & 0 & 3.41 & 56.66 & 78.96 & 84.32 & 77.74 & 64.80 & 80.58 & 71.11 & 73.45 \\
\cmidrule{2-16}
& GPTQ & 2.08e-01 & 9.70 & 1.58e-01 & 4.81 & 1.07e-01 & 3.52 & 51.45 & 76.81 & {76.52} & 72.66 & 60.31 & 76.82 & 67.09 & 68.81 \\
& \cellcolor{g2ptqblue!12}\name
& \cellcolor{g2ptqblue!12}1.59e-01
& \cellcolor{g2ptqblue!12}9.61
& \cellcolor{g2ptqblue!12}1.15e-01
& \cellcolor{g2ptqblue!12}4.87
& \cellcolor{g2ptqblue!12}8.53e-02
& \cellcolor{g2ptqblue!12}3.55
& \cellcolor{g2ptqblue!12}50.17
& \cellcolor{g2ptqblue!12}{77.44}
& \cellcolor{g2ptqblue!12}75.78
& \cellcolor{g2ptqblue!12}74.33
& \cellcolor{g2ptqblue!12}{61.81}
& \cellcolor{g2ptqblue!12}78.13
& \cellcolor{g2ptqblue!12}{67.72}
& \cellcolor{g2ptqblue!12}\textbf{69.34} \\
& \cellcolor{g2ptqblue!12}\name$^{*}$
& \cellcolor{g2ptqblue!12}\textbf{1.58e-01}
& \cellcolor{g2ptqblue!12}\textbf{9.56}
& \cellcolor{g2ptqblue!12}\textbf{1.09e-01}
& \cellcolor{g2ptqblue!12}\textbf{4.73}
& \cellcolor{g2ptqblue!12}\textbf{8.19e-02}
& \cellcolor{g2ptqblue!12}\textbf{3.49}
& \cellcolor{g2ptqblue!12}{51.54}
& \cellcolor{g2ptqblue!12}74.83
& \cellcolor{g2ptqblue!12}76.37
& \cellcolor{g2ptqblue!12}{74.46}
& \cellcolor{g2ptqblue!12}61.54
& \cellcolor{g2ptqblue!12}{78.40}
& \cellcolor{g2ptqblue!12}67.25
& \cellcolor{g2ptqblue!12}69.20 \\

\bottomrule
\end{tabular}
}
\end{center}
\vspace{-1ex}
\caption{Weight-only and weight-activation quantization results on Qwen3-30B-A3B.}
\label{tab:qwen3_30b_a3b}
\end{table*}

\begin{table*}[!t]
\begin{center}
\resizebox{1.0\linewidth}{!}
{
\begin{tabular}{c|l|cc|cc|cc|cccccccc}
\toprule
\multirow{2}{*}{\textbf{Model}} & \multirow{2}{*}{\textbf{Method}} & \multicolumn{2}{c|}{\textbf{WikiText2}} & \multicolumn{2}{c|}{\textbf{UltraChat}} & \multicolumn{2}{c|}{\textbf{NuminaMath}} & \multicolumn{8}{c}{\textbf{CommonSense QA}} \\
\cmidrule(lr){3-4} \cmidrule(lr){5-6} \cmidrule(lr){7-8} \cmidrule(lr){9-16}
& & \textbf{KL} & \textbf{PPL} & \textbf{KL} & \textbf{PPL} & \textbf{KL} & \textbf{PPL} & \textbf{ARC-C} & \textbf{ARC-E} & \textbf{C-eval} & \textbf{HellaS} & \textbf{LAMB} & \textbf{PIQA} & \textbf{Wino} & \textbf{Avg.} \\
\midrule

\multirow{3}{*}{\begin{tabular}{@{}c@{}}\textbf{Qwen3.8} \\ \textbf{Flash-Next}\end{tabular}}
& BF16 & 0 & 4.67 & 0 & 6.30 & 0 & 3.14 & 64.51 & 79.08 & 89.38 & 87.59 & 72.81 & 82.97 & 71.59 & 78.28 \\
\cmidrule{2-16}
& GPTQ & 1.78e-01 & 5.02 & 9.36e-02 & 6.76 & 7.61e-02 & 3.23 & 62.03 & 78.03 & 87.30 & 87.19 & 72.48 & 83.62 & 72.22 & 77.55 \\
& \cellcolor{g2ptqblue!12}\name$^{*}$ & \cellcolor{g2ptqblue!12}\textbf{1.52e-01} & \cellcolor{g2ptqblue!12}\textbf{4.85} & \cellcolor{g2ptqblue!12}\textbf{7.43e-02} & \cellcolor{g2ptqblue!12}\textbf{5.77} & \cellcolor{g2ptqblue!12}\textbf{6.20e-02} & \cellcolor{g2ptqblue!12}\textbf{3.03} & \cellcolor{g2ptqblue!12}62.71 & \cellcolor{g2ptqblue!12}79.12 & \cellcolor{g2ptqblue!12}88.04 & \cellcolor{g2ptqblue!12}87.40 & \cellcolor{g2ptqblue!12}72.39 & \cellcolor{g2ptqblue!12}84.17 & \cellcolor{g2ptqblue!12}71.82 & \cellcolor{g2ptqblue!12}\textbf{77.95} \\

\bottomrule
\end{tabular}
}
\end{center}
\vspace{-1ex}
\caption{4-bit weight-only quantization results on Qwen3.8-Flash-Next.}
\label{tab:qwen38-flash-next-w4a16}
\end{table*}

\begin{table*}[!t]
\begin{center}
\resizebox{1.0\linewidth}{!}
{
\begin{tabular}{c|l|ccccc}
\toprule
\textbf{Model} & \textbf{Method} & \textbf{GPQA-Diamond} & \textbf{Live-Code-Bench} & \textbf{ArXiv-Math} & \textbf{IFBench} & \textbf{Avg.} \\
\midrule

\multirow{3}{*}{\begin{tabular}{@{}c@{}}\textbf{Qwen3.8} \\ \textbf{Flash-Next}\end{tabular}}
& BF16 & \ensuremath{91.58 \pm 0.58} & \ensuremath{93.78 \pm 0.29} & \ensuremath{44.01 \pm 2.80} & \ensuremath{73.00 \pm 0.88} & \ensuremath{75.59 \pm 0.84} \\
\cmidrule{2-7}
& GPTQ & \ensuremath{90.57 \pm 0.29} & \ensuremath{92.64 \pm 0.05} & \textbf{43.37\ensuremath{\pm}0.56} & \ensuremath{74.11 \pm 0.51} & \ensuremath{75.17 \pm 0.11} \\
& \cellcolor{g2ptqblue!12}\name$^*$
& \cellcolor{g2ptqblue!12}\textbf{91.08\ensuremath{\pm}1.27}
& \cellcolor{g2ptqblue!12}\textbf{92.92\ensuremath{\pm}0.33}
& \cellcolor{g2ptqblue!12}\ensuremath{43.36 \pm 2.97}
& \cellcolor{g2ptqblue!12}\textbf{74.22\ensuremath{\pm}0.69}
& \cellcolor{g2ptqblue!12}\textbf{75.40\ensuremath{\pm}0.40} \\

\bottomrule
\end{tabular}
}
\end{center}
\vspace{-1ex}
\caption{4-bit weight-only quantization results for the Qwen3.8-Flash-Next model, evaluated on reasoning benchmarks. We report the mean accuracy and variance across three random seeds.}
\label{tab:qwen38-flash-next-w4a16-reasoning}
\end{table*}